\documentclass[10pt,journal,compsoc]{IEEEtran}
\usepackage[nocompress]{cite}
\usepackage[pdftex]{graphicx}

\usepackage{color}
\usepackage{amssymb}
\usepackage{mathtools}
\usepackage{amsthm}
\usepackage{bm}
\usepackage{subfigure}
\usepackage{diagbox}
\usepackage{multirow}
\usepackage{booktabs}
\usepackage{algorithm}
\newtheorem{theorem}{Theorem}[section]

\newtheorem{lemma}[theorem]{Lemma}

\theoremstyle{definition}
\newtheorem{definition}[theorem]{Definition}

\theoremstyle{remark}

\newif\ifupdate\updatefalse 
\newcommand{\modify}[2]{\ifupdate{#1}\else{\color{red}#2}\fi}

\def\ourM{CRESP}

\usepackage{amsmath}
\usepackage{algorithmic}

\usepackage{array}

\ifCLASSOPTIONcompsoc
 \usepackage[caption=false,font=footnotesize,labelfont=sf,textfont=sf]{subfig}
\else
 \usepackage[caption=false,font=footnotesize]{subfig}
\fi
\usepackage{url}

\begin{document}
%
\title{Generalization in Visual Reinforcement Learning with the Reward Sequence Distribution}

\author{Jie~Wang$^{\ast}$,~\IEEEmembership{Senior~Member,~IEEE,}
        Rui~Yang,
        Zijie~Geng,
        Zhihao~Shi,
        Mingxuan~Ye,
        Qi~Zhou,
        Shuiwang~Ji,~\IEEEmembership{Fellow,~IEEE,}
        Bin~Li,~\IEEEmembership{Member,~IEEE,}
        Yongdong~Zhang,~\IEEEmembership{Senior~Member,~IEEE,}
        and~Feng~Wu,~\IEEEmembership{Fellow,~IEEE}
\IEEEcompsocitemizethanks{\IEEEcompsocthanksitem J.~Wang, R.~Yang, Z.~Geng, M.~Ye, Q.~Zhou, B.~Li, Y.~Zhang and F.~Wu are with: a) CAS Key
Laboratory of Technology in GIPAS, University of Science and Technology of China, Hefei 230027, China; b) Institute of Artificial Intelligence, Hefei Comprehensive National Science Center, Hefei 230091,
China. 
E-mail: jiewangx@ustc.edu.cn, yr0013@mail.ustc.edu.cn, ustcgzj@mail.ustc.edu.cn, mingxuanye@mail.ustc.edu.cn, zhouqida@mail.ustc.edu.cn, binli@ustc.edu.cn, zhyd73@ustc.edu.cn, fengwu@ustc.edu.cn.
\IEEEcompsocthanksitem S.~Ji is with the Department of Computer Science and Engineering, Texas A\&M University, College Station, TX77843 USA. E-mail: sji@tamu.edu.}
\thanks{Manuscript received April 19, 2005; revised August 26, 2015.}}

%
%

\markboth{Journal of \LaTeX\ Class Files,~Vol.~14, No.~8, August~2015}%
{Shell \MakeLowercase{\textit{et al.}}: Generalization in Visual Reinforcement Learning with the Reward Sequence Distribution}
%



\IEEEtitleabstractindextext{%
\begin{abstract}
Generalization in partially observed markov decision processes (POMDPs) is critical for successful applications of visual reinforcement learning (VRL) in real scenarios.
A widely used idea is to learn task-relevant representations that encode task-relevant information of common features in POMDPs, i.e., rewards and transition dynamics.
As transition dynamics in the latent state space---which are task-relevant and invariant to visual distractions---are unknown to the agents, existing methods alternatively use transition dynamics in the observation space to extract task-relevant information in transition dynamics.
However, such transition dynamics in the observation space involve task-irrelevant visual distractions, degrading the generalization performance of VRL methods.
To tackle this problem, we propose the {\em \textbf{r}eward \textbf{s}equence \textbf{d}istribution} conditioned on the {\em starting \textbf{o}bservation} and the {\em predefined subsequent \textbf{a}ction sequence} (RSD-OA).
The appealing features of RSD-OA include that: (1) RSD-OA is invariant to visual distractions, as it is conditioned on the predefined subsequent action sequence without task-irrelevant information from transition dynamics, and (2) the reward sequence captures long-term task-relevant information in both rewards and transition dynamics.
Experiments demonstrate that our representation learning approach based on RSD-OA significantly improves the generalization performance on unseen environments, outperforming several state-of-the-arts on DeepMind Control tasks with visual distractions.
\end{abstract}

}

\maketitle

\IEEEdisplaynontitleabstractindextext

%
\IEEEpeerreviewmaketitle

\IEEEraisesectionheading{\section{Introduction}\label{sec:introduction}}

\IEEEPARstart{V}{isual} reinforcement learning (VRL) aims to solve complex control tasks---which are formulated as partially observed markov decision process (POMDP) problems---directly from high-dimensional image observations~\cite{nips/LeeNAL20,nips/ChenLSA21}. Prior works have achieved remarkable success in VRL, such as DrQ for locomotion control~\cite{iclr/YaratsKF21}, IMPALA for multi-task learning~\cite{icml/EspeholtSMSMWDF18}, and QT-Opt for robot grasping~\cite{corr/abs-1806-10293}.

However, these methods are difficult to generalize to test environments, as the image observations usually involve unseen visual distractions, such as dynamic backgrounds and colors of the objects under control.
This can lead to the agent overfitting to visual distractions in training environments and failing to learn transferable skills based on task-relevant information~\cite{iclr/SongJTDN20,icml/RaileanuF21}.

To train an agent with transferable skills, one of the promising approaches is to learn desired task-relevant representations from rewards and transition dynamics. Such representations not only facilitate the learning of future rewards and transition dynamics but also discard task-irrelevant information from observations~\cite{nips/MazoureCDBH20,iclr/AgarwalMCB21,icml/FanL22}. Some prior works introduce prediction tasks to learn both the rewards and transition dynamics for extracting task-relevant information~\cite{icml/0001LSFKPGP20,iclr/0001MCGL21}. Other works propose self-supervised or unsupervised auxiliary tasks, which incorporate the rewards and transition dynamics into their optimization objectives to encode task-relevant information~\cite{iclr/AgarwalMCB21,iclr/MazoureAHKM22}.


In POMDPs, the transition dynamics in the latent state space (latent state transition dynamics)---which are task-relevant and invariant to visual distractions---are unknown to the agents. Thus, the aforementioned methods alternatively use the transition dynamics in the observation space (observation transition dynamics)\modify{to extract task-relevant information in transition dynamics}{.}
However, such observation transition dynamics involve visual distractions with task-irrelevant information, which significantly degrades the generalization performance of VRL methods.
We illustrate these two transition dynamics in Figure~\ref{fig-rl} and provide more discussions in Section~\ref{sec-4}.

\begin{figure}
    \centering
    \includegraphics[scale=0.5]{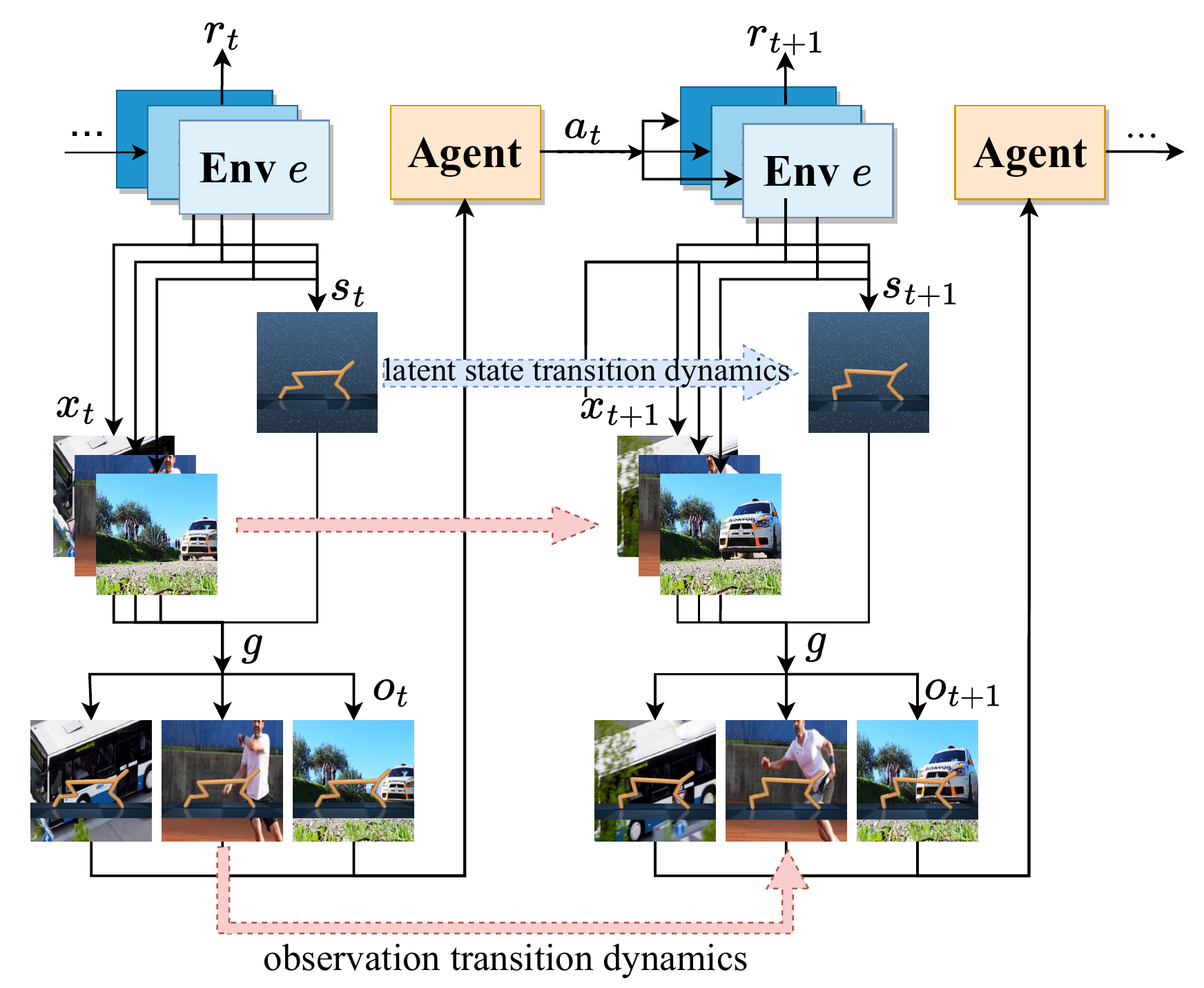}
    \caption{The agent-environment interactions with dynamic background distractions. Each environment $e$ provides a latent state $s_t$ and a background $x_t$ to generate an observation $o_t=g(s_t,x_t)$ through a nonlinear function $g$. The agent receives $o_t$ and takes an action $a_t$ in $e$, leading to the transitions of latent states (from $s_t$ to $s_{t+1}$), backgrounds (from $x_t$ to $x_{t+1}$), and observations (from $o_t$ to $o_{t+1}$). The red arrow at the bottom represents the observation transition dynamics, which are available to the agent but involve background distractions. The blue arrow represents the latent state transition dynamics, which are task-relevant but unavailable to the agent.}
    \label{fig-rl}
\end{figure}

The working horse of this paper is the {\em \textbf{r}eward \textbf{s}equence \textbf{d}istribution} conditioned on the {\em starting \textbf{o}bservation} and the {\em predefined subsequent \textbf{a}ction sequence} (RSD-OA). 
The major novelty of RSD-OA is that it captures the long-term task-relevant information in both rewards and observation transition dynamics without task-irrelevant information.
Specifically, (1) RSD-OA is invariant to visual distractions, as it is conditioned on the predefined subsequent action sequence, without task-irrelevant information from observation transition dynamics, and (2) the reward sequence captures long-term task-relevant information in both rewards and observation transition dynamics.
Furthermore, we provide theoretical analysis in Section~\ref{sec-4} to show that the representations learned by using RSD-OA can derive optimal policies in unseen test environments.

We provide examples in Figure~\ref{fig-rsds} to illustrate that RSD-OA is invariant to visual distractions and captures task-relevant information in both rewards and transition dynamics.
Figure~\ref{fig-rsds}(a) shows a counter-example: the agent performs the action sequence generated by the current policy $\pi_t$ given the observation sequence $\mathbf{o}_t$, and then it receives the reward sequence $\mathbf{r}_{t+1}$. 
The corresponding distribution $p_{\pi_t}(\mathbf{r}_{t+1}|\mathbf{o}_t)$ is dependent on visual distractions from the observation sequence.
In contrast, Figure~\ref{fig-rsds}(b) shows an example of our proposed RSD-OA: the agent performs the predefined action sequence $\mathbf{a}_t$ to obtain the reward sequence $\mathbf{r}_{t+1}$. RSD-OA $p(\mathbf{r}_{t+1} | o_t, \mathbf{a}_t)$ is independent of the current policy $\pi_t$ and avoids using the observation sequence $\mathbf{o}_t$, which effectively discards the task-irrelevant information. Notice that we use $\mathbf{o}_t$ to denote the vector of observation sequence and use $o_t$ to denote the observation at timestamp $t$.
Moreover, as the obtained rewards are related to the latent states, the reward sequence implicitly encodes the task-relevant information of transition dynamics.

\begin{figure}
    \centering
    \includegraphics[width=9.1cm]{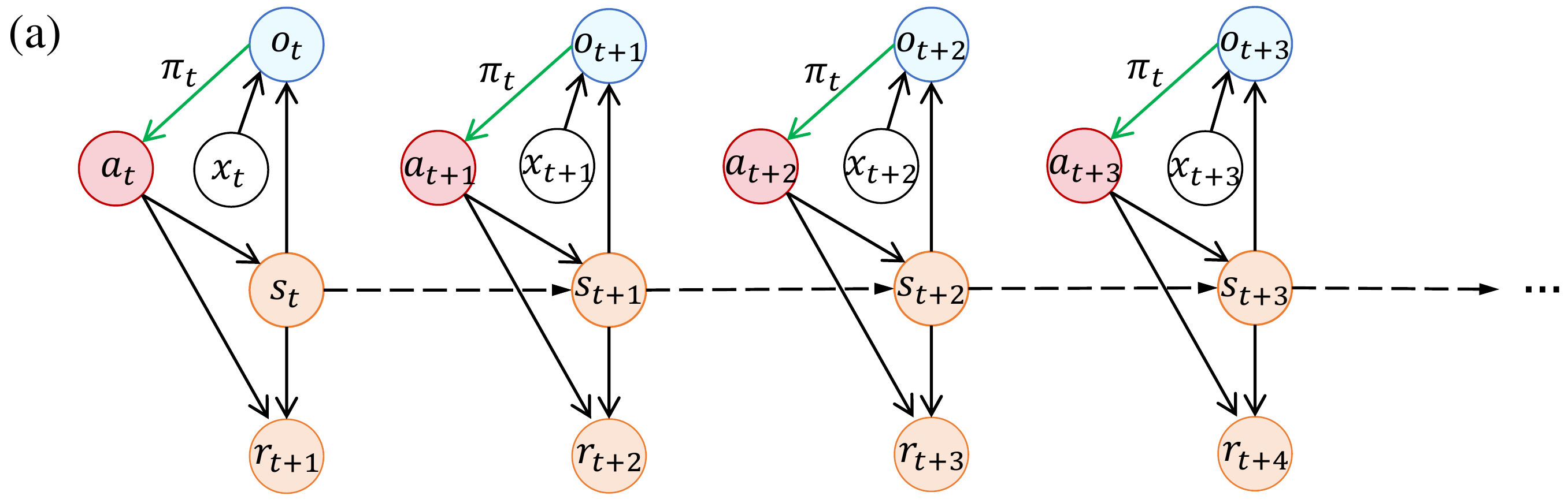}
    \includegraphics[width=9.1cm]{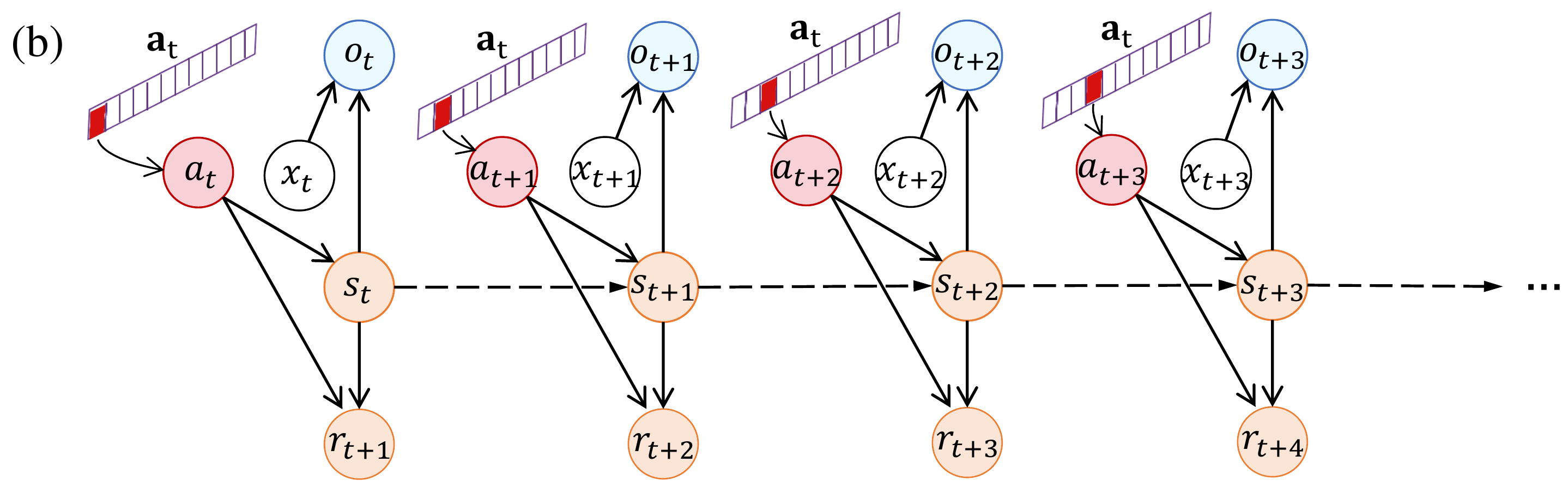}
    \caption{Illustration of reward sequences. 
    (a) The agent maps an observation $o_t$ to an action $a_t$ by following the current policy $\pi_t$ (the green arrow from $o_t$ to $a_t$), and it takes the action to receive a reward $r_{t+1}$, leading to the latent state transition (the dashed arrow from $s_t$ to $s_{t+1}$).
    In this way, the obtained reward sequence $\mathbf{r}_{t+1}=(r_{t+1},r_{t+2},\cdots)$ depends on the current policy $\pi_t$ and the observation sequence $\mathbf{o}_t=(o_t,o_{t+1},\cdots)$, thus influenced by visual distractions from observations.
    (b) Starting from an observation $o_t$, the agent takes the subsequent actions from a predefined action sequence $\mathbf{a}_t=(a_t, a_{t+1},\cdots)$, which is independent of the observations.
    In this way, the obtained reward sequence $\mathbf{r}_{t+1}$ is determined by the starting latent state $s_t$, latent state transition dynamics (dashed arrows), and the predefined action sequence $\mathbf{a}_t$, without involving visual distractions.
    }
    \label{fig-rsds}
\end{figure}

Based on RSD-OA, we propose a novel task-relevant representation learning approach---namely, \textbf{C}haracteristic \textbf{Re}ward \textbf{S}equence \textbf{P}rediction (\ourM{})---to encode the information of RSD-OA.
Specifically, we propose a supervised auxiliary task that uses the characteristic function of RSD-OA in spectral space as the supervision signal. This is because that the characteristic function can well reconstruct the corresponding high-dimensional distribution, and meanwhile, it can be easily computed via sampling~\cite{ansari2020characteristic, yu2004empirical}.
Experiments on DeepMind Control Suite~\cite{corr/abs-1801-00690} with visual distractors~\cite{corr/abs-2101-02722} demonstrate that \ourM{} significantly improves the generalization performance, outperforming several state-of-the-arts in unseen test environments.

This paper is a significant extension from a short version ``Learning Task-relevant Representations for Generalization via Characteristic Functions of Reward Sequence Distributions", which appears in SIGKDD 2022 \cite{Yang_2022}.
This journal manuscript extends the conference version by proposing a significantly enhanced version of \ourM{}, namely \ourM{-T}.
\ourM{-T} significantly improves the robustness of \ourM{} by introducing more effective optimization objective and architecture.
Specifically, first, to encode RSD-OA more accurately, \ourM{-T} proposes a novel and effective spectral cosine similarity, which introduces the phase information of RSD-OA besides its amplitude (see Section~\ref{subsec-5.3.1}).
Second, to capture task-relevant features more effectively, \ourM{-T} proposes to model the temporal dependencies of RSD-OA by introducing the attention mechanisms (see Section~\ref{subsec-5.3.2}).
Finally, we conduct extensive experiments under more challenging multiple training environment settings with dynamic color distractions to demonstrate that \ourM{-T} achieves an average performance improvement of ${\bf+23.5\%}$ over \ourM{} (see Section~\ref{subsec-6.2}).

\section{Related Work}

Our work considers solving the problem of generalization in VRL by learning representations without task-irrelevant information from observations.
Specifically, we use the characteristic functions of RSD-OA as the supervision signals.
Moreover, we introduce a transformer architecture for
effectively encoding information of RSD-OA.

\subsection{Generalization in VRL}
The study of generalization in VRL aims to produce VRL algorithms whose policies generalize well to unseen test environments.
Many works apply standard supervised learning regularization techniques, such as dropout, batch normalization, and L2 weight decay~\cite{farebrother2018generalization, nips/IglCLTZDH19}, to prevent overfitting for performance improvement of generalization.
Although they are easy to implement, these methods do not exploit any properties of sequential decision-making problems.
Other approaches focus on data augmentation~\cite{lee2020network, laskin2020reinforcement, raileanu2021automatic}. They enlarge the available data space and implicitly incorporate prior knowledge about the types of distractions in test environments.
Although these methods show promising results in well-designed experiment settings, strong assumptions about the variation between the training and test environments may limit their real applications~\cite{corr/SurveyOfGen}. 
In contrast to these methods, we consider a more realistic setting without assuming such prior knowledge.

Another line of work uses \textbf{representation learning} techniques to improve generalization performance in VRL.
Prior works~\cite{lange2010deep,lange2012autonomous} have proposed a two-step learning process, which first trains an auto-encoder by using a reconstruction loss for low-dimensional representations, and then uses these representations for policy optimization.
However, learned representations do not necessarily extract all useful features and may involve task-irrelevant information.
Some works use bisimulation metrics to learn representations that are invariant to irrelevant features~\cite{iclr/0001MCGL21}.
However, such methods use the transition dynamics that vary with the environments, which may cause learned representations to involve task-irrelevant information from visual distractions.
A recent study~\cite{LehnertLF20} allows the agent to predict future reward sequences for representation learning. However, this method only considers finite MDPs, which is difficult to extend to high-dimensional observation spaces, e.g., images.

\subsection{Characteristic Functions of Random Variables}
The characteristic function of a random variable is equal to the Fourier transforms of its probability density functions.
Based on probability theory, characteristic functions can be used to specify high-dimensional distributions.
This is because two random variables have the same distribution if and only if they have the same characteristic function.
Prior works have explored the use of characteristic functions to solve data-driven problems, including model-fitting and data generation~\cite{ansari2020characteristic,yu2004empirical}.
In our work, we leverage this tool for a simple and tractable approximation of high-dimensional distributions.
Our experiments demonstrate that the use of characteristic functions enables our method to effectively learn the distributions of reward sequences.

\subsection{Transformer Architecture}
Transformers, the neural network architectures based on attention mechanisms, have achieved great success in many tasks including natural language processing, computer vision, and time series forecasting.
Benefiting from attention mechanisms, transformers have strong abilities to model short-range and long-range temporal dependencies in sequential data, which allows several methods to model the temporal dependencies in the input or output sequences without considering their distance~\cite{VaswaniSPUJGKP17,aaai/ZhouZPZLXZ21}.
In recent years, transformer architectures have been progressively used in RL approaches to solve sequential tasks, such as the decision transformer to predict actions autoregressively~\cite{nips/ChenLRLGLASM21} and the transformer-based model to predict dynamics~\cite{corr/abs-2202-09481}.
In this paper, we also take advantage of transformers to model dependencies of inputs---an observation and given action sequences---for learning the proposed RSD-OA. Experiments demonstrate that our transformer-based method outperforms several state-of-the-arts for generalization.

\section{Preliminaries}
\label{sec:pre}
We consider a family of environments with the same high-level task but different visual distractions in VRL, and we denote the set of these environments by $\mathcal{E}$.
We model each environment $e\in\mathcal{E}$ as a partially observed markov decision process (POMDP) denoted by a tuple $\mathcal{M}^e=(\mathcal{S}, \mathcal{O}, \mathcal{A},\mathcal{R}, p,p^e,\gamma)$, where $\mathcal{S}$ is the latent state space, $\mathcal{O}$ is the observation space, $\mathcal{A}$ is the action space, $\mathcal{R}$ is the reward space, $p(s',r|s,a)$ is the latent state transition probability, $p^e(o',r|o,a)$ is the observation transition probability for the environment $e\in\mathcal{E}$, and $\gamma \in [0,1)$ is the discount factor. We assume that $\mathcal{R}$ is bounded.

At each time step $t$, the agent achieves an unseen latent state $S_t$ and obtains an observation $O_t$.\footnote{Throughout this paper, we use uppercase letters such as $S_t$ and $O_t$ to denote random variables, and use lowercase letters such as $s_t$ and $o_t$ to denote the corresponding values of random variables.}
We assume that the observation is determined by the latent state and some task-irrelevant visual factors that vary with environments, such as backgrounds or agent colors in DeepMind Control tasks.
Formally, we suppose that there exists an \textit{observation function} $g:\mathcal{S}\times\mathcal{X}\to\mathcal{O}$ \cite{iclr/SongJTDN20,du2019provably} such that $O_t=g(S_t,X_t)$, where $\mathcal{X}$ is the set of such visual factors, and $X_t \in \mathcal{X}$ is a random variable independent with $S_t$ and $A_t$.
We also suppose that visual factors have their transition probability $q^e(x'|x)$.

Figure~\ref{fig-rl} shows the agent-environment interactions in POMDPs with visual distractions.
We aim to find a policy $\pi (\cdot|o_t)$ that maximizes the expected cumulative reward $\mathbb{E}^{e} \left[ \sum_{t=0}^{\infty} \gamma^t R_t \right]$ simultaneously in all environments $e\in\mathcal{E}$ with the same task, where $\mathbb{E}^{e}[\cdot]$ is the expectation taken in the environment $e$.

We assume that the environments follow a generalized Block structure~\cite{icml/0001LSFKPGP20,du2019provably}. 
That is, an observation $o\in\mathcal{O}$ uniquely determines its generating latent state $s$ and the visual factor $x$.
This \mbox{assumption} implies that the observation function $g(s,x)$ is invertible with respect to both $s$ and $x$.
For simplicity, we denote $s=[o]_s$ and $x=[o]_x$ as the generating latent state and visual factor respectively for observation $o$.
Furthermore, we have the formulation among the transition dynamics $p^e(o',r|o,a)=p(s',r|s,a)q^e(x'|x)$, where $s=[o]_s, s'=[o']_s$, $x=[o]_x$, and $x'=[o']_x$.

\section{RSD-OA}
\label{sec-4}

In this section, we consider encoding task-relevant information in both rewards and observation transition dynamics without task-irrelevant information. Thus, we propose RSD-OA for general stochastic environments. Moreover, for deterministic environments, a special case of stochastic environments, we discuss the reward sequence function of the starting observation and the predefined subsequent action sequence (RSF-OA).




We first discuss the challenge of learning task-relevant representations from rewards and transition dynamics without task-irrelevant information.
In VRL, as latent state transition dynamics are usually unavailable to the agents, many prior works alternatively use the observation transition dynamics to extract task-relevant information~\cite{iclr/0001MCGL21,corr/abs-1807-03748,aaai/Wang00LL22}.
However, the observation transition dynamics are relevant to the visual factors because they comprise the transition dynamics of both latent states and task-irrelevant visual factors.
See Figure~\ref{fig-rl} for an illustration.
Formally, we have $$p^e(o',r|o,a)=p(s',r|s,a)q^e(x'|x).$$ 
This formula implies that the observation transition probability varies with the environment $e\in\mathcal{E}$.
Therefore, learning representations by directly using observation transition dynamics may encode task-irrelevant information from the transition probability of visual factors.

In contrast to observation transition dynamics, we find that the one-step reward distribution conditioned on the starting observation and the given action is relevant to VRL tasks and invariant to visual distractions.
Formally, if two observations $o$ and $o'$ are generated by the same latent state $s$, i.e., $s=[o]_s=[o']_s$, we have $p^e(r|o,a)=p^e(r|o',a)=p(r|s,a)$ for any $a\in\mathcal{A}$ and $e\in\mathcal{E}$.
Thus, we consider using the reward distribution instead of observation transition dynamics for representation learning.
Since our goal is to maximize the expected cumulative rewards, we need not only the current reward but also the sequences of future rewards.
Therefore, we extend the one-step reward distribution to RSD-OA.

For the ease of reference, we introduce some new notations.
We denote $\mathcal{A}^T=\{\mathbf{a}=(a_1,\cdots, a_T):a_i\in\mathcal{A}\}$ and $\mathcal{R}^T=\{\mathbf{r}=(r_2,\cdots,r_{T+1}):r_i\in\mathcal{R}\}$ as the spaces of action sequences and reward sequences with length $T$, respectively.
Let $\Delta (\mathcal{R}^T)$ be the set of probability distributions over $\mathcal{R}^T$.
The subsequent actions $\mathbf{A}=(A_1,\cdots,A_{T})$ is a $T$-dimensional random vector over $\mathcal{A}^T$. The sequence of the subsequent rewards $\mathbf{R}=(R_{2},\cdots,R_{T+1})$ is a $T$-dimensional random vector over $\mathcal{R}^T$.
\footnote{We use bold uppercase letters such as $\mathbf{A}$, $\mathbf{R}$, and $\bm{\Omega}$ to denote random vectors in high-dimensional spaces and use bold lowercase letters such as $\mathbf{a}$, $\mathbf{r}$, and $\bm{\omega}$ to denote deterministic vectors in such spaces.}

\textbf{Deterministic Environments}:
To clarify our idea, we first consider deterministic environments, which are special cases of stochastic environments. The rewards in deterministic environments are generated by the reward function.

For deterministic environments, we propose RSF-OA.
Specifically, starting from an observation $o\in\mathcal{O}$ with the corresponding latent state $s = [o]_s\in\mathcal{S}$, if we perform a predefined action sequence $\mathbf{a}=(a_1,\cdots,a_{T})\in\mathcal{A}^T$, we will receive a reward sequence $\mathbf{r}=(r_{2},\cdots,r_{T+1})\in\mathcal{R}^T$ from the deterministic environment $e\in\mathcal{E}$.
This reward sequence $\mathbf{r}$ is uniquely determined by the starting latent state $s$ and the predefined action sequence $\mathbf{a}$. 
Formally, we have $\mathbf{r}=RSF\left([o]_s, \mathbf{a}\right)$.

Notice that RSF-OA does not involve task-irrelevant information from observation transition dynamics, as its formulation does not involve the observation sequence.
See Figure~\ref{fig-rsds} for an illustration.
The main advantages of RSF-OA are include:
(1) the reward sequence is invariant to visual distractions, as it is only related to the latent state and the predefined action sequence;
(2) the reward sequence captures long-term task-relevant information.
Therefore, we can use RSF-OA to encode the task-relevant information for representation learning.



\textbf{Stochastic Environments}:
For general stochastic environments, we propose RSD-OA to capture task-relevant information. 
Different from deterministic environments, the reward sequence in stochastic environments is a random vector. 
Thus, formally, RSD-OA $p(\mathbf{r}|o,\mathbf{a})$ is the probability density function of reward sequence $\mathbf{R} = \mathbf{r} \in \mathcal{R}^T$ conditioned on the starting observation $O=o\in\mathcal{O}$ and the action sequence $\mathbf{A}=\mathbf{a}\in\mathcal{A}^T$.
For any $o\in\mathcal{O}, \mathbf{a}=(a_1,\cdots,a_T)$, and $\mathbf{r}=(r_2,\cdots,r_{T+1})$, we have
\begin{align*}
    p(\mathbf{r}|o,\mathbf{a})=\prod_{i=1}^T p(r_{i+1}|s, a_1, \cdots, a_i),
\end{align*}
where $s=[o]_s$ denotes the starting latent state, and $p(r_{i+1}|s,a_1,\cdots,a_i)$ denotes the probability density function of the reward $r_{i+1} \in \mathcal{R}$ conditioned on the starting latent state $s$ and the action sequence $(a_1,\cdots,a_i)$.
The same as RSF-OA, RSD-OA does not involve the task-irrelevant information from observation transition dynamics.


One of the appealing features of RSD-OA is that it is invariant to the visual distractions. Specifically, in any starting observation $o\in\mathcal{O}$ derived from a latent state $s\in\mathcal{S}$, if the agent performs a predefined action sequence $\mathbf{a}\in\mathcal{A}^T$, it will receive the same reward sequence random vector $\mathbf{R}\in\mathcal{R}^T$. Such random vector and its distribution RSD-OA are independent of the visual distractions from the starting observation $o$. They are indeed related to the starting latent state $s$ and the action sequence $\mathbf{a}$.
Figure~\ref{fig-rsds}(b) shows an example with $\mathbf{R}=\mathbf{r}$. 
Formally, we provide the formula that for any $o,o'\in\mathcal{O}$ such that $s=[o]_s=[o']_s$, we have
\begin{align*}
    p(\mathbf{r}|o,\mathbf{a}) = p(\mathbf{r}|o',\mathbf{a}) = p(\mathbf{r}|s,\mathbf{a}).
\end{align*}



Based on the invariance of RSD-OA, we consider learning task-relevant representations by leveraging RSD-OA.
Specifically, we propose to learn RSD-OA to improve representations, e.g., using RSD-OA to generate reward {sequences} as supervised signals.
We then define such task-relevant representations as $T$-level reward sequence representations as follows.
\begin{definition}
A representation $\Phi(o)$ is a \textit{$T$-level reward sequence representation} if it can derive the distribution of any reward sequence $\mathbf{r}\in\mathbf{R}^T$ received from any observation $o\in\mathcal{O}$ by following any action sequence $\mathbf{a}\in\mathcal{A}^T$ with length $T$, i.e., there exists $f$ such that
\begin{align*}
   f(\mathbf{r};\Phi(o),\mathbf{a}) = p(\mathbf{r}|o,\mathbf{a}),\, \forall\, \mathbf{r}\in\mathbf{R}^T,\,o\in\mathcal{O},\,\mathbf{a}\in\mathcal{A}^T,
\end{align*}
where $\Phi:\mathcal{O}\to\mathcal{Z}$ is an encoder.
\end{definition}

Another appealing feature of RSD-OA is that it can capture long-term task-relevant information in both rewards and observation transition dynamics.
Specifically, (1) the reward sequence is related to the latent state sequence (see Figure~\ref{fig-rsds}(b)), and thus it implicitly encodes the information of latent state transition dynamics---the task-relevant components of observation transition dynamics; (2) the reward sequence with length $T$ provides long-term task-relevant information in the next $T$ steps.
Therefore, the $T$-level reward sequence representations indeed extract rich task-relevant information.
Formally, $T$-level reward sequence representations are $T'$-level reward sequence representations, where $T,\,T'\in\mathbb{N}^*$ and $T>T'$.
If $T$ tends to infinity, the representations will encode all task-relevant information from rewards and observation transition dynamics.
We then provide the following definition.

\begin{definition}
A representation $\Phi(o)$ from any observation $o \in \mathcal{O}$ is a \textit{reward sequence representation} if it is a $T$-level reward sequence representation for all $T\in\mathbb{N}^*$.
\end{definition}

Such reward sequence representations are equivalent to $\infty$-level reward sequence representations.
In practice, we learn finite $T$-level reward sequence representations as approximations. 
To provide a theoretical guarantee for the approximation, the following theorem gives a value bound between the true optimal value function and the value function on top of the $T$-level reward sequence representations.

\begin{theorem}
\label{thm:bd}
Let $\Phi(o)$ be a $T$-level reward sequence representation from any observation $o \in \mathcal{O}$, $V^e_*:\mathcal{O}\to\mathbb{R}$ be the optimal value function in the environment $e\in\mathcal{E}$, $\Bar{V}^e_*:\mathcal{Z}\to\mathbb{R}$ be the optimal value function on the representation space, and $\bar{r}$ be a bound of the reward space, i.e., $|r|<\bar{r}$ for any $r\in\mathcal{R}$.
We have
\begin{align*}
    0\le V^e_*(o)-\Bar{V}^e_*\circ\Phi(o)\le\frac{2\gamma^T}{1-\gamma}\bar{r},
\end{align*}
for any $o\in\mathcal{O}$ and $e\in\mathcal{E}$.
\end{theorem}
\begin{proof}
See Appendix~\ref{app-thm:bd}
\end{proof}

\begin{figure*}
    \begin{center}
    \centerline{\includegraphics[scale=0.48]{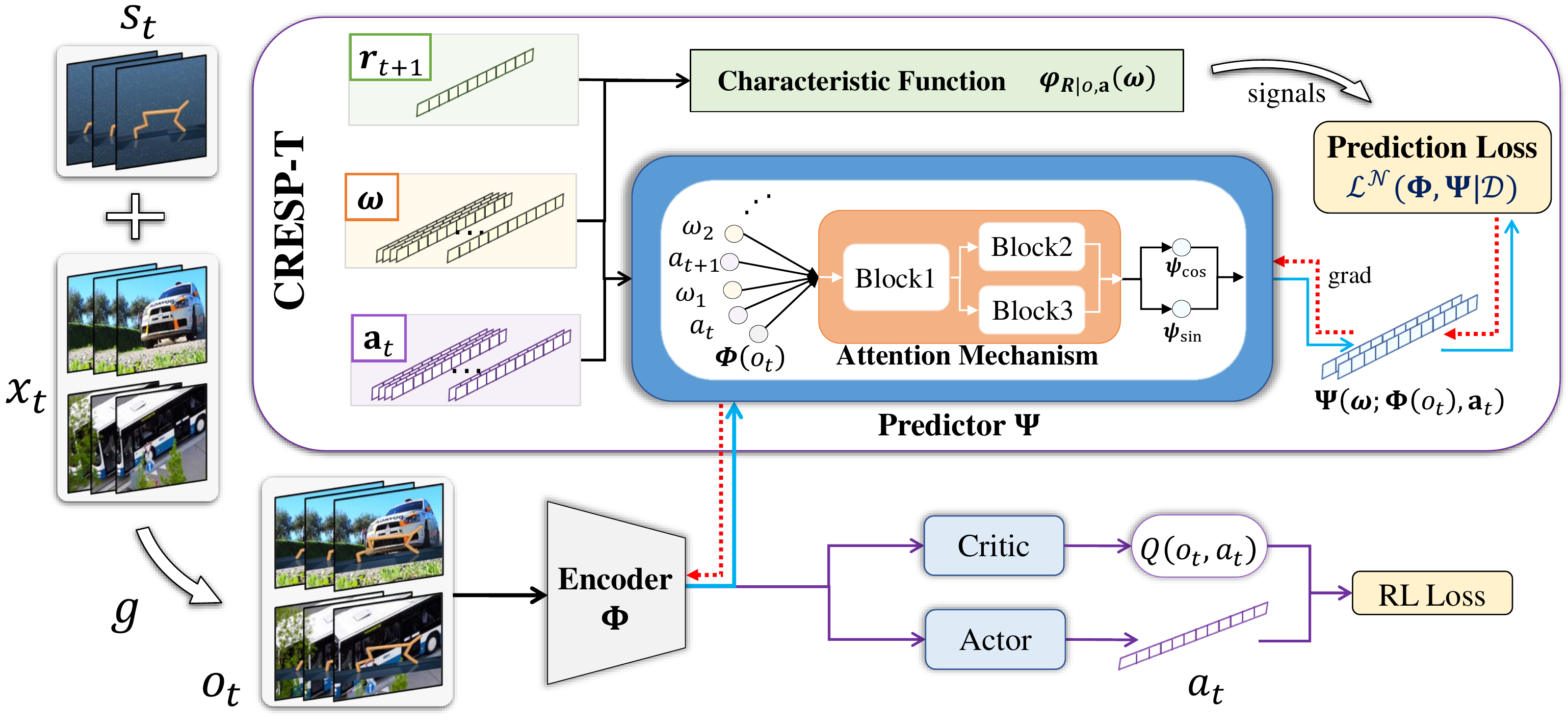}}
    \caption{The overall architecture of \ourM{-T}.
    \ourM{-T} minimizes the predictive loss to train an encoder $\Phi$ and simultaneously uses $\Phi$ to learn a policy in the actor-critic setting.
    In the prediction task, \ourM{-T} predicts the characteristic functions of RSD-OA through the encoder $\Phi$ and the predictor $\Psi$ (in the purple box).
    The predictive loss $\mathcal{L^N}(\Phi, \Psi|\mathcal{D})$ provides the gradients (red lines) to update both the predictor $\Psi$ and the encoder $\Phi$.
    Here $\mathbf{r}_{t+1}$ and $\mathbf{a}_t$ are sequences drawn from the replay buffer $\mathcal{D}$.
    $\bm{\omega}$ are sampled from the standard Gaussian distribution $\mathcal{N}$.}
    \label{fig-arc}
    \end{center}
\end{figure*}

\section{\ourM{} for Representation Learning}
\label{sec-5}

In this Section, we propose our task-relevant representation learning approach, \textbf{C}haracteristic \textbf{Re}ward \textbf{S}equence \textbf{P}rediction (\ourM{}), which uses the characteristic functions to encode the information of RSD-OA.
Specifically, we first introduce the characteristic functions of RSD-OA in Section~\ref{subsec-5.1}.
We then introduce \ourM{} in Section~\ref{subsec-5.2}.
Finally, we extend the conference version by proposing an enhanced version of \ourM{}, namely \ourM{-T}, in Section~\ref{subsec-5.3}.

\subsection{Characteristic Functions of RSD-OA}
\label{subsec-5.1}

To encode the information of RSD-OA, we propose to use the characteristic function, by noticing that the characteristic function can well reconstruct the corresponding high-dimensional distribution (see Appendix~\ref{lem:cf}), and meanwhile, it can be easily computed via sampling~\cite{ansari2020characteristic, yu2004empirical}. 
We define the characteristic function $\varphi_{\mathbf{R}|o,\mathbf{a}}$ of RSD-OA as
\begin{align*}
    \varphi_{\mathbf{R}|o,\mathbf{a}} (\bm{\omega})=\mathbb{E}_{\mathbf{R}\sim p(\cdot|o, \mathbf{a})}\left[e^{j\langle\bm{\omega},\mathbf{R}\rangle}\right]=\int_{\mathcal{R}^T} e^{j\langle\bm{\omega},\mathbf{r}\rangle}p(\mathbf{r}|o, \mathbf{a})\mathrm{d}\mathbf{r},
\end{align*}
where $\bm{\omega}\in\mathbb{R}^T$ denotes the random vector of characteristic function, and $j=\sqrt{-1}$ is the imaginary unit.
Since we consider discounted cumulative rewards for VRL tasks, we use $\langle\cdot,\cdot\rangle$ to denote the weighted inner product, i.e., 
$\langle \bm{\omega}, \mathbf{r}\rangle=\sum_{t=1}^T \gamma^t\omega_tr_t,$
where $\gamma$ is the discounted factor.

The advantages of characteristic function also include:
1) $\left|\varphi_{\mathbf{R}|o,\mathbf{a}}(\bm{\omega})\right|\le\mathbb{E}_{\mathbf{R}\sim p(\cdot|o, \mathbf{a})}\left|e^{j\langle\bm{\omega},\mathbf{R}\rangle}\right|=1,$ which indicates that the characteristic function always exists and is uniformly bounded;
2) the characteristic function $\varphi_{\mathbf{R}|o,\mathbf{a}}$ of RSD-OA is uniformly continuous on $\mathbb{R}^T$, which makes it tractable for learning.

Based on characteristic functions $\varphi_{\mathbf{R}|o,\mathbf{a}}$ of RSD-OA, we introduce an equivalent definition of $T$-level reward sequence representations in the following theorem. Such the theorem provides the theoretical basis for learning task-relevant representations by using the characteristic function $\varphi_{\mathbf{R}|o,\mathbf{a}}$ of RSD-OA. 
\begin{theorem}
\label{thm:cf}
A representation $\Phi(o)$ from any observation $o\in\mathcal{O}$ is a $T$-level reward sequence representation if and only if there exits a predictor $\Psi$ such that for all $\bm{w}\in\mathbb{R}^{T},o\in\mathcal{O}$ and $\mathbf{a}\in\mathcal{A}^T$,
\begin{align*}
    \Psi(\bm{\omega};\Phi(o),\mathbf{a})=\varphi_{\mathbf{R}|o,\mathbf{a}}(\bm{\omega})=\mathbb{E}_{\mathbf{R}\sim p(\cdot|o,\mathbf{a})}\left[e^{j\langle\bm{\omega},\mathbf{R}\rangle}\right].
\end{align*}
\end{theorem}
\begin{proof}
See Appendix~\ref{app-thm:cf}.
\end{proof}

\subsection{\ourM{}}
\label{subsec-5.2}

We introduce \ourM{} that predicts the characteristic function of RSD-OA for task-relevant representation learning. 
Specifically, we first apply an encoder $\Phi:\mathcal{O}\to\mathcal{Z}$ to extract the representation $\Phi(o)$ from the observation $o\in \mathcal{O}$. We then introduce a predictor $\Psi$ that maps the random vector $\bm{\omega}$, the representation $\Phi(o)$, and the predefined action sequence $\mathbf{a}$ to the characteristic function $\varphi_{\mathbf{R}|o,\mathbf{a}}$.
Thus, we can train the encoder $\Phi$ and the predictor $\Psi$ by minimizing the distance between the predicted value $\Psi(\bm{\omega};\Phi(o),\mathbf{a})$ and the target value 
$\varphi_{\mathbf{R}|o,\mathbf{a}}$.
Based on the training, we can learn an encoder to extract the task-relevant representation.

Based on our idea, we first consider formulating the optimization objective for representation learning.
Then, we propose the implementation of \ourM{}.

\subsubsection{Optimization Objective of \ourM{}}
\label{subsec-5.2.1}

We consider using the weighted mean squared error to measure the distance between predicted values $\Psi(\bm{\omega};\Phi(o),\mathbf{a})$ and target values $\varphi_{\mathbf{R}|o,\mathbf{a}}$ of characteristic functions of RSD-OA. Thus, we introduce the predictive loss function of the given observation $o$ and predefined action sequence $\mathbf{a}$:
\begin{align}
\label{loss:WSE1}
     & L^{\mathcal{N}}_{WSE} (\Phi, \Psi|o,\mathbf{a}) \\
    =& \left\|\Psi(\cdot; \Phi(o),\mathbf{a})-\varphi_{\mathbf{R}|o,\mathbf{a}}(\cdot)\right\|^2_\mathcal{N} \nonumber\\ 
    =& \,\, \mathbb{E}_{\bm{\Omega}\sim\mathcal{N}}\left[ \left| \Psi\left(\bm{\Omega};\Phi(o), \mathbf{a}\right) - \varphi_{\mathbf{R}|o,\mathbf{a}}(\bm{\Omega}) \right|^2_2 \right]\nonumber\\
    =& \int_{\mathbb{R}^T}\left|\Psi\left(\bm{\omega};\Phi(o), \mathbf{a}\right) - \varphi_{\mathbf{R}|o,\mathbf{a}}(\bm{\omega})\right|^2\mathcal{N}(\bm{\omega})\mathrm{d}\bm{\omega}, \nonumber
\end{align}
where $\mathcal{N}$ is an arbitrary probability density function on $\mathbb{R}^T$, and the weighted squared norm of a function is defined as
\begin{align*}
    \|f(\cdot)\|_{\mathcal{N}}=\sqrt{\int_{\mathbb{R}^T}\left|f(\mathbf{w})\right|^2\mathcal{N}(\bm{\omega})\mathrm{d}\bm{\omega}}.
\end{align*}
Moreover, the $(o,\mathbf{a})$ pairs are sampled from the replay buffer $\mathcal{D}$. Thus, we derive the predictive loss function of \ourM{}:
\begin{align}
    & L^{\mathcal{N}}_{WSE}(\Phi,\Psi |\mathcal{D}) = \mathbb{E}_{(O,\mathbf{A})\sim\mathcal{D}}\left[L^{\mathcal{N}}_{WSE}(\Phi, \Psi|O,\mathbf{A})\right] \\
    = &\,\, \mathbb{E}_{(O,\mathbf{A})\sim\mathcal{D},\bm{\Omega}\sim\mathcal{N}}\left[ \left| \Psi\left(\bm{\Omega};\Phi(O), \mathbf{A}\right) - \varphi_{\mathbf{R}|O,\mathbf{A}}(\bm{\Omega}) \right|^2_2 \right]. \nonumber
\end{align}

In practice, since we have no access to the target values of characteristic functions of RSD-OA, we propose to optimize an upper bound on $L^{\mathcal{N}}_{WSE}$:
\begin{align}
\label{loss:WSE3}
    &\mathcal{L}^{\mathcal{N}}_{WSE}(\Phi, \Psi |\mathcal{D}) \\
    = &\,\, \mathbb{E}_{(O,\mathbf{A},\mathbf{R}) \sim \mathcal{D},\bm{\Omega}\sim\mathcal{N}}\left[ \left| \Psi\left(\bm{\Omega}; \Phi(O), \mathbf{A}\right) - e^{j\langle\bm{\Omega},\mathbf{R}\rangle} \right|^2_2 \right] \nonumber\\
    \geq &\,\, \mathbb{E}_{(O,\mathbf{A})\sim\mathcal{D},\bm{\Omega}\sim\mathcal{N}}\left[ \left| \Psi\left(\bm{\Omega}; \Phi(O), \mathbf{A}\right) - \varphi_{\mathbf{R}|O,\mathbf{A}}(\bm{\Omega}) \right|^2_2 \right] \nonumber\\
    = &\,\, L^{\mathcal{N}}_{WSE}(\Phi,\Psi |\mathcal{D}). \nonumber
\end{align}

Due to the complex form of characteristic functions in spectral space, we divide the predictor $\Psi$ into two parts $(\psi_{\cos}, \psi_{\sin})$, where $\psi_{\cos}$ estimates the real parts of characteristic functions, and $\psi_{\sin}$ estimates the imaginary parts.
Therefore, we rewrite the weighted squared error loss of \ourM{} in Equation~(\ref{loss:WSE3}) as:
\begin{align}
    \label{eq-wse}
    & \mathcal{L}^\mathcal{N}_{WSE}(\Phi,\Psi |\mathcal{D})\\
    =&\,\, \mathbb{E}_{(O,\mathbf{A},\mathbf{R}) \sim \mathcal{D},\bm{\Omega}\sim\mathcal{\mathcal{N}}} \left[\left\| \psi_{\cos} \left(\bm{\Omega}; \Phi(O), \mathbf{A}\right) - \cos\left(\langle\bm{\Omega},\mathbf{R}\rangle\right) \right\|^2_2 \right. \nonumber\\
    &\quad\quad\quad\quad\quad\quad\quad + \left. \left\| \psi_{\sin} \left(\bm{\Omega}; \Phi(O), \mathbf{A}\right) - \sin\left(\langle\bm{\Omega},\mathbf{R}\rangle\right) \right\|^2_2 \right] \nonumber.
\end{align}

\subsubsection{Implementation of \ourM{}}

For implementation, we apply a 3-layer MLP as the predictor $\Psi$ after the encoder $\Phi$. We use last two layers to output $\psi_{\cos}$ and $\psi_{\sin}$, respectively. The architecture of \ourM{} is in Appendix~\ref{app-imp}.
Moreover, to compute predicted and target values in Equation~(\ref{eq-wse}), we use a Gaussian distribution as $\mathcal{N}$ in characteristic functions of RSD-OA, and we conduct an experiment of hyperparameters (i.e., the mean $\bm{\mu}$ and the standard deviation $\bm{\sigma}$) of $\mathcal{N}$ in Appendix~\ref{app-hs-as}. Based on the empirical results, we select the standard Gaussian distribution $\mathcal{N}(\bm{0},I)$.
We also consider sampling $\kappa$ data points $\{ \bm{\omega}_i \}_{i=1}^{\kappa}$ from $\mathcal{N}$ and compute target values $\{ \varphi_{\mathbf{R}|o,\mathbf{a}} \left( \bm{\omega}_i \right) \}_{i=1}^{\kappa}$ for each data instance $(o, \mathbf{a}, \mathbf{r})$.
For the predicted values $\{ \Psi\left( \bm{\omega}_i; \Phi(o), \mathbf{a} \right) \}_{i=1}^{\kappa}$, we use the encoder $\Phi$, the predictor $\Psi=(\psi_{\cos}, \psi_{\sin})$, and inputs $\left(o, \mathbf{a}, \{ \bm{\omega}_i \}_{i=1}^{\kappa} \right)$ to predict the real and imaginary values of characteristic functions of RSD-OA.
In addition, we conduct an experiment to select $\kappa=256$ in Appendix~\ref{app-hs-as}.

\subsection{\ourM{-T}}
\label{subsec-5.3}

We extend \ourM{} by proposing an enhanced version, namely \ourM{-T}, which significantly improves the robustness of \ourM{} through the more effective optimization objective and architecture.

\subsubsection{Optimization Objective of \ourM{-T}}
\label{subsec-5.3.1}

We notice that the weighted mean squared error may lose the phase information of characteristic functions of RSD-OA. This is because the angle between two vectors may be large even if the corresponding Euclidean distance is small, e.g., the vectors $(0,0.1)$ and $(0,-0.1)$. 
Therefore, we propose to estimate the phase in spectral space by using the cosine similarity, and then we introduce the spectral cosine similarity as follows.
\begin{align}
\label{loss:CS1}
    L^{\mathcal{N}}_{SCS}& (\Phi, \Psi |o, \mathbf{a}) \\
    =& - CosSim^{\mathcal{N}}(\Psi(\cdot;\Phi(o),\mathbf{a}), \varphi_{\mathbf{R}|o,\mathbf{a}}(\cdot)) \nonumber\\
    =& -Re\left(\frac{\langle \Psi(\cdot;\Phi(o),\mathbf{a}), \varphi_{\mathbf{R}|o,\mathbf{a}}(\cdot)\rangle_\mathcal{N}}{\|\Psi(\cdot;\Phi(o),\mathbf{a})\|_\mathcal{N}\cdot\|\varphi_{\mathbf{R}|o,\mathbf{a}}(\cdot)\|_\mathcal{N}}\right), \nonumber
\end{align}
where $\left\langle f(\cdot), h(\cdot)\right\rangle_{\mathcal{N}}$ denotes the weighted inner product of the functions $f(\cdot)$ and $h(\cdot)$ as follows.
\begin{align*}
    \left\langle f(\cdot), h(\cdot)\right\rangle_{\mathcal{N}} = \int_{\mathbb{R}^T}f(\bm{\omega})\overline{h(\bm{\omega})}\mathcal{N}(\bm{\omega})\mathrm{d}\bm{\omega}.
\end{align*}

Notice that the term $\|\varphi_{\mathbf{R}|o,\mathbf{a}}(\cdot)\|_\mathcal{N}$ is independent of $\Phi$ and $\Psi$ when $(o, \mathbf{a})$ is given. Therefore, we can omit such term and rewrite the loss function in Equation~(\ref{loss:CS1}) as:
\begin{equation}
    L^{\mathcal{N}}_{SCS}(\Phi, \Psi |o, \mathbf{a}) = -Re\left(\frac{\langle \Psi(\cdot;\Phi(o),\mathbf{a}), \varphi_{\mathbf{R}|o,\mathbf{a}}(\cdot)\rangle_\mathcal{N}}{\|\Psi(\cdot;\Phi(o),\mathbf{a})\|_\mathcal{N}}\right).
\end{equation}
We also sample the $(o,\mathbf{a})$ pairs from the replay buffer $\mathcal{D}$. Then, we derive the predictive loss function of \ourM{-T}:
\begin{align}
\label{loss:CS3}
    & \mathcal{L}^{\mathcal{N}}_{SCS}(\Phi,\Psi |\mathcal{D}) = \mathbb{E}_{(O,\mathbf{A})\sim\mathcal{D}}\left[L^{\mathcal{N}}_{SCS}(\Phi, \Psi|O,\mathbf{A})\right]\\
    = &\,\, \mathbb{E}_{(O,\mathbf{A})\sim\mathcal{D},\bm{\Omega}\sim\mathcal{N}}\left[ -Re\left(\frac{ \Psi(\bm{\Omega};\Phi(O),\mathbf{A})\overline{ \varphi_{\mathbf{R}|O,\mathbf{A}}(\bm{\Omega})}}{\|\Psi(\cdot;\Phi(O),\mathbf{A})\|_\mathcal{N}}\right) \right] \nonumber\\
    = &\,\, \mathbb{E}_{(O,\mathbf{A},\mathbf{R})\sim\mathcal{D}, \bm{\Omega}\sim\mathcal{N}}\left[  -Re\left(\frac{ \Psi(\bm{\Omega};\Phi(O),\mathbf{A})e^{-i\langle \bm{\Omega}, \mathbf{R}\rangle}}{\|\Psi(\cdot;\Phi(O),\mathbf{A})\|_\mathcal{N}}\right)\right]. \nonumber
\end{align}

By dividing the predictor $\Psi$ into $(\psi_{\cos}, \psi_{\sin})$, we rewrite the proposed loss function in Equation~(\ref{loss:CS3}) as:
\begin{align}
    \label{eq-cos}
    & \mathcal{L}^{\mathcal{N}}_{SCS}(\Phi,\Psi |\mathcal{D}) \\
    =&\,\, \mathbb{E}_{(O,\mathbf{A},\mathbf{R}) \sim \mathcal{D}}
    \left[\frac{
        \begin{aligned}
            \mathbb{E}_{\bm{\Omega}\sim\mathcal{N}}\left[\psi_{\cos} \left(\bm{\Omega}; \Phi(O), \mathbf{A}\right)\cos\left(\langle\bm{\Omega},\mathbf{R}\rangle\right)\right.\\
            \left.+ \psi_{\sin} \left(\bm{\Omega}; \Phi(O), \mathbf{A}\right) \sin\left(\langle\bm{\Omega},\mathbf{R}\rangle\right)\right]
        \end{aligned}}{\sqrt{
        \begin{aligned}
            \mathbb{E}_{\bm{\Omega}\sim\mathcal{N}}\left[\|\psi_{\cos} \left(\bm{\Omega}; \Phi(O), \mathbf{A}\right)\|^2\right.\\
            +\left.\|\psi_{\sin} \left(\bm{\Omega}; \Phi(O), \mathbf{A}\right)\|^2\right]
        \end{aligned}
    }}
    \right]. \nonumber
\end{align}

Furthermore, considering the advantages of the weighted squared error in Equation~(\ref{eq-wse}) and the spectral cosine similarity in Equation~(\ref{eq-cos}) in capturing the amplitude and phase of RSD-OA, respectively, we propose:
\begin{equation}
\label{eq-comb}
\mathcal{L}^{\mathcal{N}}(\Phi, \Psi|\mathcal{D}) = \mathcal{L}^{\mathcal{N}}_{WSE}(\Phi,\Psi |\mathcal{D}) + \lambda\mathcal{L}^{\mathcal{N}}_{SCS}(\Phi,\Psi |\mathcal{D}),
\end{equation}
where $\lambda$ represents the trade-off between $\mathcal{L}^{\mathcal{N}}_{WSE}$ and $\mathcal{L}^{\mathcal{N}}_{SCS}$.
Based on such the combination, we can accurately estimate the characteristic functions of RSD-OA in spectral space. We provide the ablation studies in Section~\ref{subsec-ablation}.

In the training process, we update the encoder $\Phi$ and the predictor $\Psi$ through the auxiliary loss $\mathcal{L}^{\mathcal{N}}(\Phi, \Psi|\mathcal{D})$, and use the learned encoder $\Phi$ for the VRL tasks.
The training procedure of \ourM{-T} are shown in Algorithm~\ref{alg-cresp}.



\subsubsection{Architecture of \ourM{-T}}
\label{subsec-5.3.2}

In order to further improve representations, \ourM{-T} models temporal dependencies of the predefined subsequent action sequence $\mathbf{a}$ by using the transformer architecture.
Specifically, for any given data instance $(o, \mathbf{a}, \mathbf{r})$ and a sample $\bm{\omega}$ from $\mathcal{N}$, we first concatenate $(\Phi(o), \bm{\omega}, a_1,a_2,\cdots,a_T)$ as a sequence. We can view each item in the sequence as a token and use a linear layer to project each token to the embedding space. Then, we use a transformer architecture to process the token embeddings with the corresponding position embeddings.
According to the aforementioned process,  we can model the dependencies among $\Phi(o)$, $\bm{\omega}$, and $\textbf{a}$ by using attention mechanisms in \ourM{-T}. 
Moreover, \ourM{-T} uses last two separate blocks in the transformer as $\psi_{\cos}$ and $\psi_{\sin}$, instead of using last two separate layers in \ourM{}. The architecture of \ourM{-T} are in Figure~\ref{fig-arc}.

\begin{algorithm}[tb]
   \caption{\ourM{-T}}
   \label{alg-cresp}
    \begin{algorithmic}
        \STATE {Initialize a replay buffer $\mathcal{D}$, a policy $\pi$, a representation $\Phi$, and a function approximator $\Psi$}
        \FOR {each iteration}
            \FOR {$e$ in $\mathcal{E}$}
                \FOR {each environment step $t$}
                    \STATE Execute an action $a_t\sim \pi(\cdot|\Phi(o_t))$
                    \STATE Receive a transition $o_{t+1}, r_{t+1}\sim p^e(\cdot|o_t,a_t)$
                    \STATE Record trajectories $\{(o_{t}, \mathbf{a}_{t}, \mathbf{r}_{t+1})\}$ in $\mathcal{D}$
                \ENDFOR
            \ENDFOR
            \FOR {each gradient step}
                \STATE Sample partial trajectories from $\mathcal{D}$
                \STATE Update the representation: \textcolor{blue}{$\mathcal{L}^{\mathcal{N}}(\Phi, \Psi|\mathcal{D})$}
                \STATE Update the policy: $\mathcal{L}_{\text{RL}}(\pi)$
            \ENDFOR
        \ENDFOR
    \end{algorithmic}
\end{algorithm}

\section{Experiments}
\label{sec-6}
In this section, we evaluate the generalization performance on unseen test environments with visual distractions.
We first provide experiment settings in Section~\ref{subsec-6.1}.
Then, we conduct experiments to show main results in Section~\ref{subsec-6.2}.
We also show the impact of task-irrelevant information from observation transition dynamics on representation learning in Section~\ref{subsec-6.3}.
Moreover, we evaluate representations learned by different methods and visualize their task-relevant and -irrelevant information in Section~\ref{subsec-6.4}. 
Finally, we conduct experiments for hyperparameter selection in Section~\ref{subsec-hyperselect} and ablation in Section~\ref{subsec-ablation}.
In this journal manuscript, we extend the experiments of \ourM{-T} in Sections~\ref{subsec-6.2}, ~\ref{subsec-6.4}, ~\ref{subsec-hyperselect}, and~\ref{subsec-ablation} to show the robustness of the enhanced version.

\subsection{Experiment Settings}
\label{subsec-6.1}
In this section, we introduce the environment settings, baselines, experiment parameters, and network details.
See Appendix~\ref{app-nd} for more details.

\textbf{Environment Settings}:
In the conference version, we use two training environments for evaluation. These two environments have different dynamic visual distractions, but the types of distractions are the same. 
To further understand the performance improvement, we extend experiments to two challenging environment settings:
(1) \textit{One/Single training environment setting}:
We train the agents in an environment and evaluate in unseen test environments.
Notice that in a single training environment, we cannot rely on multiple training environments to disentangle domain-specific or domain-invariant information~\cite{iccv/PengBXHSW19}. Thus, the agents can hardly generalize to unseen environments.
(2) \textit{Three/Multiple training environment setting}:
We train the agent in three environments, which have different dynamic distractions but same types of these distractions. Then, we evaluate the agent in unseen test environments.
Notice that there are discrepancies not only
across the training and test environments, but also
in training environments themselves~\cite{corr/abs-2204-13091}. Moreover, it is more challenging to encode only task-relevant information from observations with multiple visual distractions.

\begin{table}
\renewcommand{\arraystretch}{1.1}
  \caption{Performance of different methods trained on one, two, and three training environments. All methods are evaluated on unseen test environments after 500K steps, and results are averaged over 6 DCS tasks. Highest mean scores and standard errors are marked in blue.}
  \label{table-result-2}
  \setlength\tabcolsep{2pt}
  \begin{tabular}{r|c|ccc}
  \toprule
  & Method & Num Envs 1 & Num Envs 2 & Num Envs 3 \\
  \midrule
  \multirow{7}{*}{\rotatebox{90}{Backgrounds}}
                & \textbf{\ourM{-T}} & $454 \pm 49$ & $\textcolor{blue}{\mathbf{680} \pm \mathbf{54}(+20\%)} $  & $\mathbf{672} \pm \mathbf{52} (+6\%)$ \\
                & \textbf{\ourM{}}  & $439\pm 62$  & $\mathbf{649} \pm \mathbf{55} (+15\%)$ & $\textcolor{blue}{\mathbf{674} \pm \mathbf{58} (+6\%)}$ \\
                & DrQ   & $\textcolor{blue}{\mathbf{499} \pm \mathbf{47}}$ & $566 \pm 53$ & $635 \pm 58$ \\
                & CURL  & $\mathbf{467} \pm \mathbf{56}$ & $533\pm 85$ & $542 \pm 65$ \\
                & DBC  & $-$ & $186 \pm 30$ & $-$ \\
                & MISA & $-$ & $269 \pm 65$ & $-$ \\
                & SAC  & $223\pm 24$ & $191\pm 20$ & $195 \pm 22$ \\
  \midrule
  \multirow{7}{*}{\rotatebox{90}{Colors}}
                & \textbf{\ourM{-T}} & $\textcolor{blue}{\mathbf{461} \pm \mathbf{58}(+15\%)}$ & $\textcolor{blue}{\mathbf{560} \pm \mathbf{53}(+55\%)}$ & $\textcolor{blue}{\mathbf{544} \pm \mathbf{51}(+44\%)}$ \\
                & \textbf{\ourM{}}  & $\mathbf{444}\pm \mathbf{55} (+10\%)$  & $\mathbf{526} \pm \mathbf{80} (+46\%)$ & $\mathbf{455} \pm \mathbf{61} (+20\%)$ \\
                & DrQ   & $402 \pm 61$ & $350 \pm 64$ & $378 \pm 68$ \\
                & CURL  & $266 \pm 31$ & $361\pm 62$ & $317 \pm 61$ \\
                & DBC  & $-$ & $137 \pm 23$ & $-$ \\
                & MISA & $-$ & $147 \pm 54$ & $-$ \\
                & SAC  & $161\pm 19$ & $122\pm 21$ & $118\pm 24$ \\
  \bottomrule
  \end{tabular}
\end{table}

\textbf{Baselines}:
The baseline methods include sample-efficient methods for VRL (\textbf{CURL}~\cite{icml/LaskinSA20} and \textbf{DrQ}~\cite{iclr/YaratsKF21}),
representation learning methods in VRL (\textbf{MISA}~\cite{icml/0001LSFKPGP20} and \textbf{DBC}~\cite{iclr/0001MCGL21}),
and traditional RL methods (\textbf{SAC}~\cite{icml/HaarnojaZAL18}).
For a fair comparison, all methods do not leverage any prior environmental knowledge, such as strong augmentations designed for visual distractions~\cite{lee2020network,iclr/ZhangCDL18,icml/FanWHY0ZA21}, fine-tuning in test environments~\cite{iclr/HansenJSAAEPW21}, or environmental labels to learn invariant representations~\cite{iclr/AgarwalMCB21,l4dc/SonarPM21}.

\textbf{Experiment Parameters}:
In Section~\ref{subsec-6.2}, all experiments report the means and standard errors of the cumulative rewards for 500K environment steps. In each task of these experiments, we train the agents with six random seeds in one, two, and three training environments, respectively. 
Other experiments for verification and ablation are under the two training environment setting with three random seeds.
Moreover, we adopt the action repeat of each task from Planet~\cite{icml/HafnerLFVHLD19}, which is the common setting in VRL.
In Tables~\ref{table-result-2} and~\ref{table-result}, we report the average performance for 100 episodes and boldface the highest results with blue markers. We then plot the average cumulative rewards per task at $500$K environment steps in Figure~\ref{fig-result-dcs}, where each checkpoint is also evaluated using 100 episodes on unseen test environments. Furthermore, we provide the detailed test curves during training process in Appendix~\ref{app-ar}, where the shaded region corresponds to the standard deviation.

\textbf{Network Details}:
\ourM{} and \ourM{-T} build upon SAC and follows the network architecture of DrQ. 
We use a 3-layer feed-forward ConvNet with no residual connection as the encoder. Then, we apply three fully connected layers with hidden size 1024 for actor and critic, respectively.
We also use the random cropping for image pre-processing proposed by DrQ as a weak augmentation without prior knowledge of test environments. To predict characteristic functions of RSD-OA, we use the GPT~\cite{radford2018improving} architecture with three transformer blocks and two attention heads in \ourM{-T}. For the real and imaginary parts of characteristic functions, we apply last two separate blocks. In contrast, we use a 3-layer MLP with hidden size 1024 to predict characteristic functions of RSD-OA in \ourM{}, and we apply last two separate layers to predict the real and imaginary parts. Details of all architectures are in Appendix~\ref{app-imp}.

\begin{figure}
  \centering

  \includegraphics[width=4.4cm]{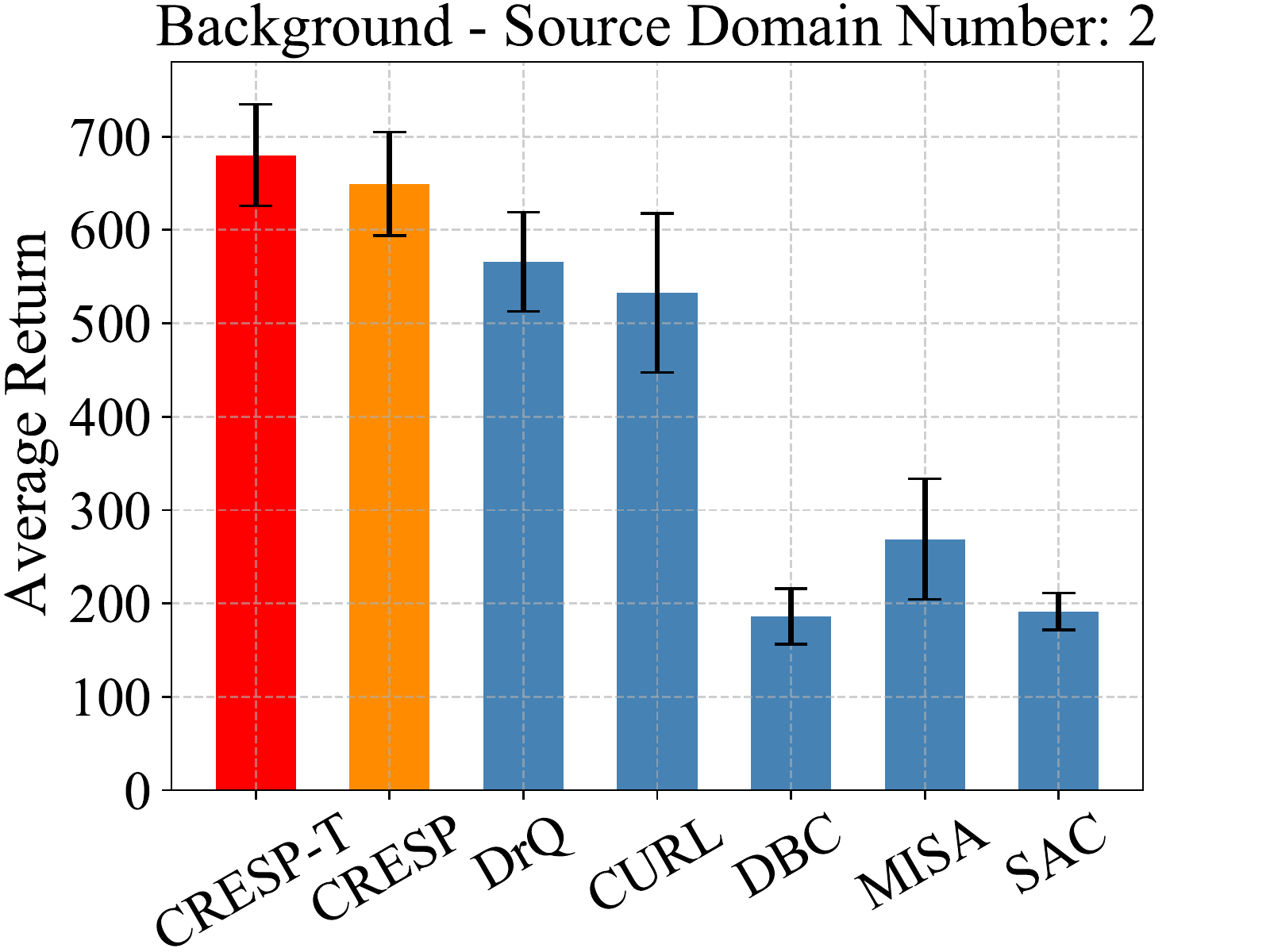}
  \includegraphics[width=4.4cm]{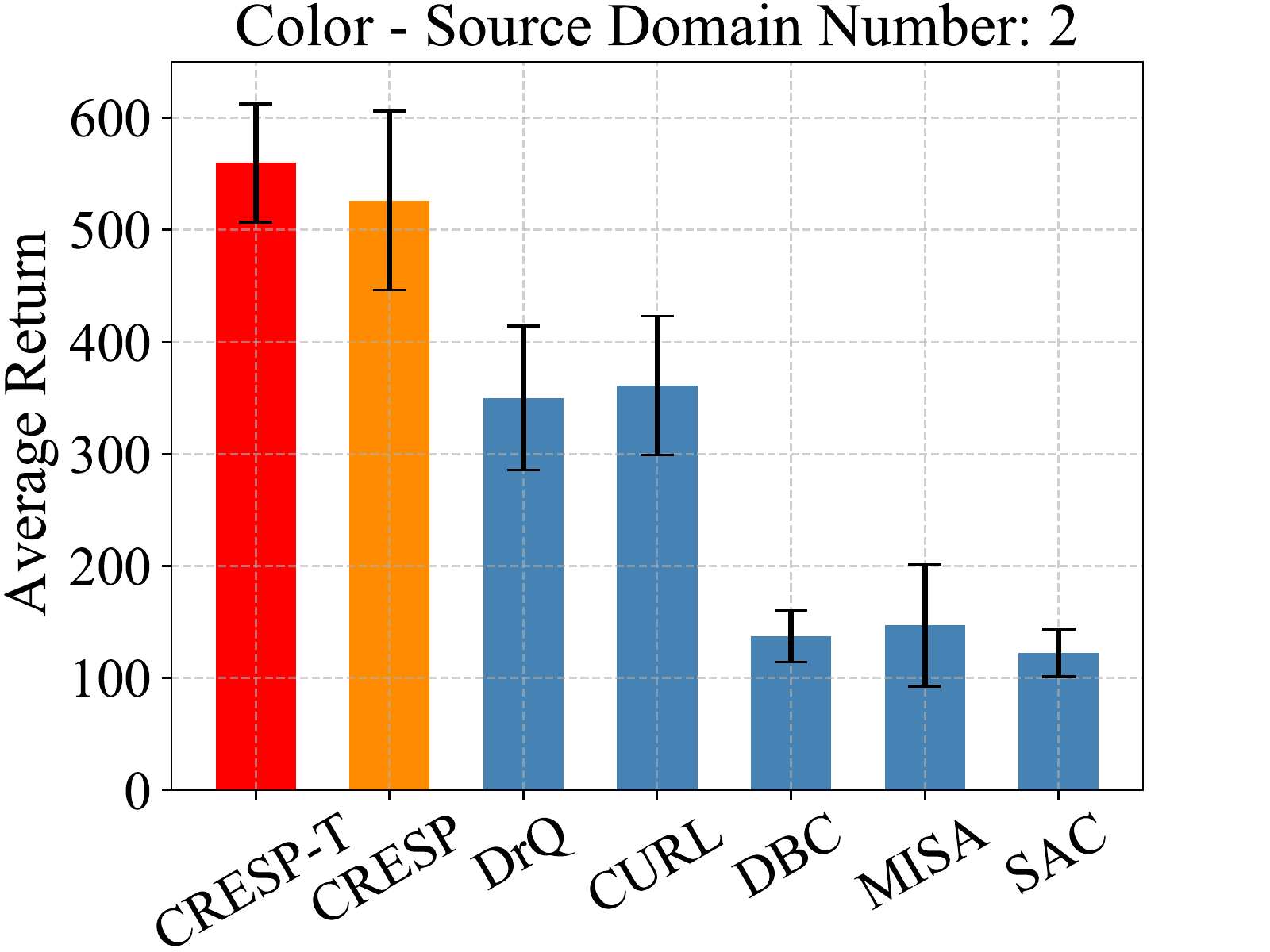}
  
  \caption{We report mean and standard error results. All results are averaged over 6 DCS tasks with 6 random seeds at 500K steps.}
  \label{fig-bar-avg}
\end{figure}

\begin{figure}
  \includegraphics[width=8.5cm]{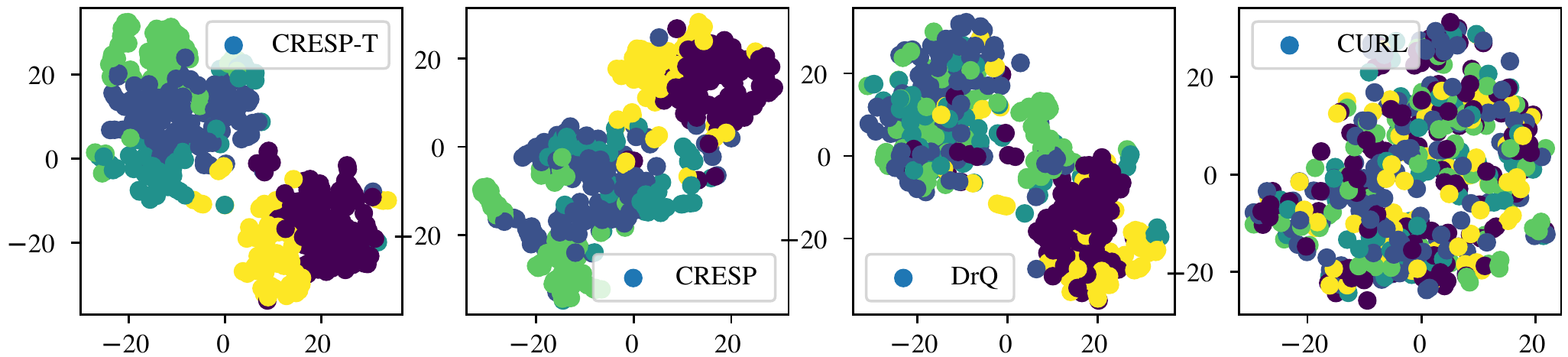}
  \vskip 0.05in
  
  \includegraphics[width=8.5cm]{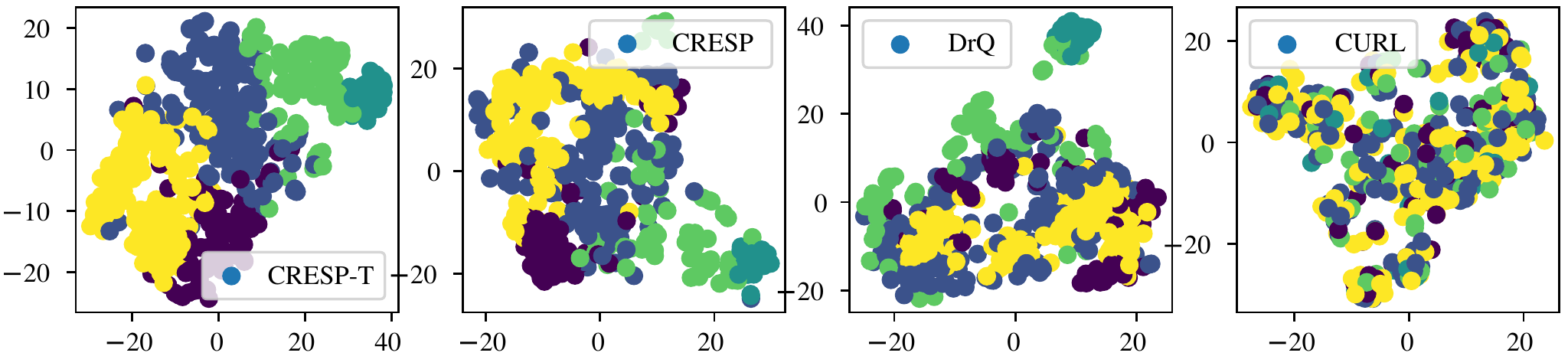}
  \caption{t-SNE visualization of learned representations. \ourM{} and \ourM{-T} correctly group different observations from similar latent states into the same color. The first row is in Cartpole-swingup task, and the second row is in Cheetah-run task.}
  \label{fig-tsne}
\end{figure}

\begin{figure*}[ht]
  \centering
  \includegraphics[width=5.8cm]{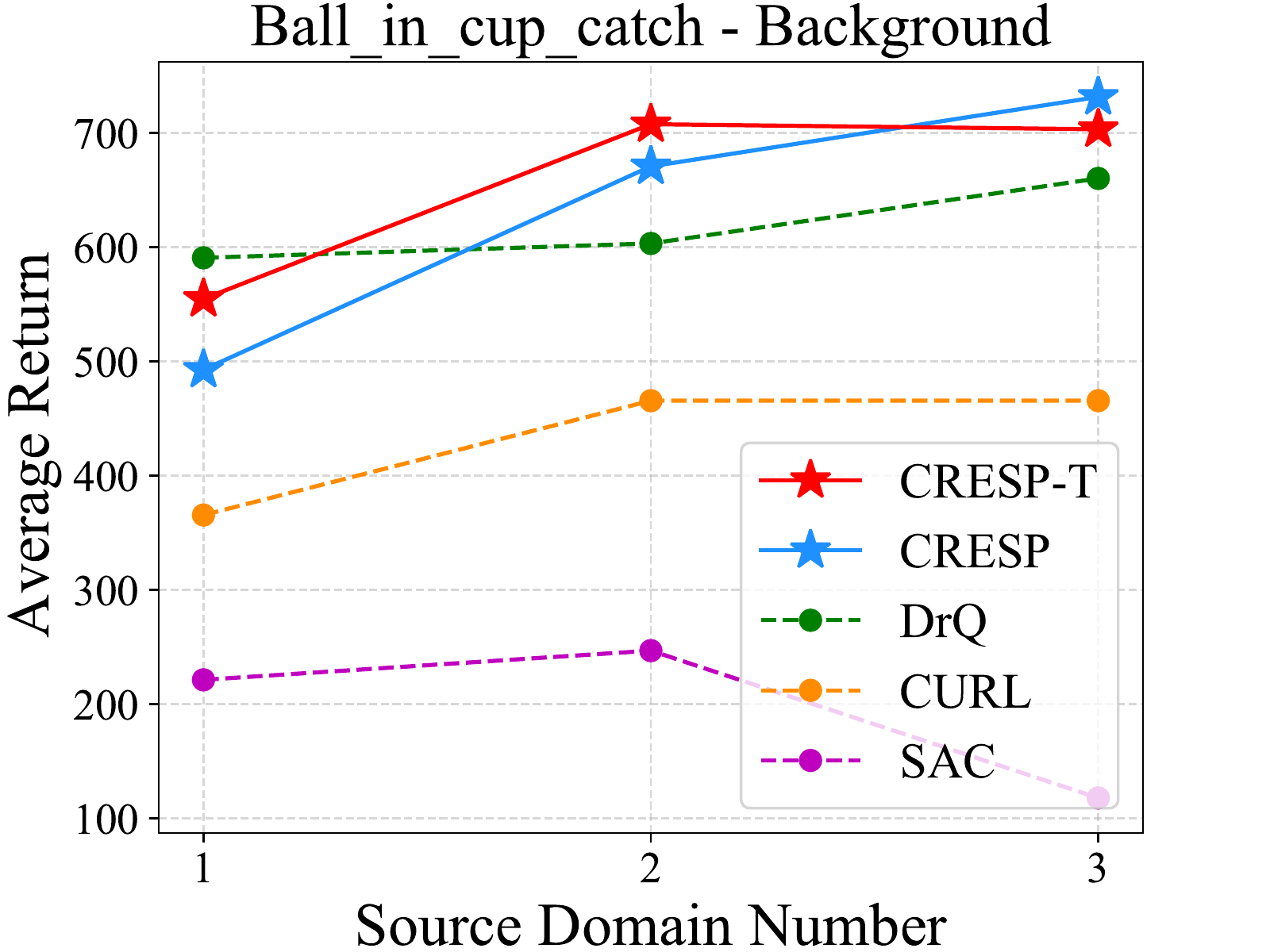}
  \includegraphics[width=5.8cm]{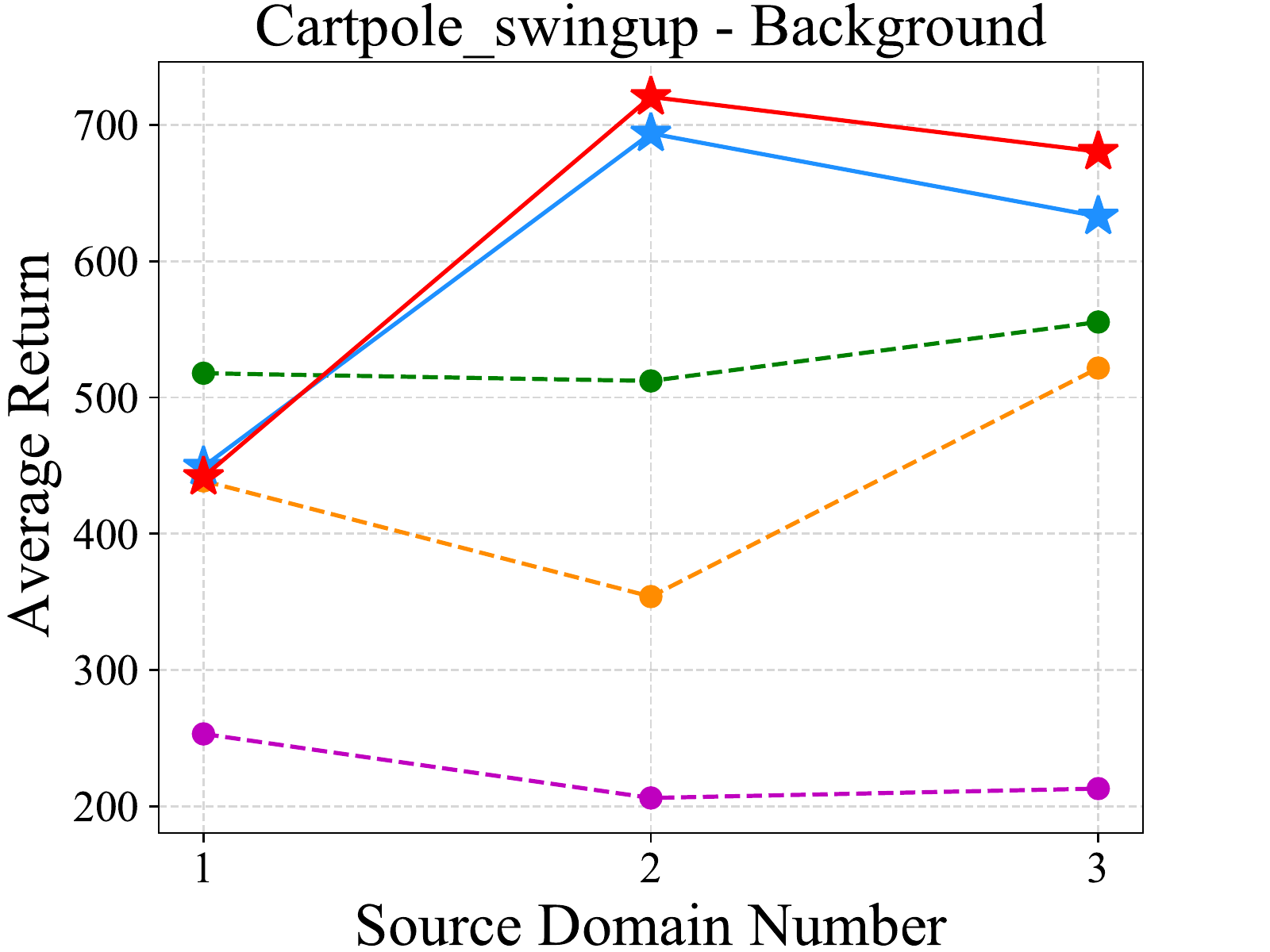}
  \includegraphics[width=5.8cm]{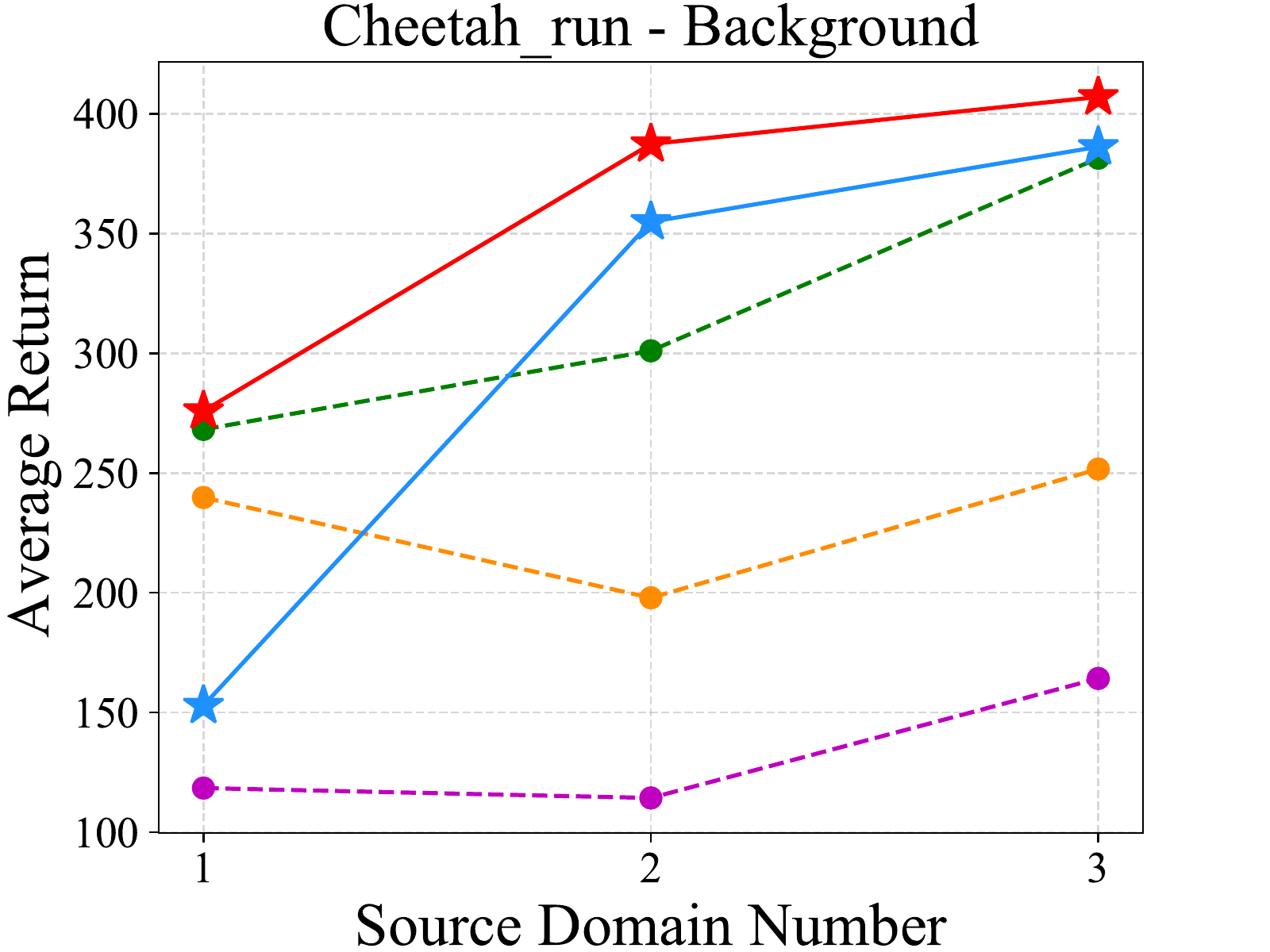}
  \vskip 0.05in
  
  \includegraphics[width=5.8cm]{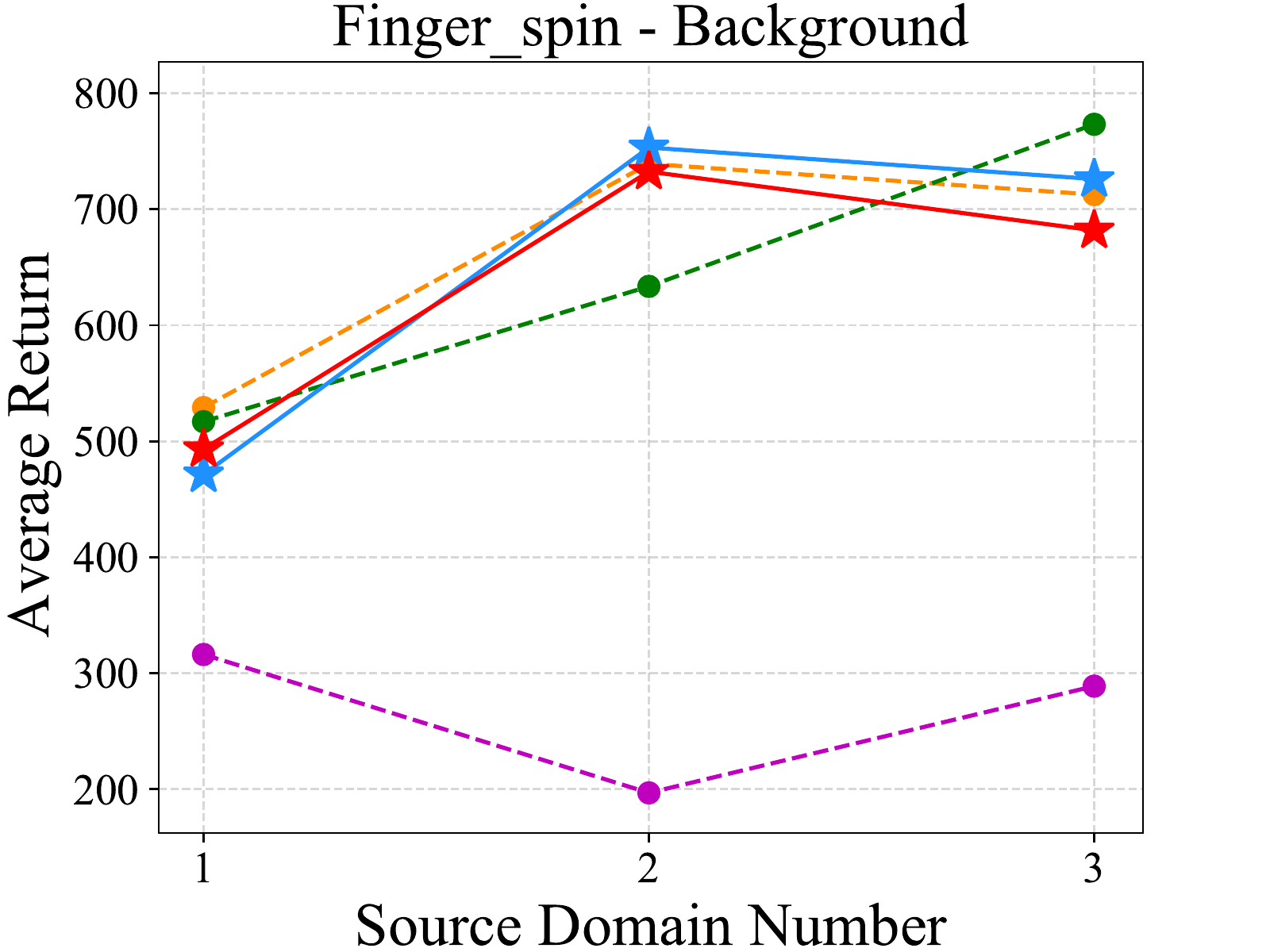}
  \includegraphics[width=5.8cm]{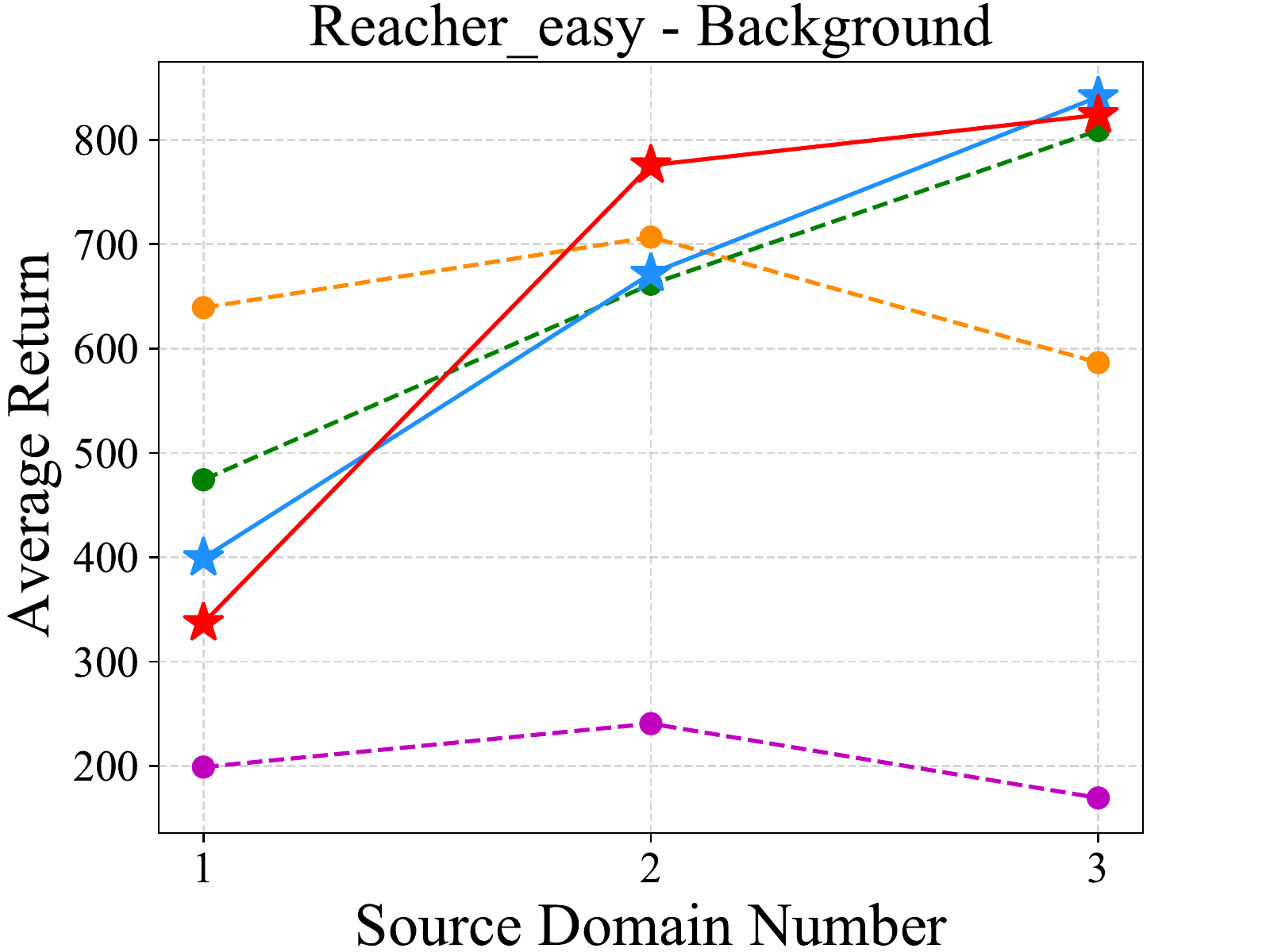}
  \includegraphics[width=5.8cm]{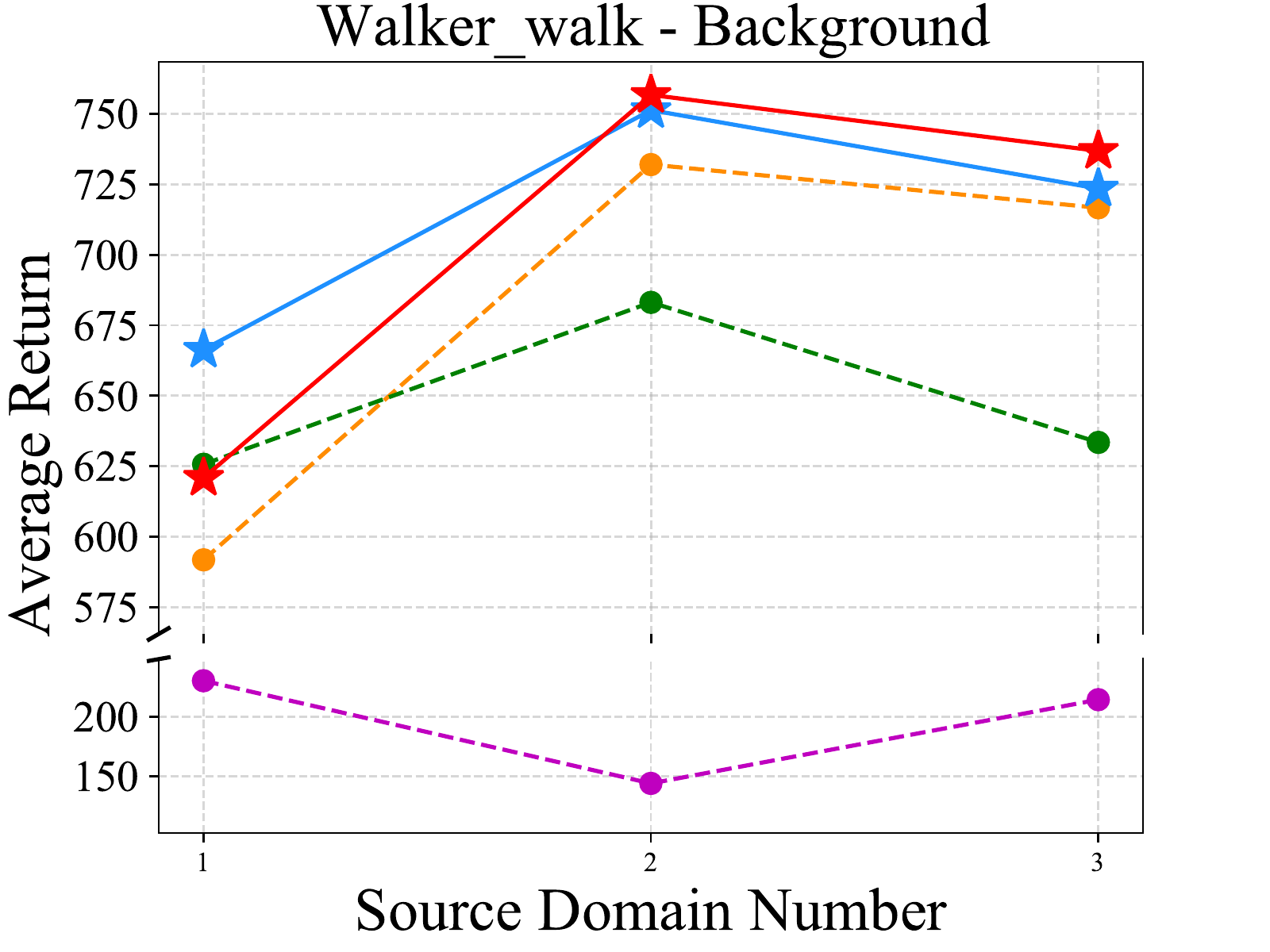}
  \vskip 0.05in
  
  \includegraphics[width=5.8cm]{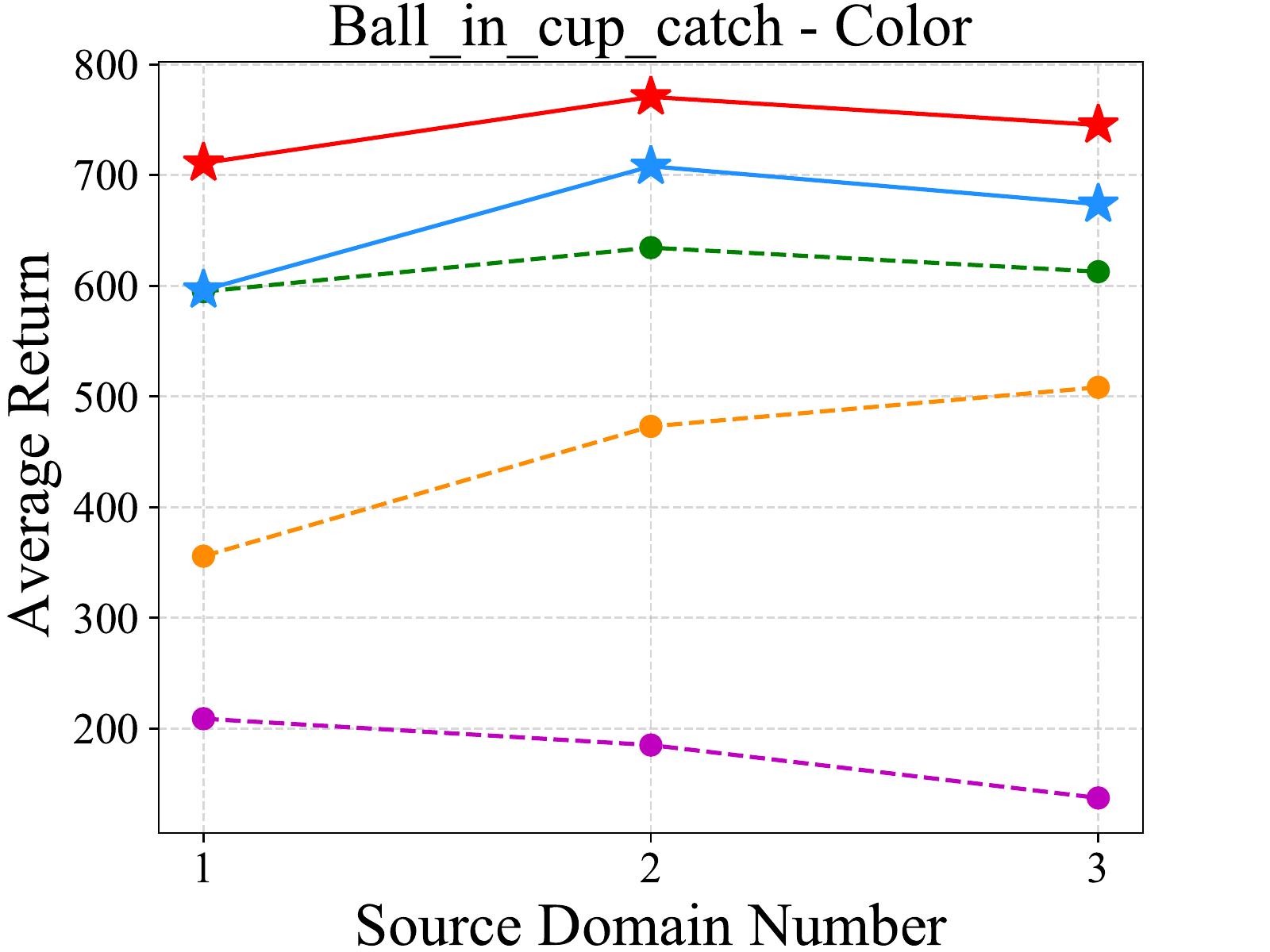}
  \includegraphics[width=5.8cm]{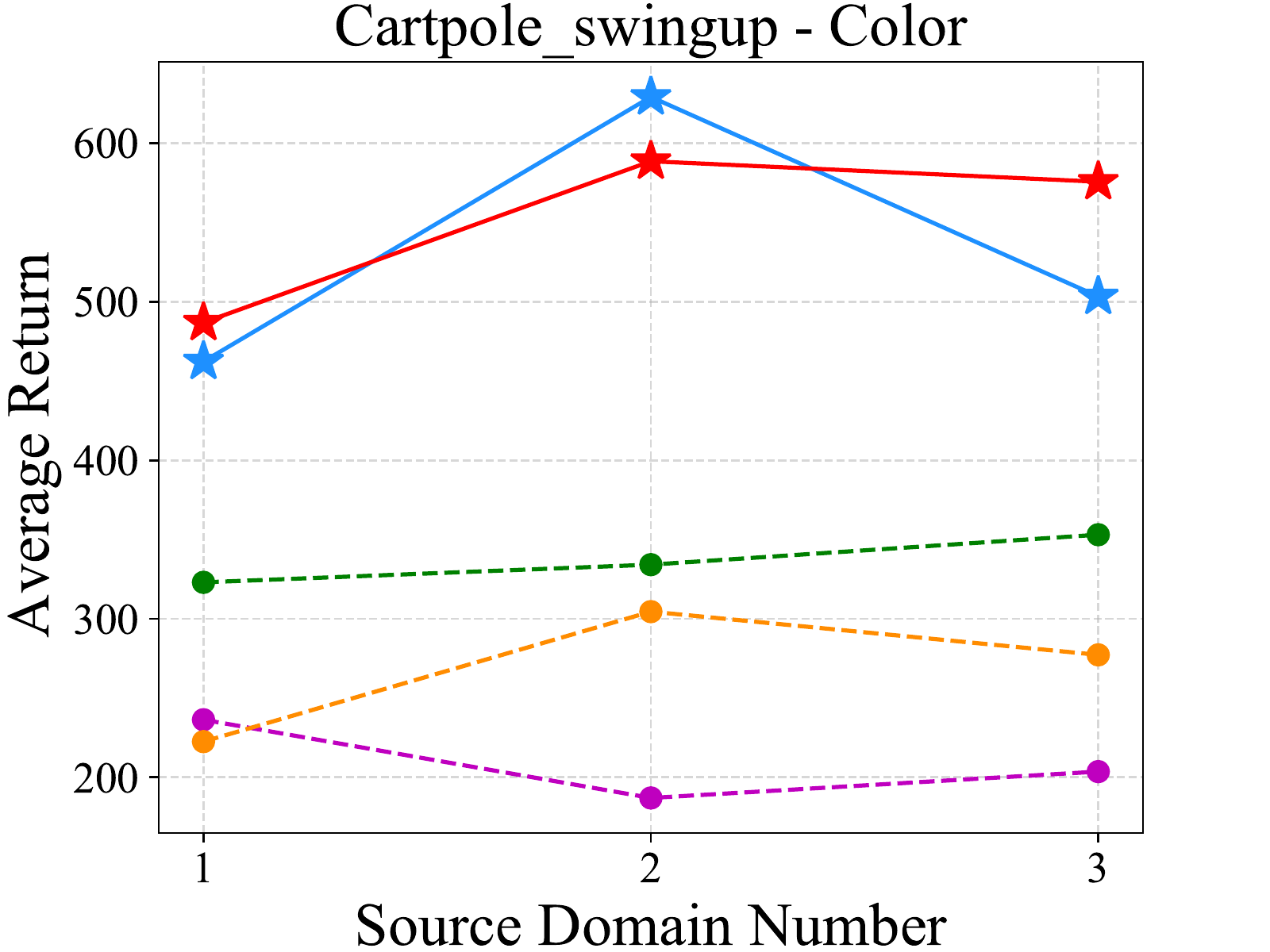}
  \includegraphics[width=5.8cm]{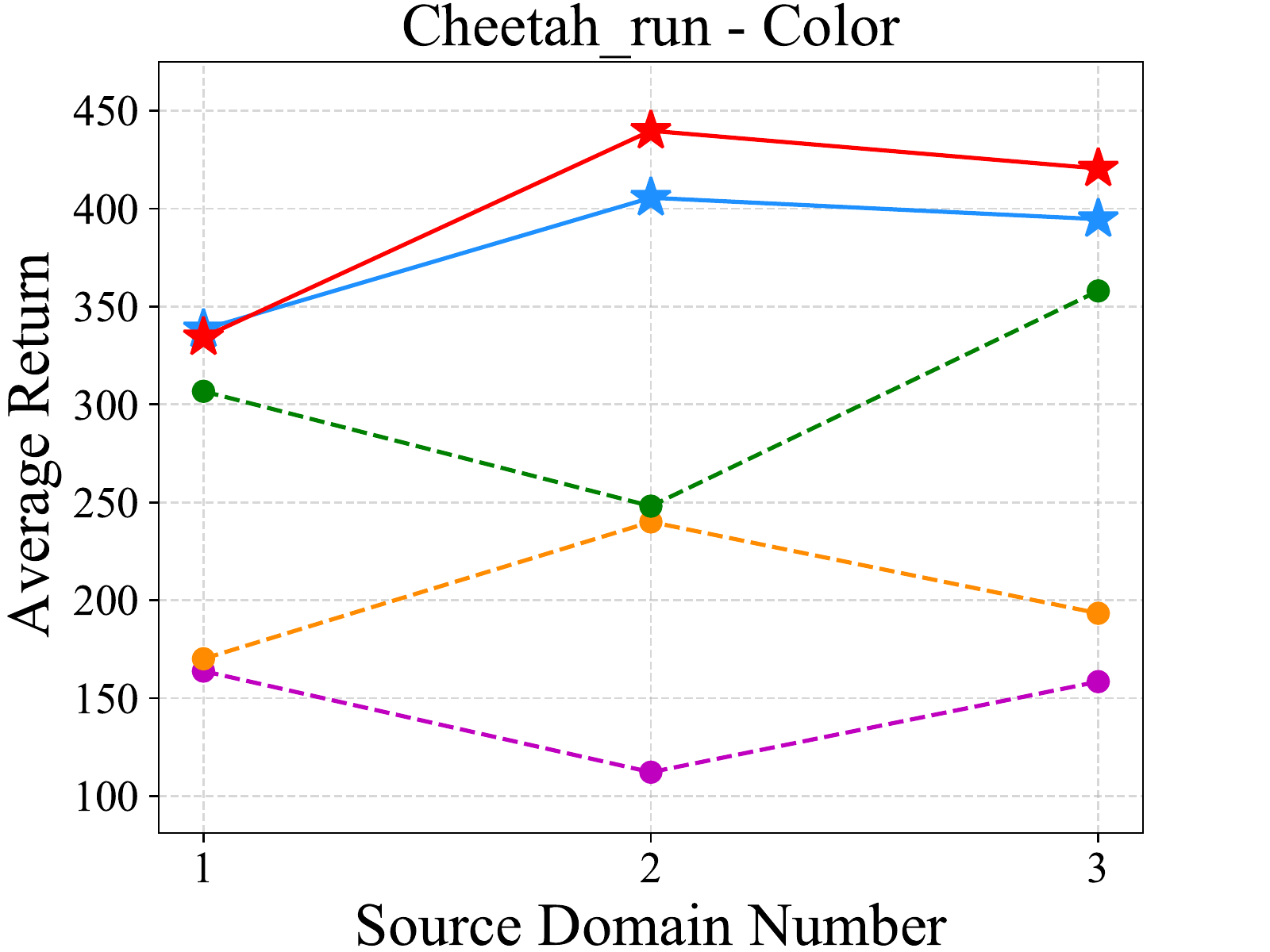}
  \vskip 0.05in
  
  \includegraphics[width=5.8cm]{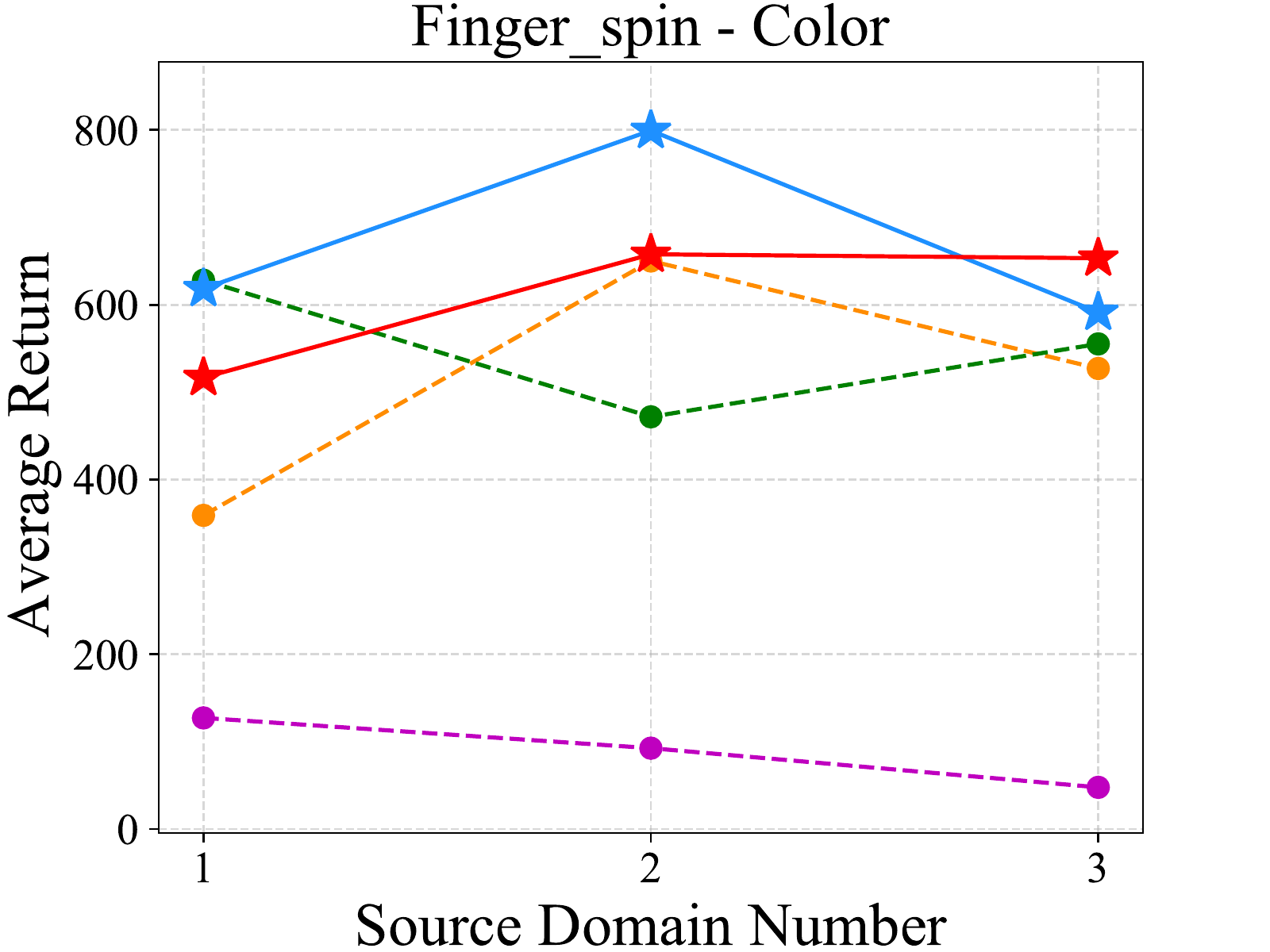}
  \includegraphics[width=5.8cm]{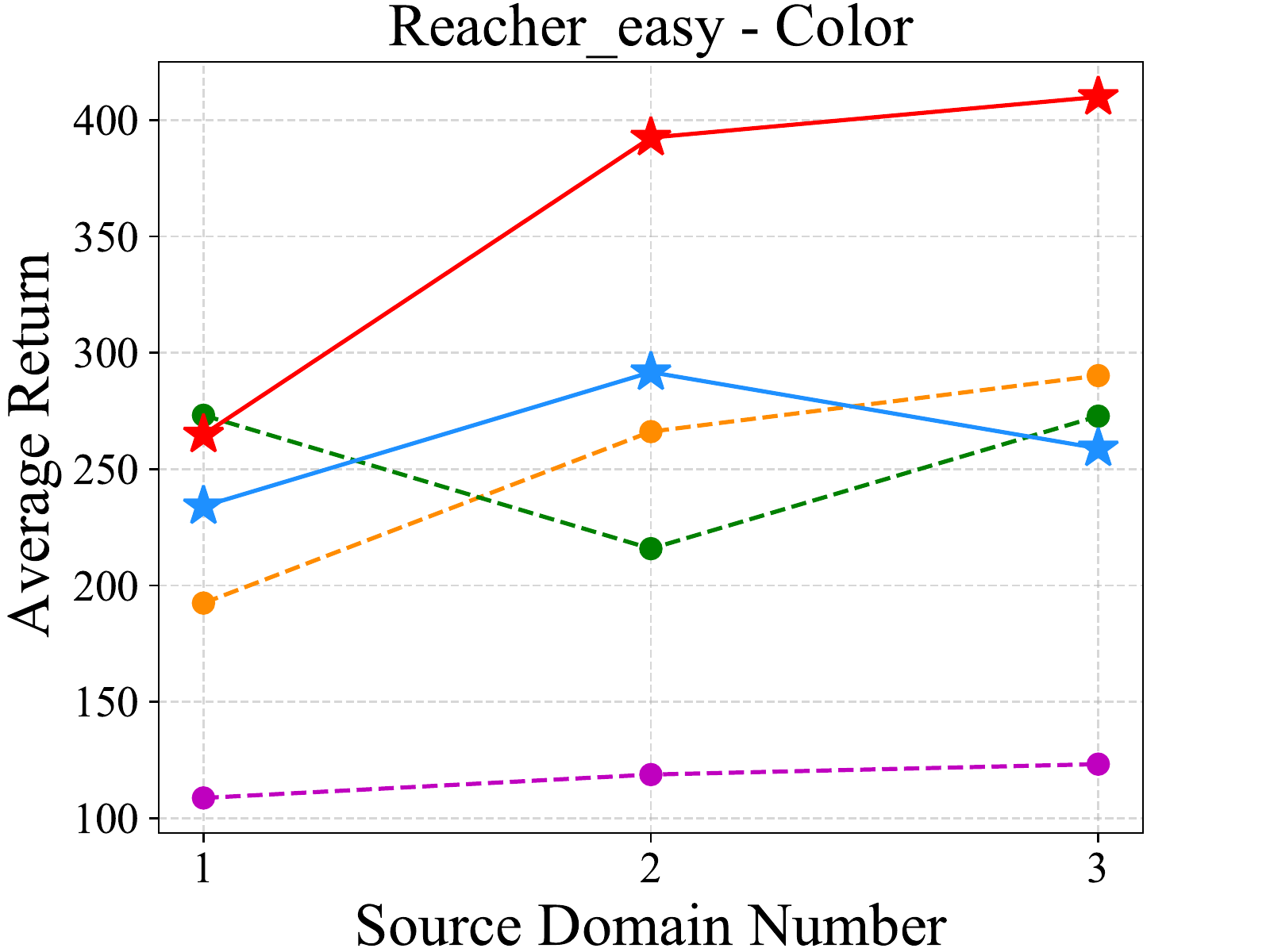}
  \includegraphics[width=5.8cm]{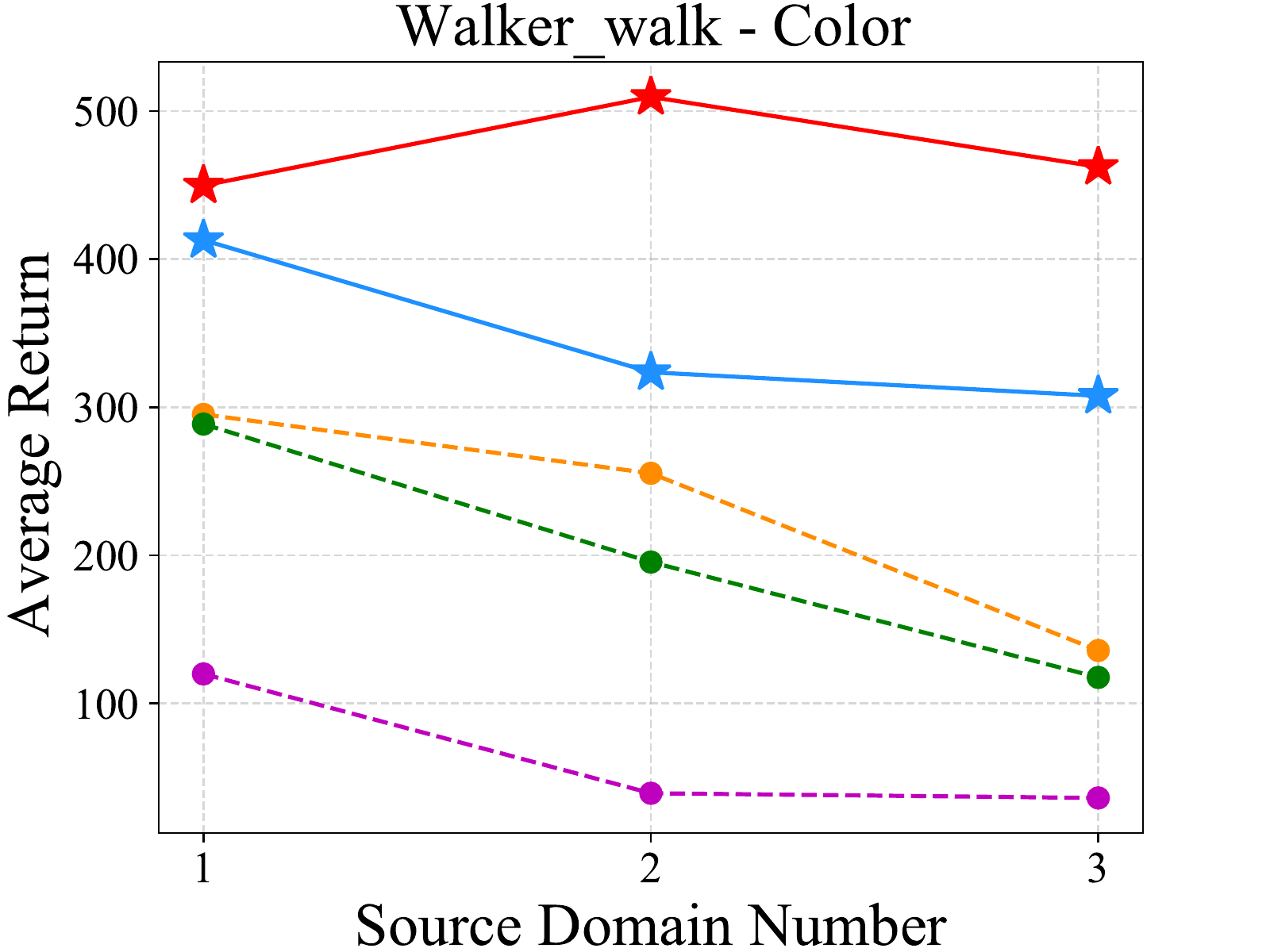}
  \caption{Results of six DCS tasks with dynamic distractions at 500K environment steps. Each checkpoint denotes the mean return over 6 trials and each return is the mean over 100 episodes on unseen test environments.}
\label{fig-result-dcs}
\end{figure*}

\subsection{Main Results under Visual Distractions}
\label{subsec-6.2}
We conduct extensive experiments on six tasks of the Distracting Control Suite (DCS)~\cite{corr/abs-2101-02722} to evaluate our approaches (\ourM{}, \ourM{-T}) and baselines.
Notice that we separately apply two types of visual distractions: (1) dynamic backgrounds and (2) dynamic colors of the objects (see Appendix~\ref{app-imp} for details).

To demonstrate that \ourM{} can improve the generalization performance, we conduct experiments under the two training environment setting and show results in Figure~\ref{fig-bar-avg} and the second column of Table~\ref{table-result-2}. 
Such evaluation results averaged over six DCS tasks show that \ourM{} and \ourM{-T} achieve average gains of $+14.7\%$ and $+20.1\%$, respectively, over the best prior method in dynamic backgrounds, and they also achieve average gains of $+45.7\%$ and $+55.1\%$, respectively, in dynamic color distractions.
Thus, we can find that \ourM{-T} improves the robustness of \ourM{} even in the two training environment setting where the performance improvement of \ourM{} is significant.

To further understand the generalization ability of \ourM{} and \ourM{-T}, we extend to more challenging settings of one and three training environments, respectively. In the one training environment setting, the agents may easily encode spurious information from observations with a single distraction~\cite{iccv/PengBXHSW19}, whereas in the three training environment setting, they may have difficulty capturing the task-relevant information from observations with multiple distractions~\cite{corr/abs-2204-13091}.
We provide the average results over six DCS tasks in the first and third columns of Table~\ref{table-result-2}. See Appendix~\ref{app-ar} for the histograms.

Results under the one training environment setting show that in dynamic color distractions, \ourM{} and \ourM{-T} achieve the average performance improvements by $+10.4\%$ and $+14.7\%$ margins, respectively. 
Since \ourM{} and \ourM{-T} achieve the better performance improvements in dynamic color distractions than those in dynamic backgrounds, and there is usually more spurious information from observations with dynamic backgrounds than those with dynamic color distractions, we can find that \ourM{} and \ourM{-T} are insensitive to spurious information.

Results under the three training environment setting show that \ourM{} and \ourM{-T} achieve the average gains of $+6.1\%$ and $+5.8\%$, respectively, in dynamic backgrounds, and they also achieve the average gains of $+20.4\%$ and $+43.9\%$, respectively, in dynamic color distractions. 
Since it gains an average of $+23.5\%$ more return over \ourM{} under the three training environment setting with dynamic color distractions, \ourM{-T} is more robust than \ourM{} in extracting task-relevant information from observations with multiple dynamic distractions.

In addition, in Figure~\ref{fig-result-dcs}, we provide the detailed performance of all methods under one, two, and three training environment settings, respectively. We also provide the average results for one, two, and three training environment settings per task in Table~\ref{table-result}.
See Appendix~\ref{app-ar} for more discussions and detailed test curves.

\begin{table*}
\renewcommand{\arraystretch}{1.2}
  \caption{We report the mean and standard error results with dynamic distractions at 500K steps. All means and standard errors are averaged over one, two, and three training environment settings.}
  \label{table-result}
  \setlength\tabcolsep{4pt}
  \begin{tabular}{r|c|cccccc}
  \toprule
  & Method & B-catch & C-swin & C-run & F-spin & R-easy & W-walk \\
  \midrule
  \multirow{5}{*}{\rotatebox{90}{Backgrounds}} & \textbf{\ourM{-T}} & $\textcolor{blue}{\mathbf{655} \pm \mathbf{23}\,\,(+6\%)}$ & $\textcolor{blue}{\mathbf{614} \pm \mathbf{7}\,\,(+16\%)}$ & $\textcolor{blue}{\mathbf{357} \pm \mathbf{12}\,\,(+13\%)}$ & $636 \pm 33$ & $\mathbf{646} \pm \mathbf{18}$ & $\mathbf{705} \pm \mathbf{25}\,\,(+4\%)$ \\
                & \textbf{\ourM{}}  & $\mathbf{632} \pm \mathbf{28}\,\,(+2\%)$ & $\mathbf{592} \pm \mathbf{14}\,\,(+12\%)$ & $298 \pm 11$ & $\mathbf{650} \pm \mathbf{49}$ & $637 \pm 11$ & $\textcolor{blue}{\mathbf{714} \pm \mathbf{12}\,\,(+5\%)}$ \\
                & DrQ  & $618 \pm 21$ & $528 \pm 12$ & $\mathbf{317} \pm \mathbf{7}$ & $641 \pm 34$ & $\textcolor{blue}{\mathbf{648} \pm \mathbf{27}}$ & $647 \pm 32$ \\
                & CURL & $432 \pm 50$ & $438 \pm 9$ & $230 \pm 6$ & $\textcolor{blue}{\mathbf{660} \pm \mathbf{39}}$ & $644 \pm 31$ & $680 \pm 22$ \\
                & SAC  & $196 \pm 30$ & $224 \pm 2$ & $132 \pm 2$ & $284 \pm 38$ & $203 \pm 13$ & $196 \pm 6$ \\
  \midrule
  \multirow{5}{*}{\rotatebox{90}{Colors}} & \textbf{\ourM{-T}} & $\textcolor{blue}{\mathbf{742} \pm \mathbf{19}\,\,(+21\%)}$ & $\textcolor{blue}{\mathbf{550} \pm \mathbf{9}\,\,(+63\%)}$ & $\textcolor{blue}{\mathbf{398} \pm \mathbf{9}\,\,(+31\%)}$ & $\mathbf{609} \pm \mathbf{18}\,\,(+10\%)$ & $\textcolor{blue}{\mathbf{356} \pm \mathbf{9}\,\,(+40\%)}$ & $\textcolor{blue}{\mathbf{474} \pm \mathbf{48}\,\,(+107\%)}$ \\
                & \textbf{\ourM{}}  & $\mathbf{659} \pm \mathbf{24}\,\,(+7\%)$ & $\mathbf{532} \pm \mathbf{13}\,\,(+58\%)$ & $\mathbf{379} \pm \mathbf{9}\,\,(+25\%)$ & $\textcolor{blue}{\mathbf{670} \pm \mathbf{35}\,\,(+21\%)}$ & $\mathbf{262} \pm \mathbf{15}\,\,(+3\%)$ & $\mathbf{348} \pm \mathbf{47}\,\,(+52\%)$ \\
                & DrQ  & $614 \pm 22$ & $337 \pm 15$ & $304 \pm 20$ & $552 \pm 49$ & $254 \pm 16$ & $201 \pm 25$ \\
                & CURL & $446 \pm 17$ & $268 \pm 7$ & $201 \pm 12$ & $512 \pm 19$ & $232 \pm 9$ & $229 \pm 35$ \\
                & SAC  & $177 \pm 7$ & $209 \pm 4$ & $144 \pm 2$ & $89 \pm 20$ & $117 \pm 6$ & $65 \pm 10$ \\
  \bottomrule
  \end{tabular}
  \vskip -0.1in
\end{table*}

To visualize the representations learned by \ourM{} and \ourM{-T}, we apply the t-distributed stochastic neighbor embedding (t-SNE) algorithm, a nonlinear dimensionality reduction technique to keep the similar high-dimensional vectors close in lower-dimensional space.
Figure~\ref{fig-tsne} illustrates that in both Cartpole-swingup and Cheetah-run tasks, \ourM{} and \ourM{-T} can well group the representations from different observations with similar latent states, since these observations are projected as adjacent points in the two-dimensional map space.

\subsection{Comparisons of Rewards and Transition Dynamics}
\label{subsec-6.3}
We conduct experiments on DCS with dynamic backgrounds to demonstrate two points: (1) the representations learned by using observation transition dynamics involve much task-irrelevant information; (2) we can learn better representations by using RSD-OA.

We evaluate the generalization performance of different representation learning methods:
(1) \textbf{RSP}: encode task-relevant information in rewards. We use starting observations and predefined action sequences to directly predict reward sequences. Specifically, we apply a 3-layer MLP with $T$-dimensional outputs after the pixel encoder to estimate $T$-dimensional reward sequences via MSE loss function.
(2) \textbf{TDP}: encode task-relevant information in transition dynamics. We apply a contrastive loss~\cite{corr/abs-1807-03748} to maximize the mutual information of representations from $o_t$ and $o_{t+T}$. Notice that we can simply implement TDP in CURL by changing the inputs of the target encoder from $o_t$ to $o_{t+T}$.
(3) \textbf{RDP}: encode task-relevant information in both rewards and transition dynamics. RDP is the combination of RSP and TDP. Notice that RDP with $T=1$ is similar to MISA.
(4) \textbf{RDP-BM}: encode task-relevant information in both rewards and transition dynamics. We first compute bisimulation metrics~\cite{iclr/0001MCGL21} in RDP and then optimize the MSE distances between representations to approximate their bisimulation metrics. Notice that RDP-BM with $T=1$ is similar to DBC.
(5) \textbf{\ourM{}}: encode task-relevant information in both rewards and transition dynamics. We learn representations by predicting the characteristic functions of RSD-OA.

For fair comparison, we apply the batch size 256, a 3-layer MLP as the predictor, and learn agents under the two training environment setting with dynamic backgrounds in all methods. We also conduct experiments to select the length $T$ of reward sequences from 1 to 7.

We show the results in Figure~\ref{fig-rp-rdp-bism}, which illustrates that \ourM{} outperforms others on Cartpole-swingup and Cheetah-run tasks. 
The results of RSP and \ourM{} demonstrate that the auxiliary task to predict characteristic functions of RSD-OA performs better than these directly predict the reward sequences---the expectations of RSD-OA.
Moreover, directly using observation transition dynamics to encode task-relevant information indeed hinder the generalization performance of agents, as RSP outperforms TDP and RDP for all the same reward lengths. These results empirically demonstrate our point of transition dynamics in Section~\ref{sec-4}.
The performances of RDP-BM are lower than others, indicating that the bisimulation metrics may need a more effective learning approach to improve representations rather than limiting the MSE distance between representations.
According to the results in Figure~\ref{fig-rp-rdp-bism}, the performance of \ourM{} is best when $T=5$.

\begin{figure}
  \centering
  \includegraphics[width=4.4cm]{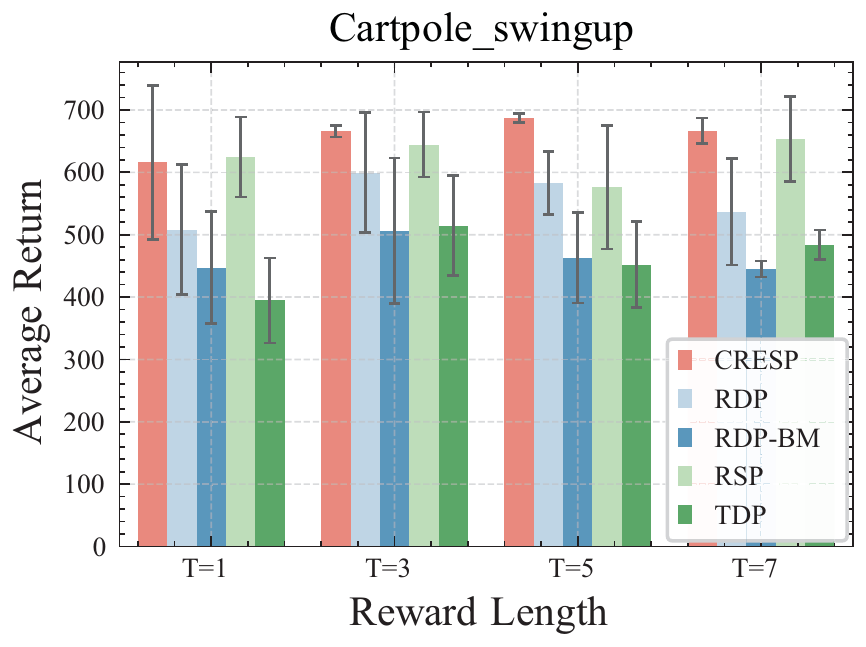}
  \includegraphics[width=4.4cm]{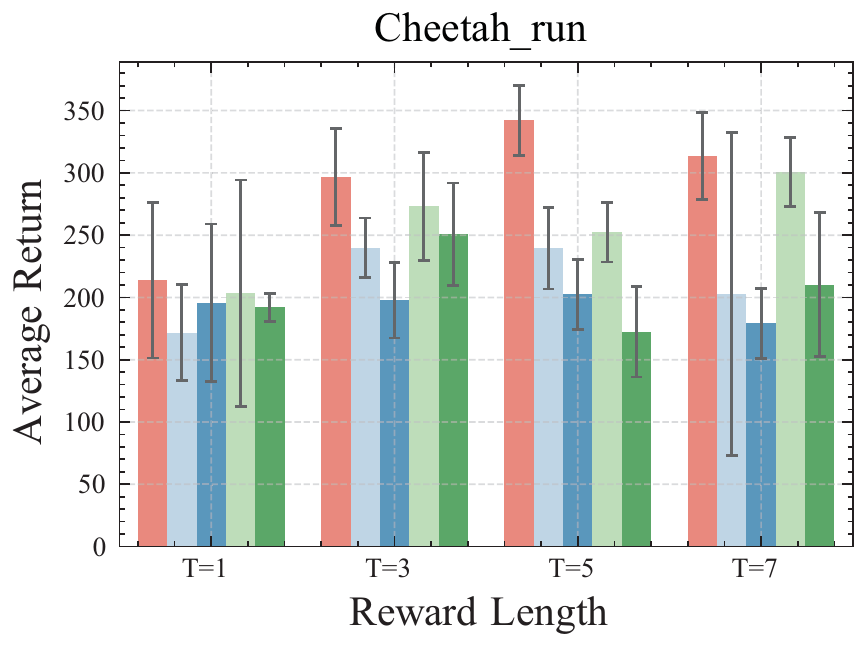}
  \vskip -0.05in
  \caption{We report the mean and standard deviation results under dynamic background settings with 3 random seeds at 500K steps.}
  \label{fig-rp-rdp-bism}
\end{figure}

\subsection{Quantify Task-Relevant and -Irrelevant Information}
\label{subsec-6.4}
To demonstrate the generalization of representations under dynamic background distractions, we design two experiments to quantify the task-relevant and -irrelevant information in representations learned by: 1) RSP; 2) TDP; 3) RDP; 4) DrQ; 5) \ourM{}; 6) \ourM{-T}.

We collect an offline dataset with 100K transitions (data of 800 episodes) drawn from environments with 20 unseen dynamic backgrounds. Then, we sample 80K transitions for training and the rest for evaluation. In each experiment, we apply a 3-layer MLP using ReLU activations up until the last layer. The inputs of the MLP are the learned representations under the two training environment setting with dynamic backgrounds at $500$K environment steps. We train the MLP by Adam with the learning rate 0.001 and evaluate every 10 epochs. In Figure~\ref{fig-tr-tir}, we report the best value of each experiment with three random seeds.

\subsubsection{Task-Irrelevant Information in Representations}
We leverage the environmental label, an one-hot vector of the environment numbers range from 1 to 20. 
Notice that the information of the environmental label is explicitly task-irrelevant for agents.
Therefore, we propose to measure the mutual information between the learned representations and the random variable of the environmental label to quantify the task-irrelevant information.
Based on contrastive learning, we update the 3-layer MLP to estimate the mutual information by minimizing the cross-entropy between representations and environmental labels.
A small cross-entropy indicates a high lower bound on the mutual information. 
Therefore, a \textit{smaller} cross-entropy means that there is \textit{more task-irrelevant} information in representations.

Experiment results in the first column of Figure~\ref{fig-tr-tir} reveal that \ourM{} and \ourM{-T} have less task-irrelevant information than other methods, which demonstrates the effectiveness of using characteristic functions of RSD-OA.
However, task-irrelevant information in representations learned by RSP is almost identical to that by TDP and RDP, which means that representations learned by RSP also involve much task-irrelevant information. This is an empirical evidence that directly predicting the expectations of RSD-OA is not a good choice for discarding task-irrelevant information. 

\begin{figure}
  \centering
  \subfigure[Task irrelevance on C-swin]{\includegraphics[width=3.8cm]{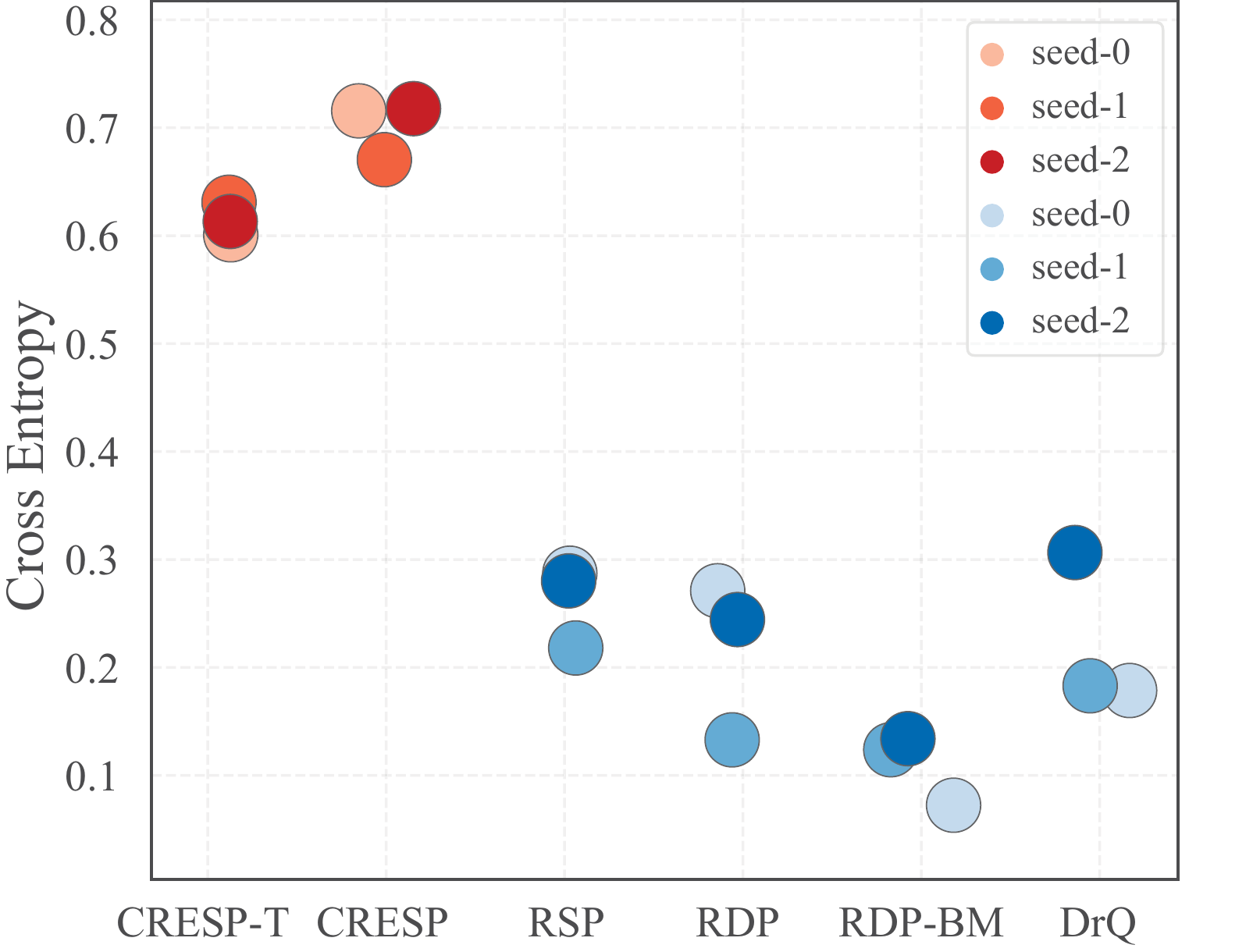}}
  \hskip 0.1in
  \subfigure[Task relevance on C-swin]{\includegraphics[width=3.8cm]{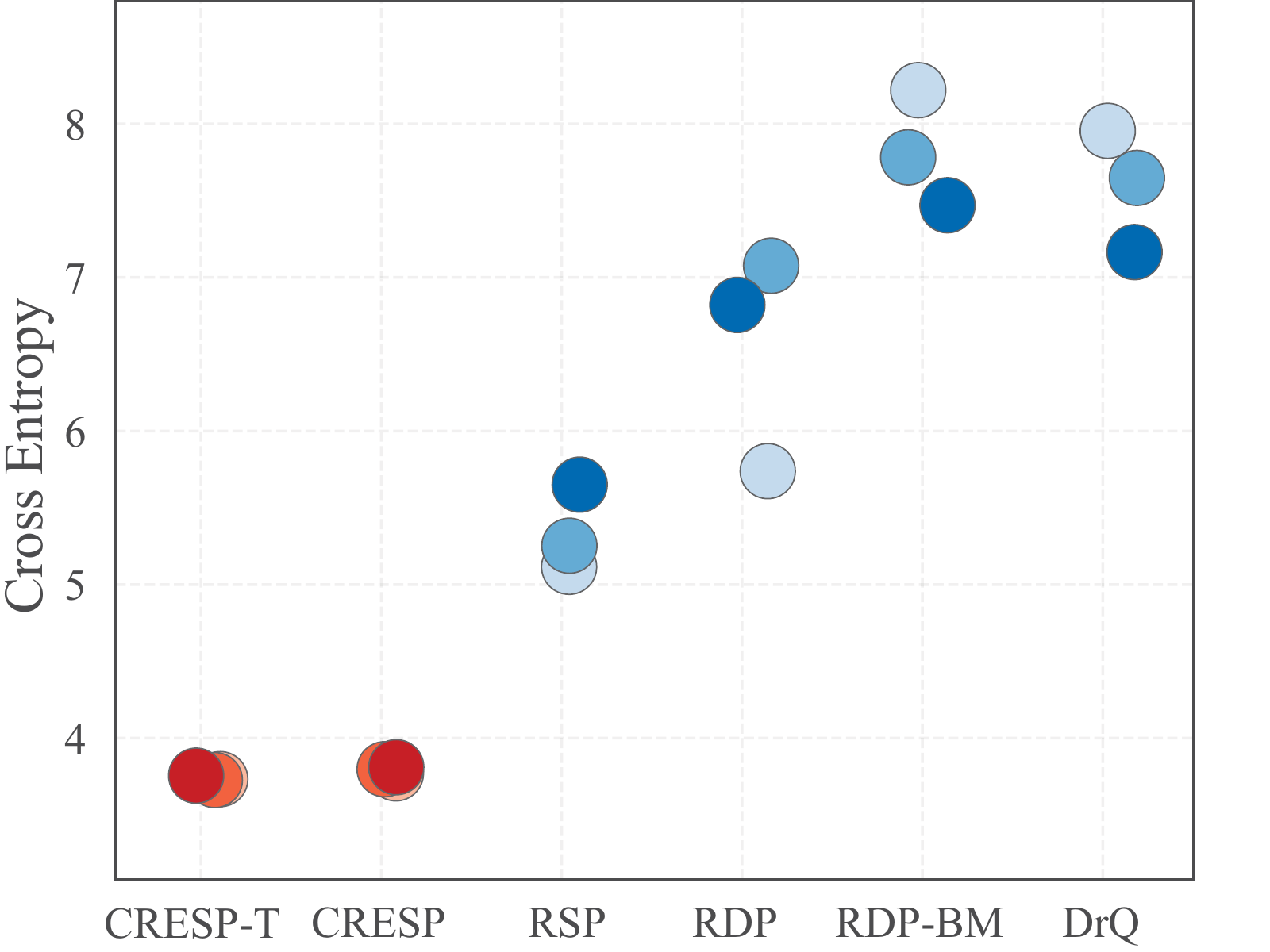}}
  
  \subfigure[Task irrelevance on C-run]{\includegraphics[width=3.8cm]{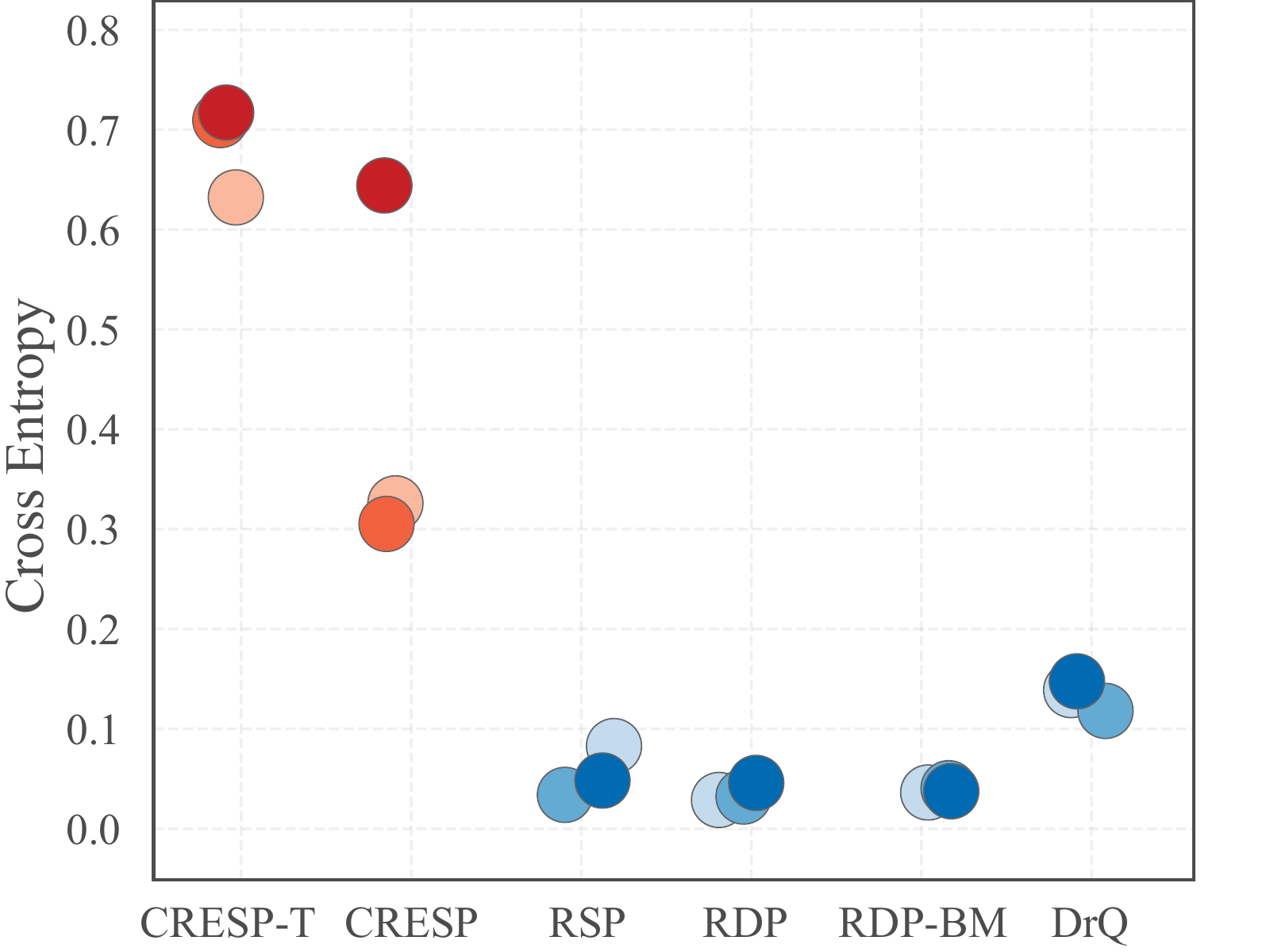}}
  \hskip 0.1in
  \subfigure[Task relevance on C-run]{\includegraphics[width=3.8cm]{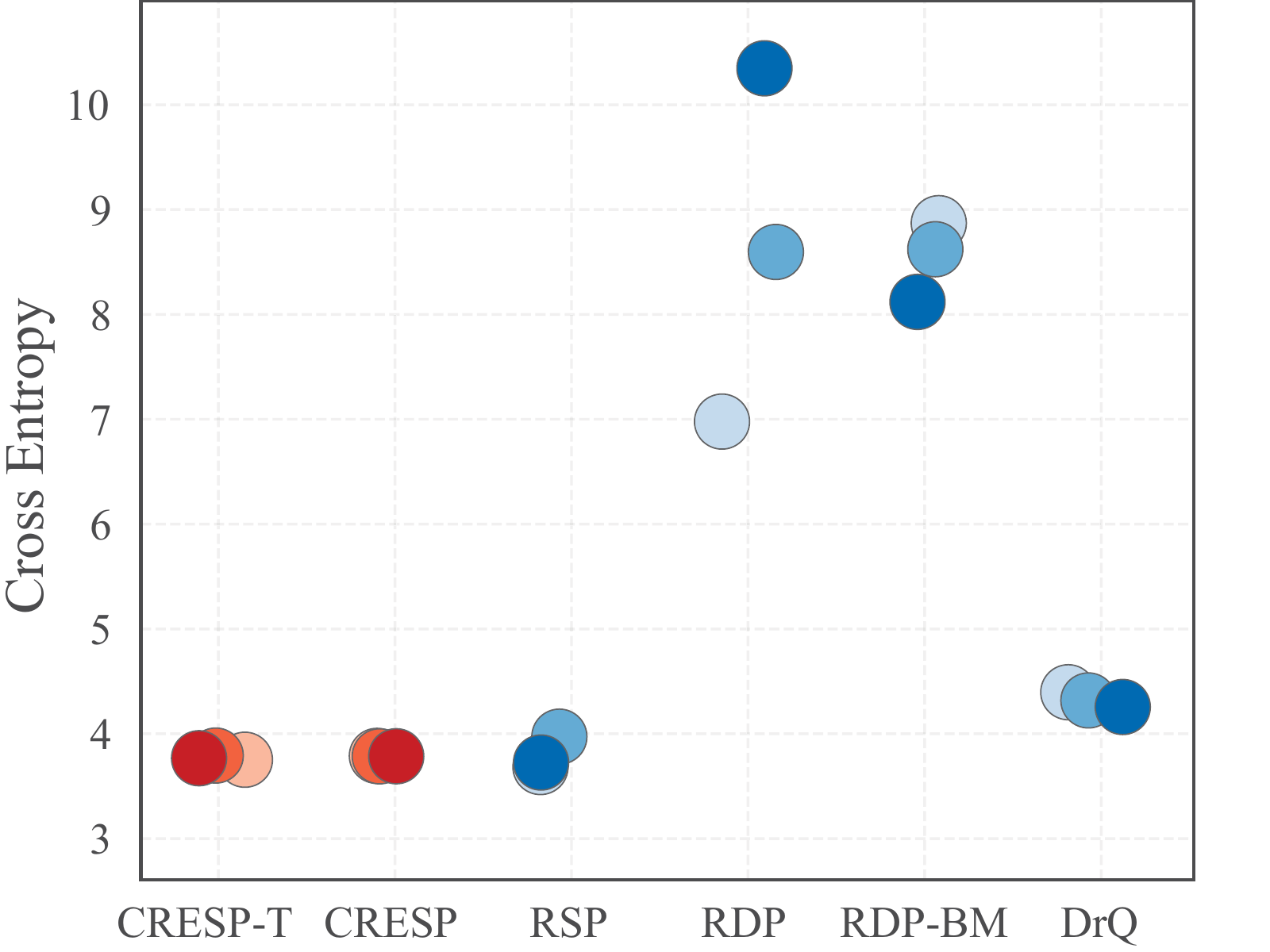}}
  \caption{Quantify the task-relevant and -irrelevant information of learned representations. In the first column, low values indicate that learned representations have much task-irrelevant information. On the contrary, low values in the second column indicate that representations have much task-relevant information. Best results are shown in red.}
  \label{fig-tr-tir}
\end{figure}

\subsubsection{Task-Relevant Information in Representations}
We think of the latent state as a random vector from the simulation to provide the information of the agent.
In principle, if the information of the latent state is presented in the pixel input, VRL algorithms should learn the representation to extract the information relevant to the latent state. Therefore, to quantify the task-relevant information, we design another experiment by measuring the mutual information between latent states and learned representations from pixels.
We also use the collected dataset and the MLP to estimate the mutual information by minimizing the cross-entropy between representations and collected latent states.
Different from the experiment that quantifies task-irrelevant information, a \textit{smaller} cross-entropy loss means that there is more \textit{task-relevant} information in learned representations.

The second column in Figure~\ref{fig-tr-tir} shows that \ourM{} and \ourM{-T} extract the most task-relevant information by predicting characteristic functions of RSD-OA than other methods.
The representations learned by RSP also encode more task-relevant information, but TDP and RDP extract less task-relevant information. The results confirm our point in Section~\ref{sec-4} that RSD-OA can capture the task-relevant information, while directly using observation transition dynamics leads to a negative impact for representation learning.

\begin{figure}
  \centering
  \subfigure[Different lengths on C-swin]{\includegraphics[width=3.9cm]{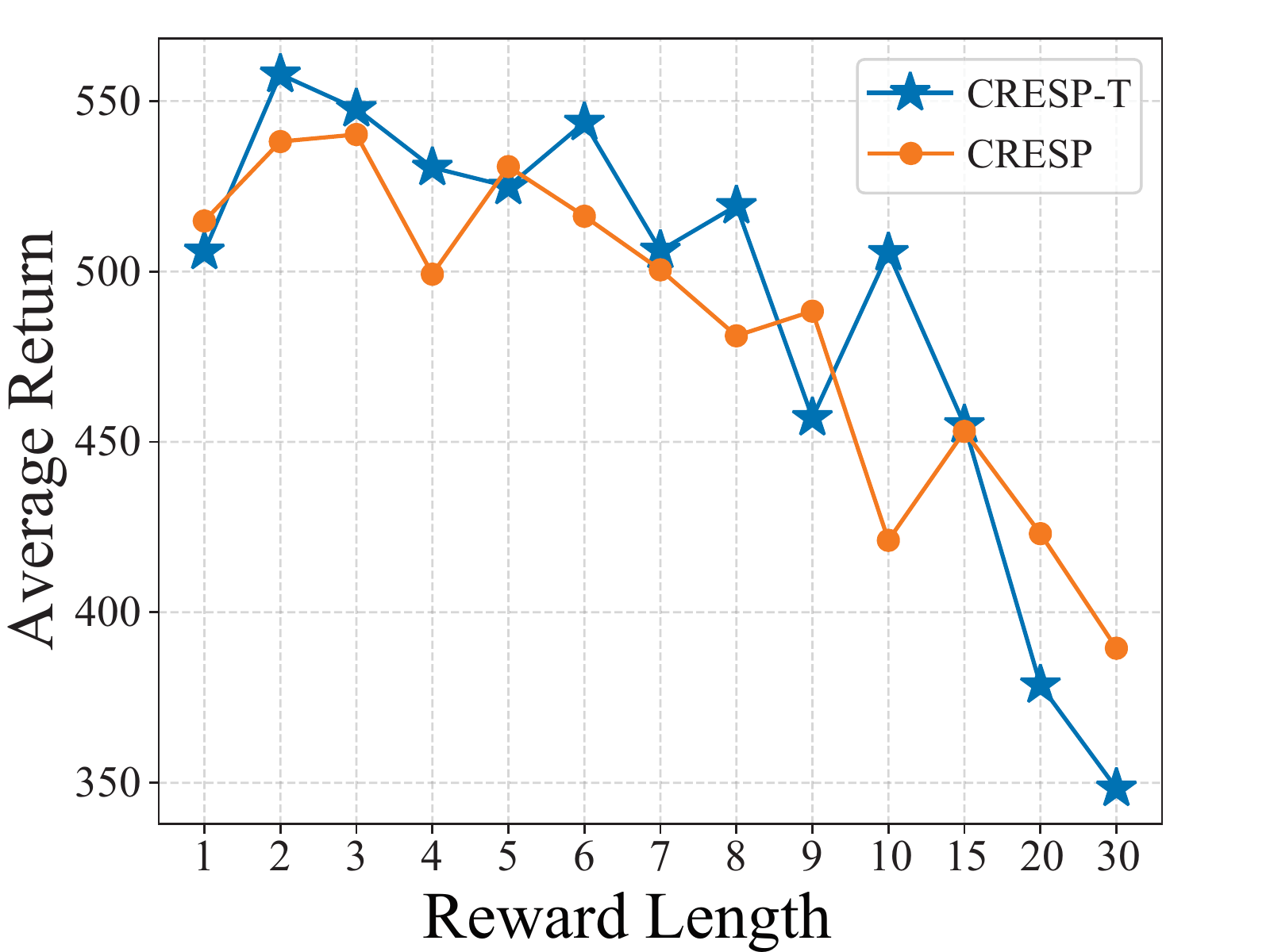}}
  \hskip 0.1in
  \subfigure[Different lengths on C-run]{\includegraphics[width=3.9cm]{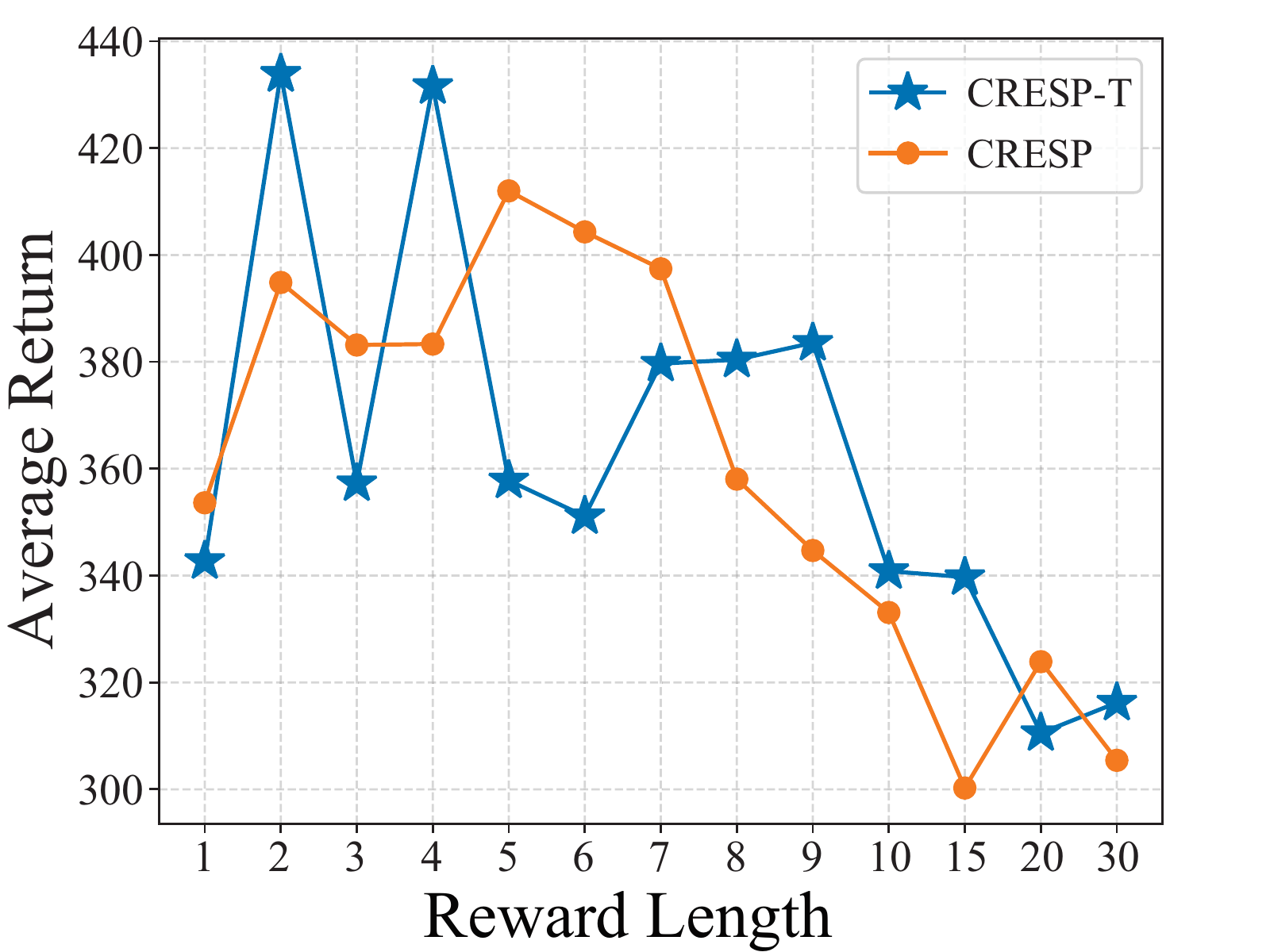}}

  \subfigure[Different prediction targets]{\includegraphics[width=3.9cm]{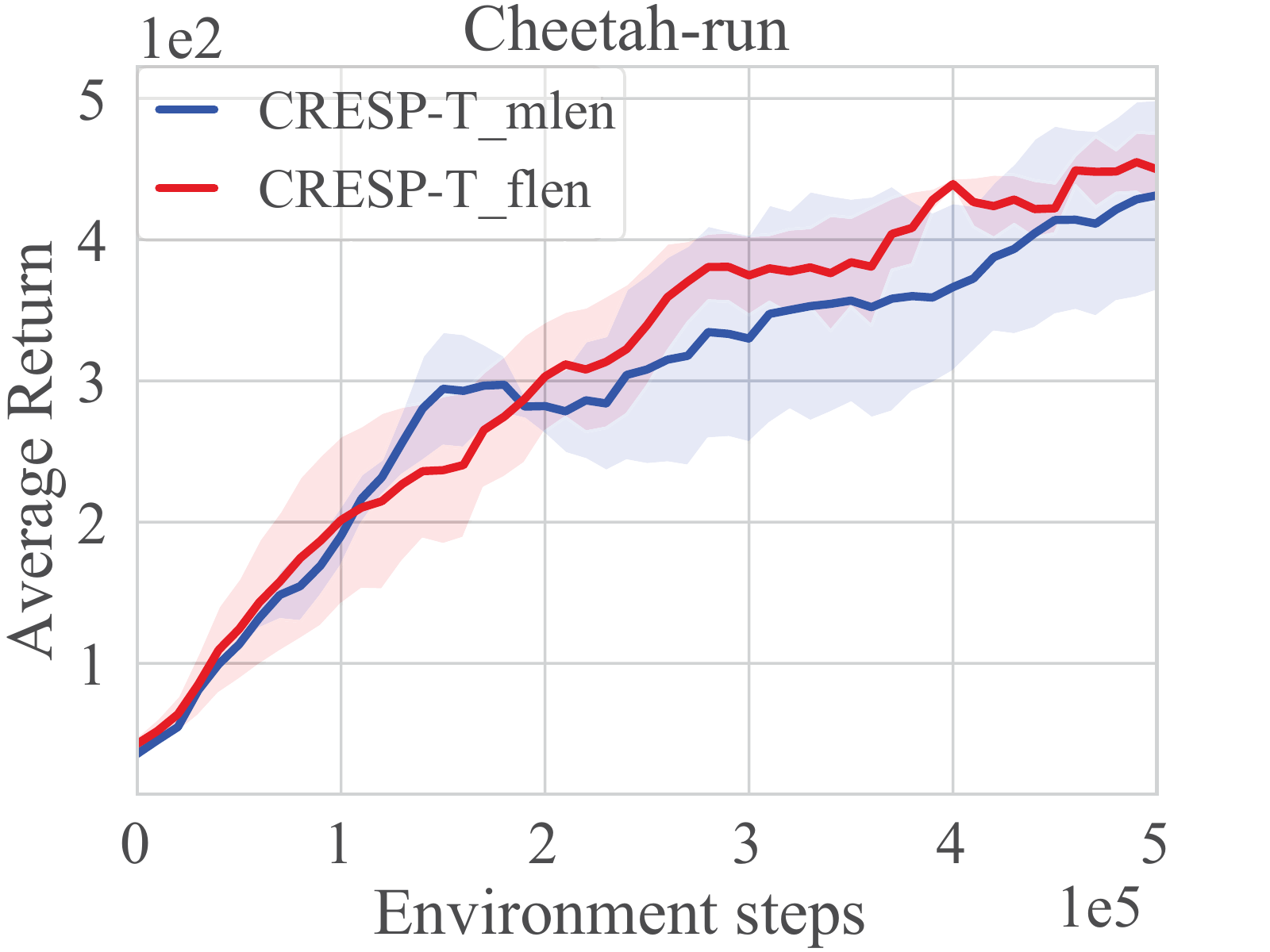}}
  \hskip 0.1in
  \subfigure[Different block numbers]{\includegraphics[width=3.9cm]{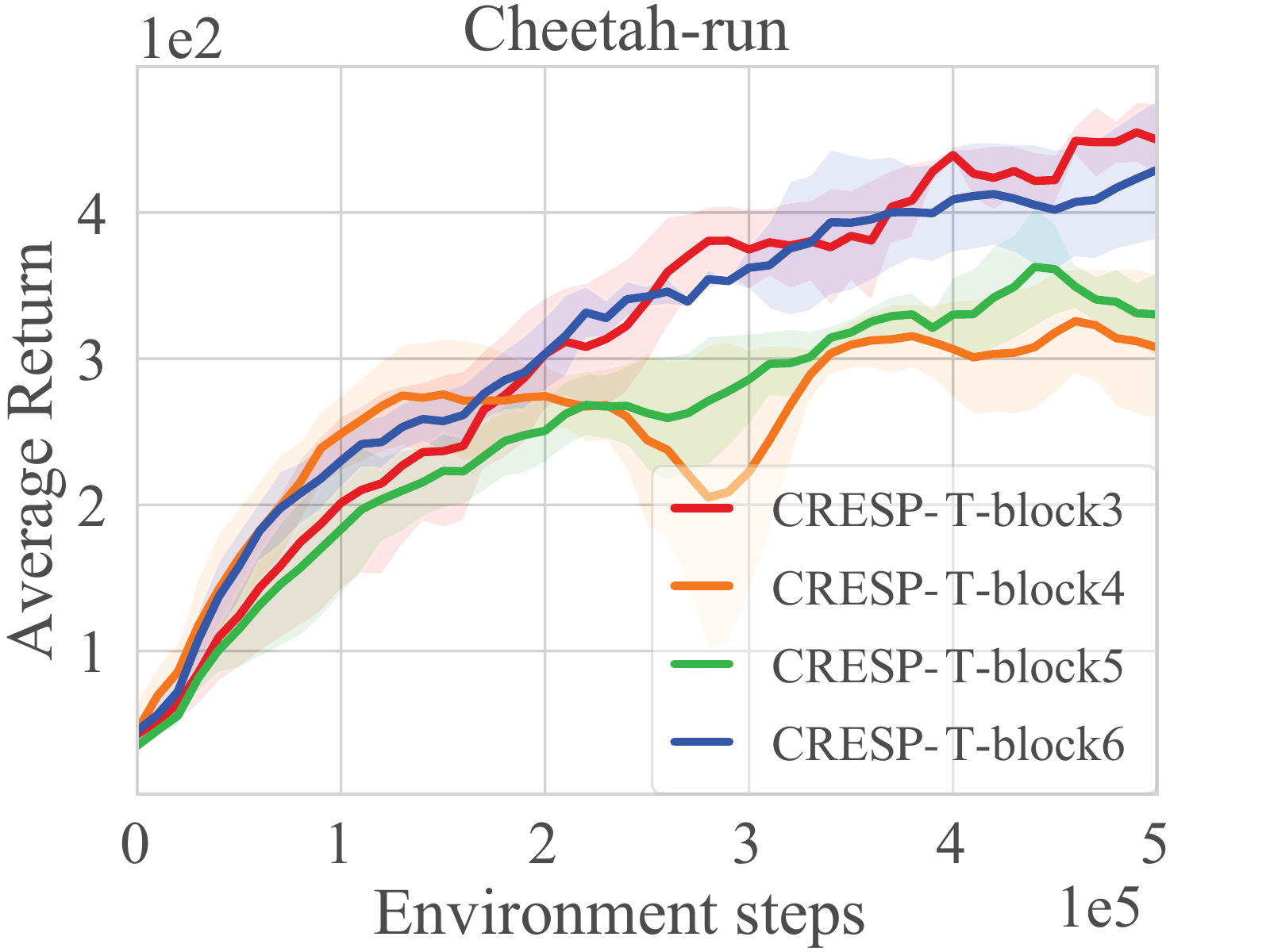}}
  \caption{We report the average results over 3 trails in unseen test environments with strong dynamic color distractions ($\beta=0.5$). All methods are trained in two training environments with weak dynamic color distractions ($\beta_1=0.1, \beta_2=0.2$). In the first row, we show the results at 500K steps to select reward lengths. In the second row, we provide the learning curves for hyperparameter selection of the transformer.}
  \label{fig-rsd-length-tsfm}
\end{figure}

\subsection{Hyperparameter Selection and Evaluation}
\label{subsec-hyperselect}
We conduct extensive experiments to select and evaluate different hyperparameters: (1) length of RSD-OA; (2) hyperparameters of the transformer architecture.
In all experiments, we use the batch size 128 and train agents over three trails in two training environments with dynamic colors.

\subsubsection{Reward Length of RSD-OA}
In the first row of Figure~\ref{fig-rsd-length-tsfm}, we report evaluation results to show that \ourM{-T} with $L=2$ achieves best performances in Cartpole-swingup and Cheetah-run tasks. Such results reveal that the generalization performance is poor when the sequence is long. We think this is because the variance of RSD-OA is too high when the reward length is too long. Therefore, high variance of RSD-OA will cause the difficulties of representation learning, resulting in much task-irrelevant information in the representations.
Based on the results in Figure~\ref{fig-rsd-length-tsfm} (a) and \ref{fig-rsd-length-tsfm} (b), we propose to use the reward length $L=5$ in \ourM{} and $L=2$ in \ourM{-T}.

\subsubsection{Transformer Architecture}
\label{ab-tsfm}
We first conduct experiments to evaluate: (1) \textit{\ourM{-T}-flen}: predict characteristic functions of RSD-OA with a fixed reward length $L$; (2) \textit{\ourM{-T}-mlen}: predict characteristic functions of multiple RSD-OA with reward lengths from 1 to $L$.
Based on $L=2$, Figure~\ref{fig-rsd-length-tsfm} (c) shows that \ourM{-T}-flen achieves better performance.

We then conduct experiments to select the block number in the transformer. Results in Figure~\ref{fig-rsd-length-tsfm} (d) show the performances of the block number from 3 to 6, which illustrates that \ourM{-T} is sensitive to the block number. For high computational efficiency and performance, we apply the transformer architecture with 3 blocks. The last two blocks of each experiment in Figure~\ref{fig-rsd-length-tsfm} (d) are used to predict the real and imaginary parts of characteristic functions, respectively. 


\begin{figure}
  \centering
  \subfigure[Ablating results on C-swin]{\includegraphics[width=3.9cm]{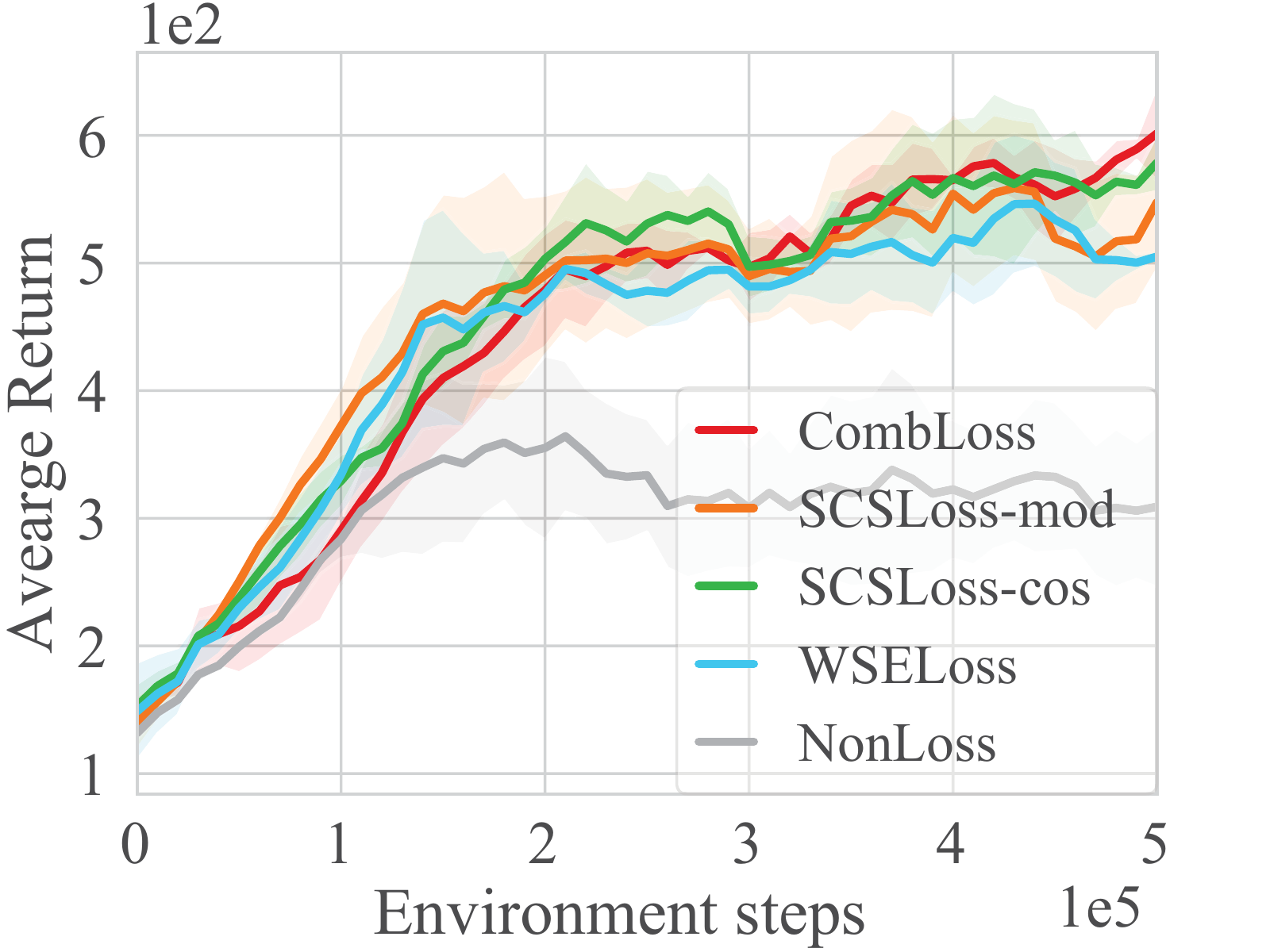}}
  \hskip 0.1in
  \subfigure[Ablating results on C-run]{\includegraphics[width=3.9cm]{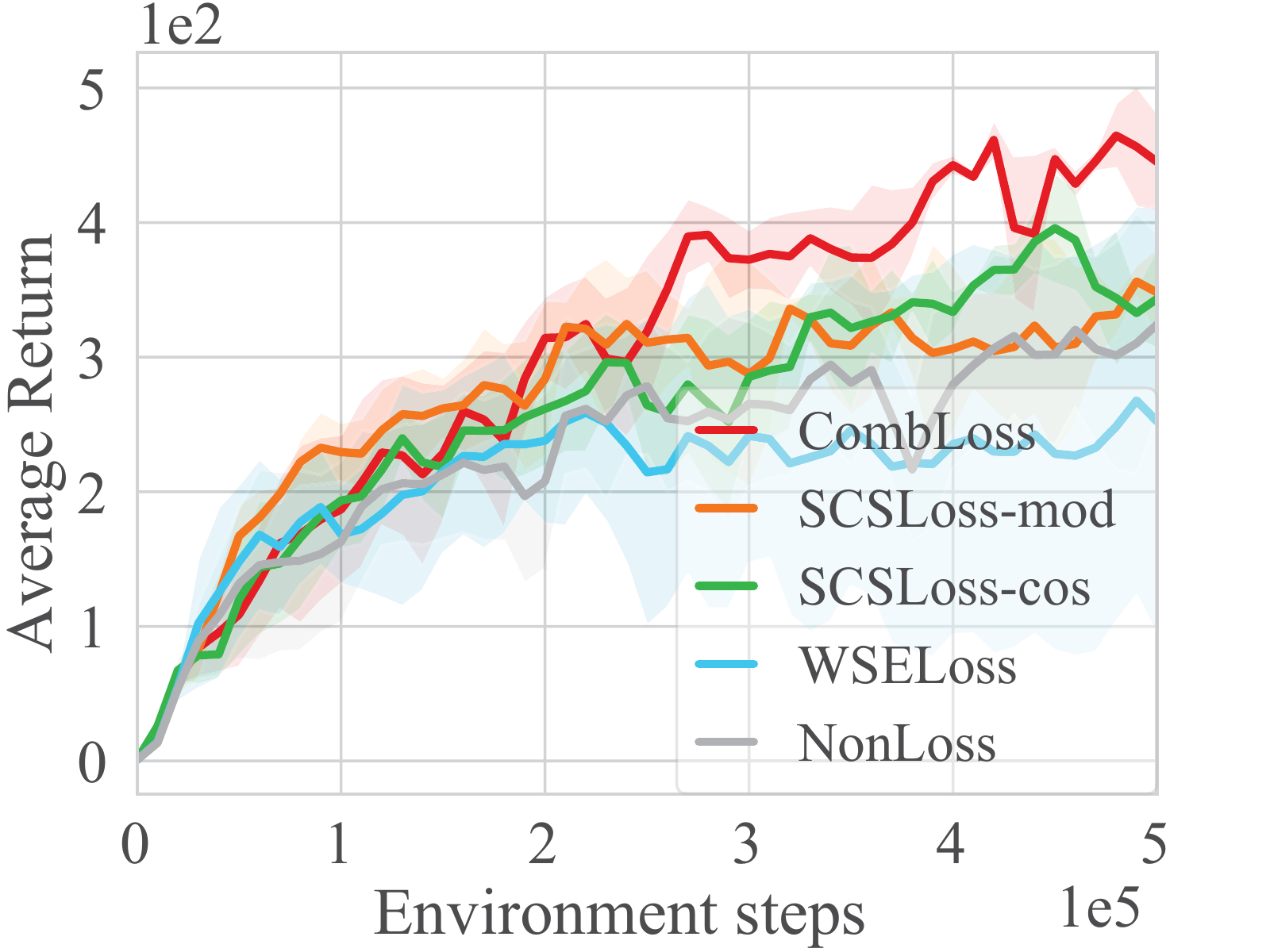}}
  
  \subfigure[Ablating results on C-swin]{\includegraphics[width=3.9cm]{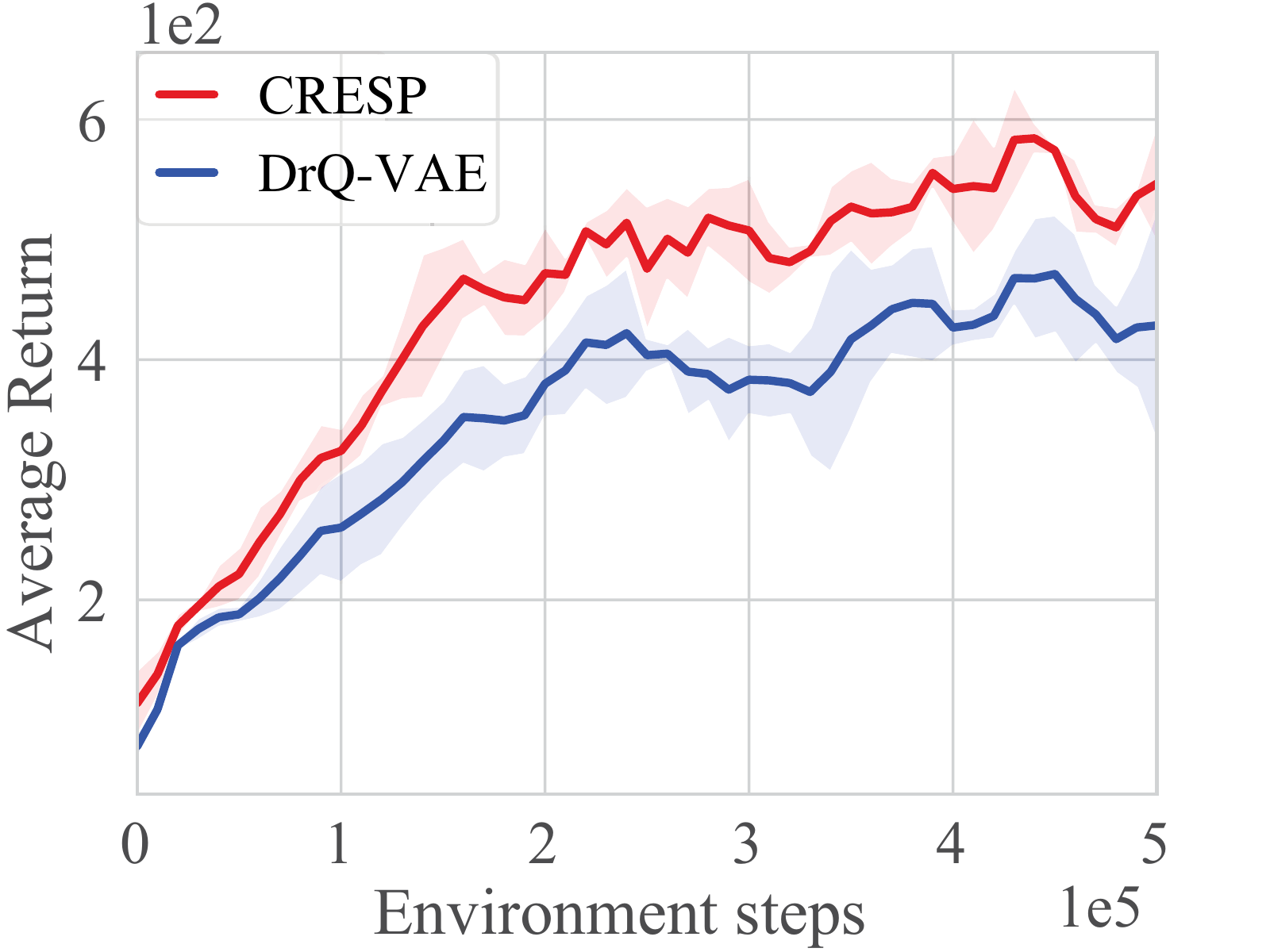}}
  \hskip 0.1in
  \subfigure[Ablating results on C-run]{\includegraphics[width=3.9cm]{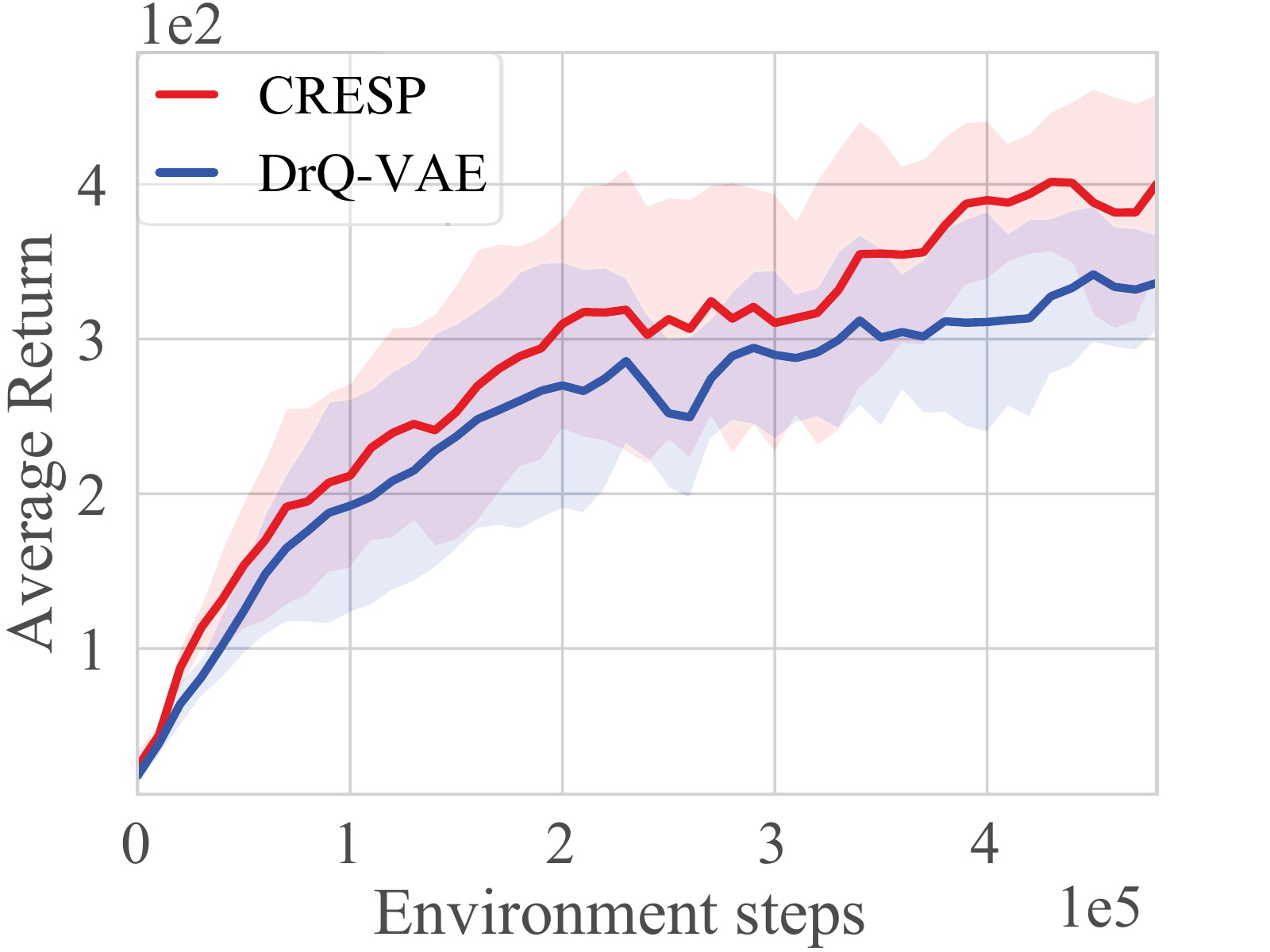}}
  \caption{We report the means and standard deviations of cumulative rewards over 3 trials per task. Results show the effectiveness of different loss functions (first row) and characteristic functions (second row).}
  \label{fig-loss-cf}
\end{figure}

\subsection{Ablation Studies}
\label{subsec-ablation}
We conduct ablation studies about the predictive loss function and characteristic functions.
In all experiments, we use the batch size 128 and train agents over three trails in two training environments with dynamic color distractions.

\subsubsection{Predictive Loss Function}
We conduct experiments of \ourM{-T} to evaluate predictive loss functions proposed in Sections~\ref{subsec-5.2.1} and~\ref{subsec-5.3.1}.
Results are in Figure~\ref{fig-loss-cf},
where \textit{NonLoss} is DrQ without proposed loss functions,
\textit{SCSLoss} is \ourM{-T} using the spectral cosine similarity in Equation~(\ref{eq-cos}),
\textit{WSELoss} is \ourM{-T} using the weighted mean squared error in Equation~(\ref{eq-wse}),
and \textit{CombLoss} is \ourM{-T} using the combination loss in Equation~(\ref{eq-comb}).
In \textit{SCSLoss}, we limit outputs of the transformer model of \ourM{-T} in $[-1,1]$. Therefore, we introduce: (1) \textit{SCSLoss-cos}: \textit{SCSLoss} with the cosine function as the activation function for outputs; (2) \textit{SCSLoss-mod}: \textit{SCSLoss} with remainder operation as the activation function for outputs.
Specifically, in \textit{SCSLoss-mod}, we use outputs to remainder 1 and keep signs of outputs.
Notice that we do not apply any activation function for outputs in \textit{WSELoss} and \textit{CombLoss}.

Results in the first row of Figure~\ref{fig-loss-cf} show that \textit{WSELoss} is worse than \textit{SCSLoss}, and \textit{CombLoss} achieves best performance. All results demonstrate that the spectral cosine similarity can encode more phase information of characteristic functions of RSD-OA.
Based on best results, we propose to use the combination loss in \ourM{-T}.

\subsubsection{Characteristic Functions}
To demonstrate that leveraging characteristic functions is a novel and effective idea for learning RSD-OA, we conduct an ablation studies to compare using characteristic functions with vanilla variational auto-encoder.
For fair comparison, we simply design \textit{DrQ-VAE} to learn RSD-OA by using VAE in DrQ, where both the encoder and decoder are 3-layer MLPs. The Gaussian noise dimension---the dimension of the input of VAE decoder---is 64, which means that the output dimension of VAE encoder based on reparameterization is 128.
We then compare DrQ-VAE with \ourM{}, which learns RSD-OA by using characteristic functions with the 3-layer MLP predictor.
The results in the second row of Figure~\ref{fig-loss-cf} reveal that using characteristic functions of RSD-OA is better to learn task-relevant representations.


\section{Conclusion}
Generalization has been a great challenge in VRL.
To address this challenge, we propose RSD-OA to capture the task-relevant information in both rewards and transition dynamics. Based on RSD-OA, we propose a novel approach, namely \ourM{}, to predict the characteristic functions of RSD-OA for representation learning.
Experiments on unseen test environments with different visual distractions demonstrate the effectiveness of \ourM{} and \ourM{-T}.
We plan to extend the idea of \ourM{} to offline RL settings, which have broad applications in real-world scenarios.

\section*{Acknowledgments}

We would like to thank all the anonymous reviewers for their insightful comments. This work was supported in part by National Nature Science Foundations of China grants U19B2026, U19B2044, 61836011, and 61836006, and the Fundamental Research Funds for the Central Universities grant WK3490000004.

\ifCLASSOPTIONcaptionsoff
  \newpage
\fi



\bibliographystyle{IEEEtran}
\bibliography{reference}
%



%

\begin{IEEEbiography}[{\includegraphics[width=1in,height=1.25in,clip,keepaspectratio]{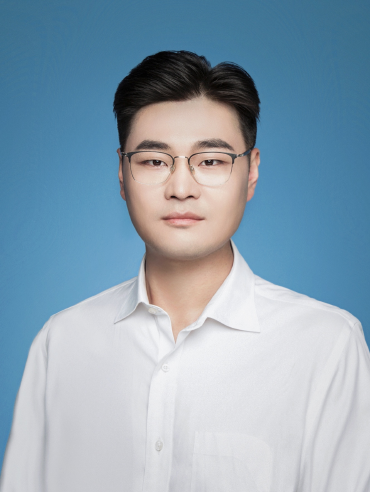}}]{Jie Wang}
	received the B.Sc. degree in electronic
	information science and technology from University of Science and Technology of China, Hefei, China, in 2005, and the Ph.D. degree in computational science from the Florida State University, Tallahassee, FL, in 2011. He is currently a Professor in the Department of Electronic Engineering and Information Science at University of Science and Technology of China, Hefei, China. His research interests include reinforcement learning, knowledge graph, large-scale optimization, deep learning, etc. He is a Senior Member of IEEE.
\end{IEEEbiography}

\begin{IEEEbiography}[{\includegraphics[width=1in,height=1.25in,clip,keepaspectratio]{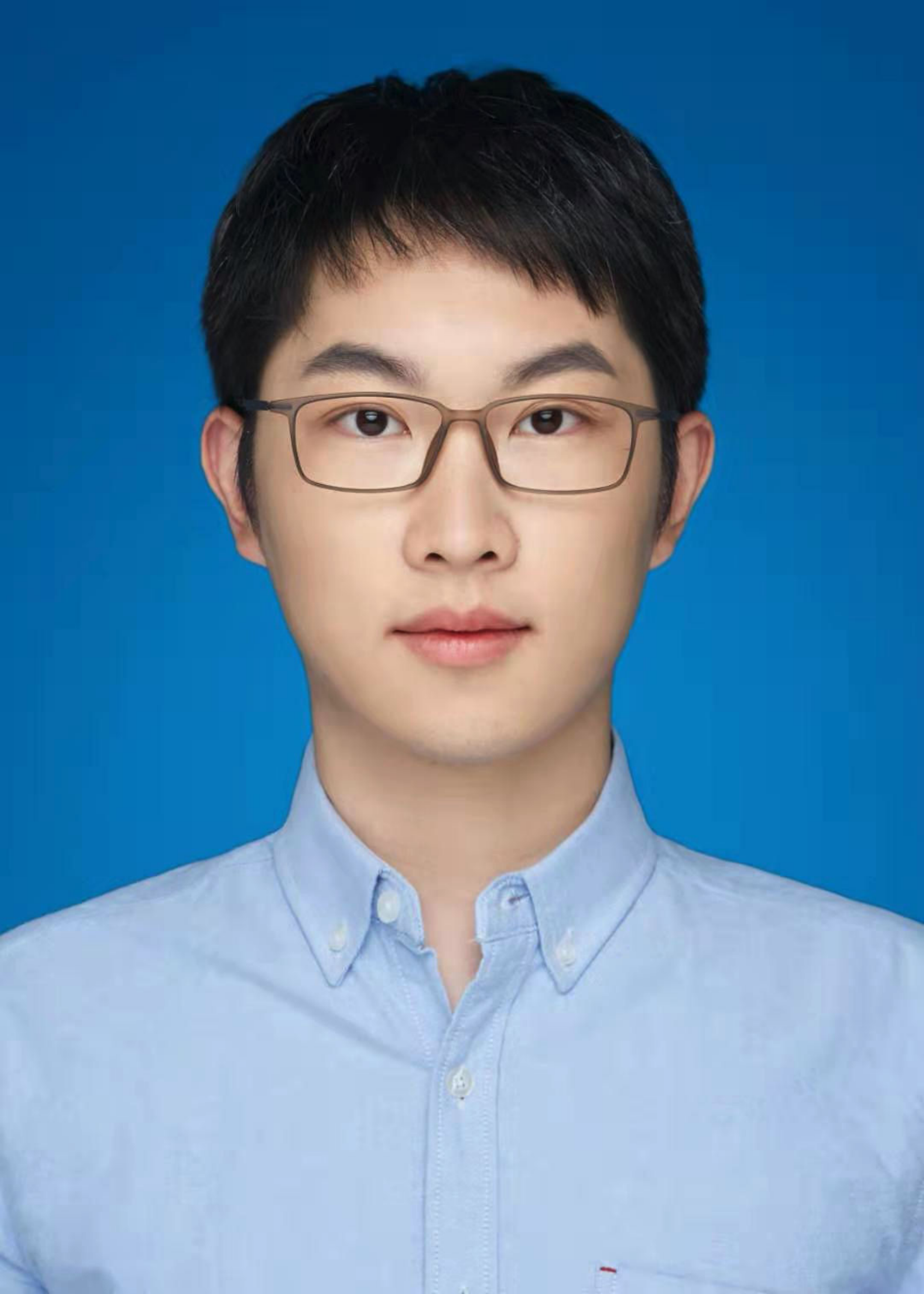}}]{Rui Yang}
	received the B.Sc. degree in information and computing science from Hefei University of Technology, Hefei, China, in 2019. He is currently a Eng.D. candidate in the Department of Electronic Engineering and Information Science at University of Science and Technology of China, Hefei, China. His research interests include reinforcement learning and representation learning.
\end{IEEEbiography}

\begin{IEEEbiography}[{\includegraphics[width=1in,height=1.25in,clip,keepaspectratio]{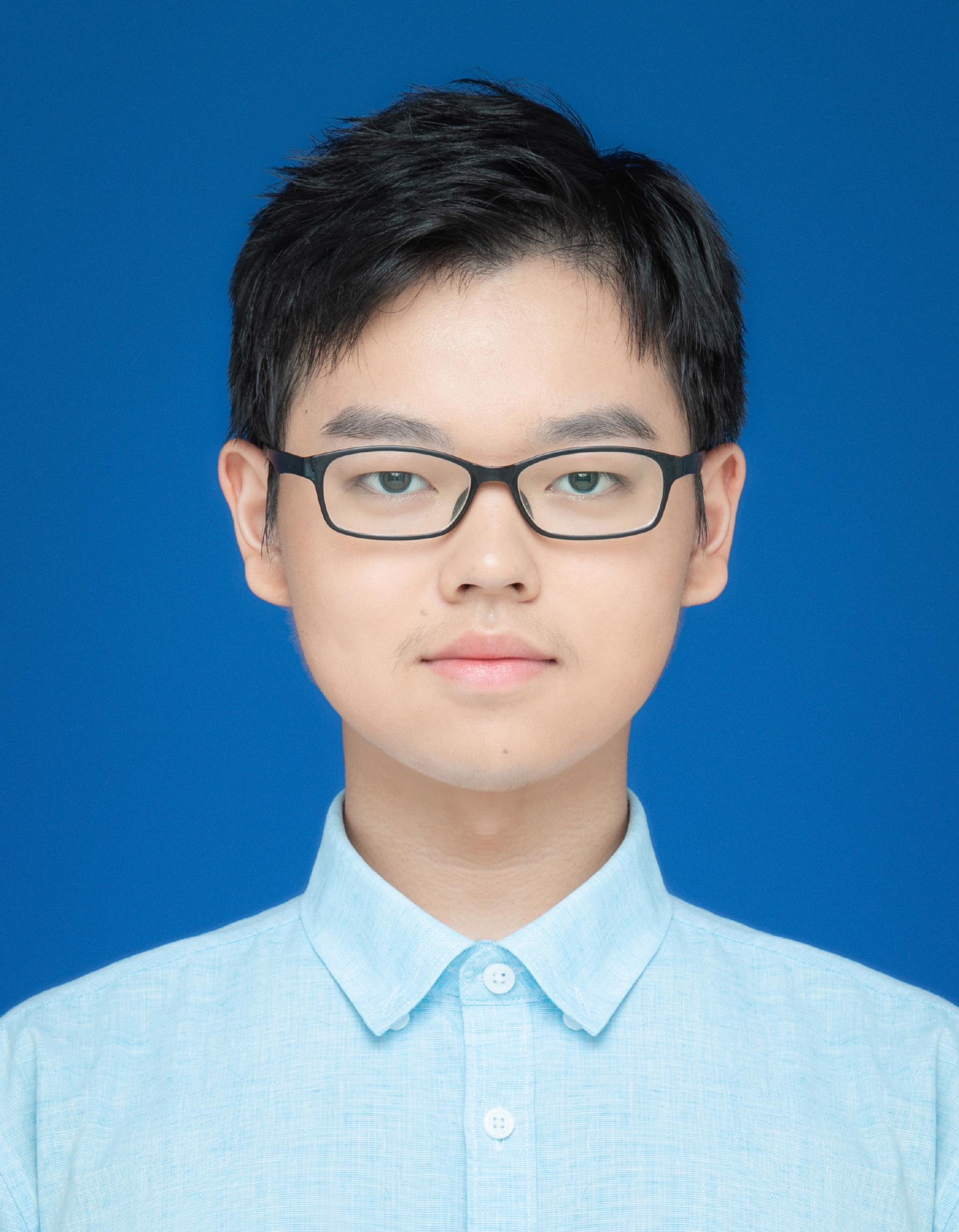}}]{Zijie Geng}
	received the B.Sc. degree in mathematics and applied mathematics from School of the Gifted Young, University of Science and Technology of China, Hefei, China, in 2022. He is currently an M.Sc. candidate in the Department of Electronic Engineering and Information Science at University of Science and Technology of China, Hefei, China. His research interests include reinforcement learning, graph learning, AI for science and machine learning for combinatorial optimization.
\end{IEEEbiography}

\begin{IEEEbiography}[{\includegraphics[width=1in,height=1.25in,clip,keepaspectratio]{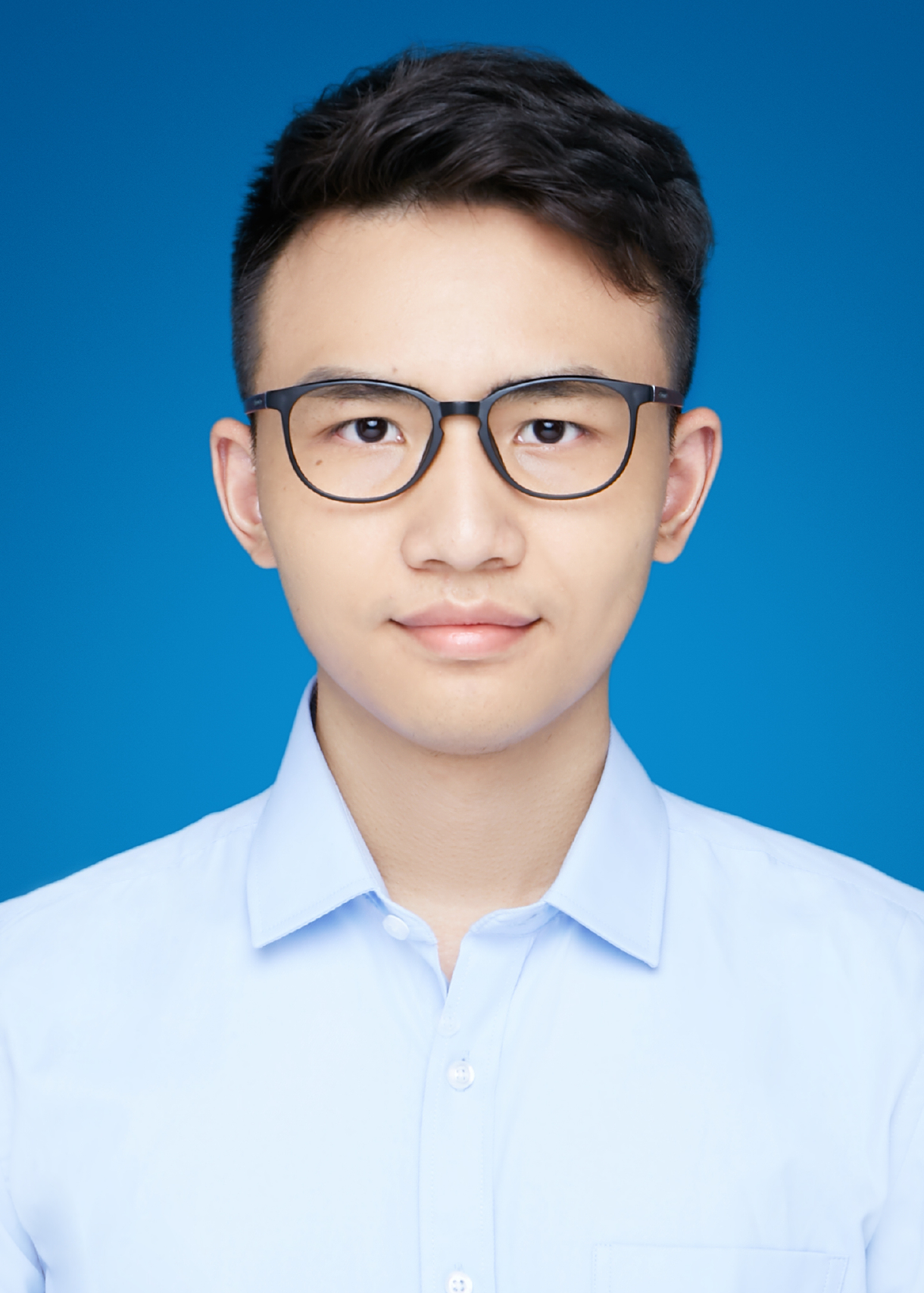}}]{Zhihao Shi}
	received the B.Sc. degree in Department of Electronic Engineering and Information Science from University of Science and Technology of China, Hefei, China, in 2020. a Ph.D. candidate in the Department of Electronic Engineering and Information Science at University of Science and Technology of China, Hefei, China. His research interests include graph representation learning and natural language processing.
\end{IEEEbiography}

\begin{IEEEbiography}[{\includegraphics[width=1in,height=1.25in,clip,keepaspectratio]{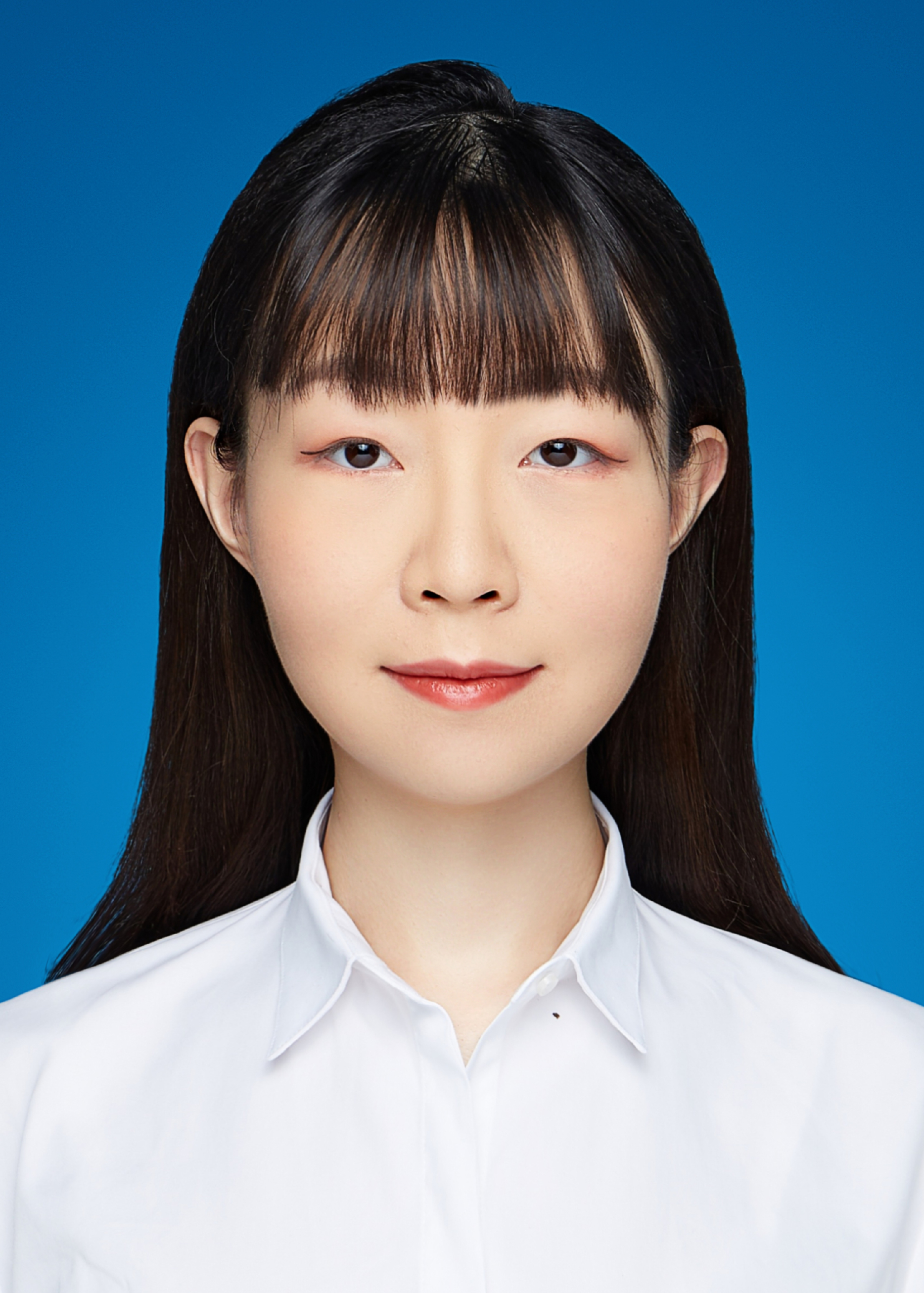}}]{Mingxuan Ye}
	received the B.Sc. degree in information and computing sciences from Nanjing University, Nanjing, China, in 2018. She is currently an M.Sc. candidate in the Department of Electronic Engineering and Information Science at University of Science and Technology of China, Hefei, China. Her research interests include reinforcement learning.
\end{IEEEbiography}

\begin{IEEEbiography}[{\includegraphics[width=1in,height=1.25in,clip,keepaspectratio]{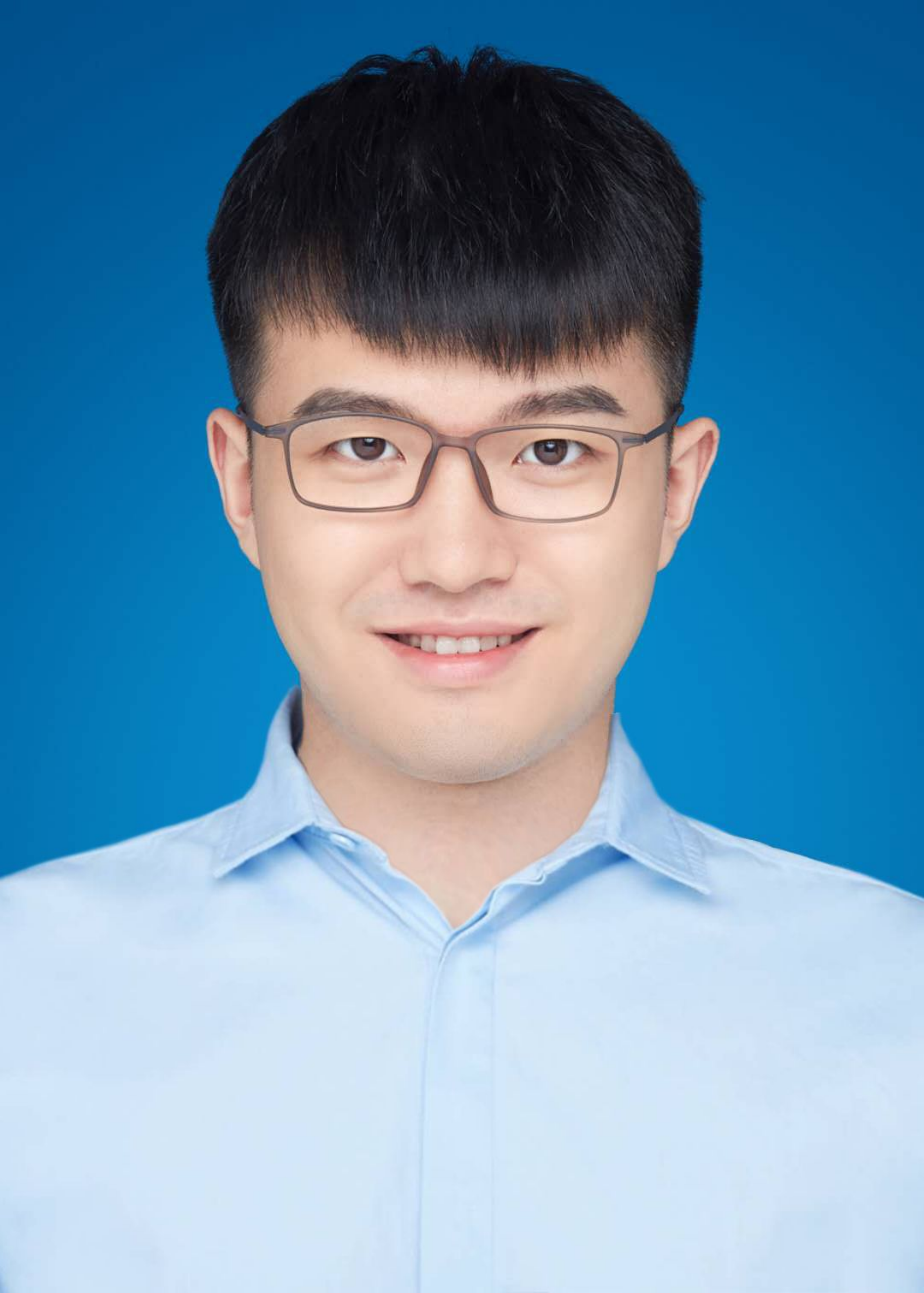}}]{Qi Zhou}
	received the B.Sc. degree in computer science and technology from University of Science and Technology of China, Hefei, China, in 2018. He is currently a Ph.D. candidate in the Department of Electronic Engineering and Information Science at University of Science and Technology of China, Hefei, China. His research interests include trusted and sample-efficient reinforcement learning.
\end{IEEEbiography}

\begin{IEEEbiography}[{\includegraphics[width=1in,height=1.25in,clip,keepaspectratio]{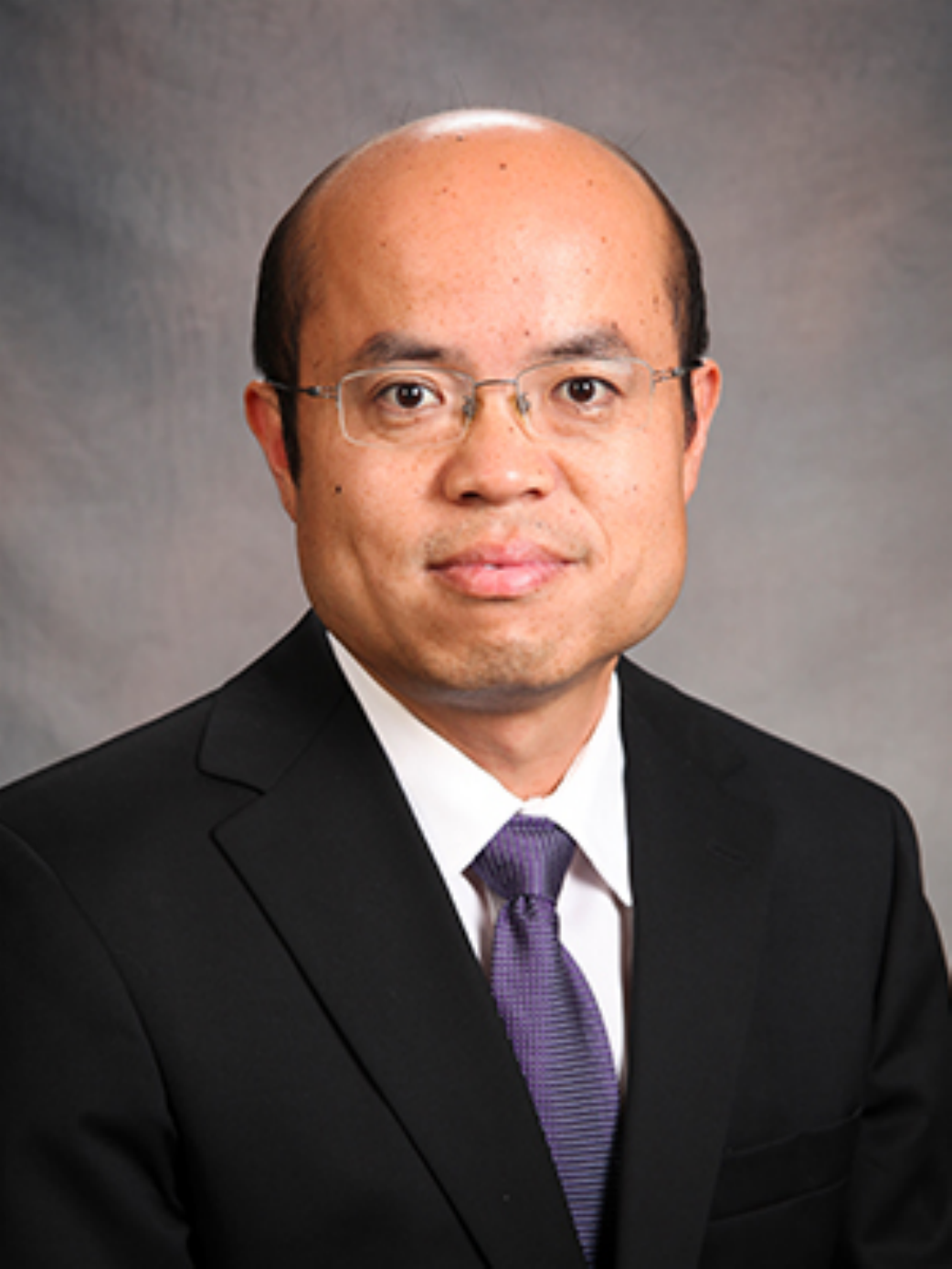}}]{Shuiwang Ji}
	received the Ph.D. degree in computer science from Arizona State University, Tempe, Arizona, in 2010. Currently, he is a Professor and Presidential Impact Fellow in the Department of Computer Science and Engineering, Texas A\&M University, College Station, Texas. His research interests include machine learning, deep learning, graph and image analysis, and quantum systems. He received the National Science Foundation CAREER Award in 2014. He is currently an Associate Editor for IEEE Transactions on Pattern Analysis and Machine Intelligence, ACM Transactions on Knowledge Discovery from Data, and ACM Computing Surveys. He regularly serves as an Area Chair or equivalent roles for data mining and machine learning conferences, including AAAI, ICLR, ICML, IJCAI, KDD, and NeurIPS. He is a Fellow of IEEE. 
\end{IEEEbiography}

\vskip -0.2in

\begin{IEEEbiography}[{\includegraphics[width=1in,height=1.25in,clip,keepaspectratio]{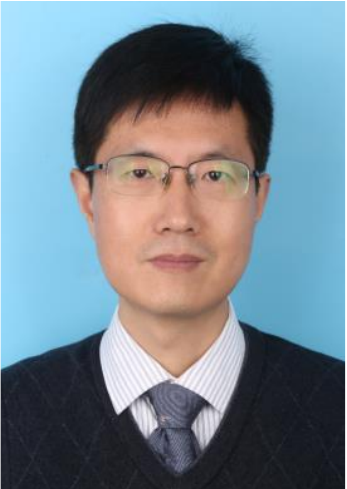}}]{Bin Li}
	received the B.Sc. degree in electrical engineering from Hefei University of Technology, Hefei, China, in 1992, the M.Sc. degree from the Institute of Plasma Physics, Chinese Academy of Sciences, Hefei, in 1995, and the Ph.D. degree in Electronic Science and Technology from the University of Science and Technology of China (USTC), Hefei, in 2001. He is currently a Professor at the School of Information Science and Technology, USTC. He has authored or co-authored over 60 refereed publications. His current research interests include evolutionary computation, pattern recognition, and human-computer interaction. Dr. Li is the Founding Chair of IEEE Computational Intelligence Society Hefei Chapter, a Counselor of IEEE USTC Student Branch, a Senior Member of Chinese Institute of Electronics (CIE), and a member of Technical Committee of the Electronic Circuits and Systems Section of CIE. He is a Member of IEEE.
\end{IEEEbiography}

\vskip -0.2in

\begin{IEEEbiography}[{\includegraphics[width=1in,height=1.25in,clip,keepaspectratio]{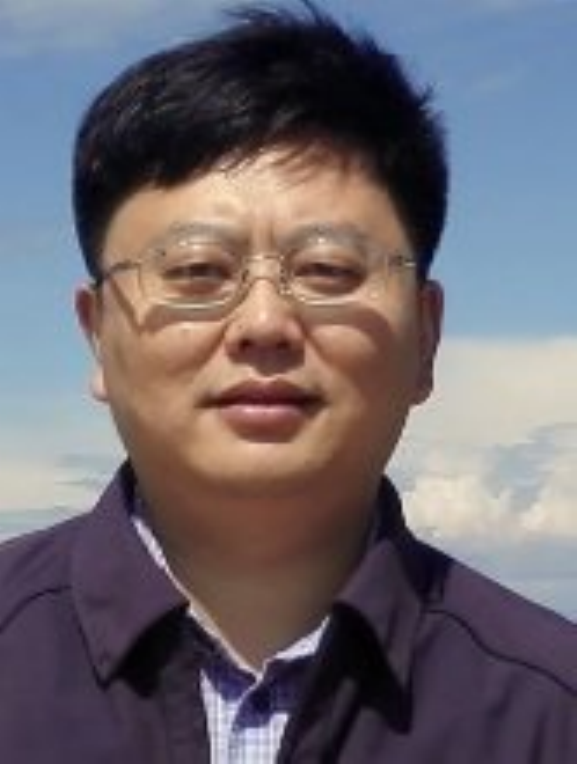}}]{Yongdong Zhang}
	received the Ph.D. degree in electronic engineering from Tianjin University, Tianjin, China, in 2002. He is currently a Professor at the University of Science and Technology of China. He has authored more than 100 refereed journal articles and conference papers. His current research interests include multimedia content analysis and understanding, multimedia content security, video encoding, and streaming media technology. He serves as an Editorial Board Member of Multimedia Systems journal and Neurocomputing. He was the recipient of the Best Paper Award in PCM 2013, ICIMCS 2013, and ICME 2010, and the Best Paper Candidate in ICME 2011. He is a Senior Member of IEEE.
\end{IEEEbiography}

\vskip -0.2in

\begin{IEEEbiography}[{\includegraphics[width=1in,height=1.35in,clip,keepaspectratio]{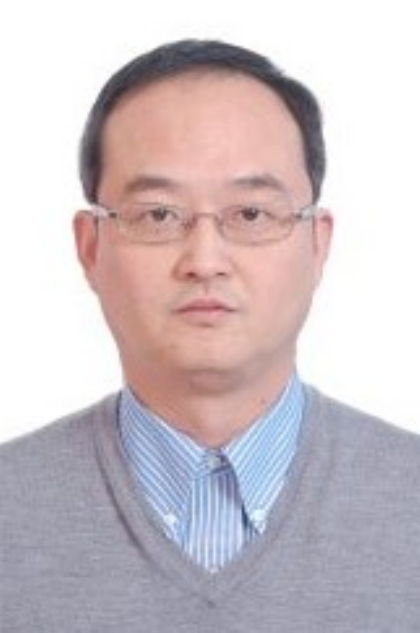}}]{Feng Wu}
	received the B.Sc. degree in electronic engineering from Xidian University, Xi’an, China, in 1992, and received the M.Sc. and Ph.D. degrees from the Harbin Institute of Technology, Harbin, China, in 1996 and 1999, respectively.
	He is now a Professor and Vice President at the University of Science and Technology of China, Hefei, China. Previously, he was a Principle Researcher and Research Manager with Microsoft Research Asia, Beijing, China.
	His research interests include image and video compression, media communication, and media analysis and synthesis. 
	He has authored or co-authored over 200 high-quality articles. His 15 techniques have been adopted into international video coding standards. 
	He serves or had served as the Editor-in-Chief for IEEE Transactions on Circuits and Systems for Video Technology (TCSVT) and as an Associate Editor for IEEE Transactions on Image Processing (TIP) and IEEE Transactions on Multimedia. He also serves as General Chair in ICME 2019, TPC Chair in MMSP
	2011, VCIP 2010, and PCM 2009. He received the IEEE CAS Mac Van Valkenburg Award in 2021, the best paper awards in IEEE TCSVT 2009, VCIP 2016, PCM 2008, and VCIP 2007, and the Best Associate Editor Award of IEEE Transactions on Image Processing (TIP) in 2018. He is a Fellow of IEEE. 
\end{IEEEbiography}








%
\newpage

\appendices

\section{Proofs}


\begin{lemma}~\cite{grimmett2020probability}
\label{lem:cf}
Two random vectors $\mathbf{X}$ and $\mathbf{Y}$ have the same characteristic function if and only if they have the same probability distribution function.
\end{lemma}


\begin{theorem}
\label{app-thm:bd}
Let $\Phi(o)$ be a $T$-level reward sequence representation from any observation $o \in \mathcal{O}$, $V^e_*:\mathcal{O}\to\mathbb{R}$ be the optimal value function in the environment $e\in\mathcal{E}$, $\Bar{V}^e_*:\mathcal{Z}\to\mathbb{R}$ be the optimal value function on the representation space, and $\bar{r}$ be a bound of the reward space, i.e., $|r|<\bar{r}$ for any $r\in\mathcal{R}$.
We have
\begin{align*}
    0\le V^e_*(o)-\Bar{V}^e_*\circ\Phi(o)\le\frac{2\gamma^T}{1-\gamma}\bar{r},
\end{align*}
for any $o\in\mathcal{O}$ and $e\in\mathcal{E}$.
\end{theorem}
\begin{proof}
By the definition of optimal value function, we have $\Bar{V}^e_*\circ\Phi(o)\le V^e_*(o)$.
It suffices to show that $V^e_*(o)-\Bar{V}^e_*\circ\Phi(o)\le\frac{\gamma^T}{1-\gamma}\bar{r}$.
Without loss of generality, let $\pi_*:\mathcal{O}\to\Delta(\mathcal{R}^T)$ be an optimal policy such that $\pi_*(\cdot|o)=\pi_*(\cdot|o')$ if $[o]_s=[o']_s$.
Let $\hat{\pi}:\mathcal{Z}\to\Delta(\mathcal{R}^T)$ be any policy in the representation space.
Then we have
\begin{align*}
    V^e_*(o)&-\Bar{V}^e_*\circ\Phi(o)\le \mathbb{E}_{\pi_*}^e\left[\sum_{t=0}^\infty\gamma^tR_{t+1}\right]-\mathbb{E}_{\hat{\pi}\circ\Phi}^e\left[\sum_{t=0}^\infty\gamma^tR_{t+1}\right]\\
    \le& \mathbb{E}_{\pi_*}^e\left[\sum_{t=0}^{T-1}\gamma^tR_{t+1}\right] - \mathbb{E}_{\hat{\pi}\circ\Phi}^e\left[\sum_{t=0}^{T-1}\gamma^tR_{t+1}\right]\\
    &+ \sum_{t=T}^\infty \gamma^t \left|\mathbb{E}_{\pi_*}^e\left[R_{t+1}\right]-\mathbb{E}_{\hat{\pi}\circ\Phi}^e\left[R_{t+1}\right]\right|\\
    \le& \mathbb{E}_{\pi_*}^e\left[\sum_{t=0}^{T-1}\gamma^tR_{t+1}\right] - \mathbb{E}_{\hat{\pi}\circ\Phi}^e\left[\sum_{t=0}^{T-1}\gamma^tR_{t+1}\right]+ 2\sum_{t=T}^\infty \gamma^t \bar{r}\\
    \le& \mathbb{E}_{\pi_*}^e\left[\sum_{t=0}^{T-1}\gamma^tR_{t+1}\right] - \mathbb{E}_{\hat{\pi}\circ\Phi}^e\left[\sum_{t=0}^{T-1}\gamma^tR_{t+1}\right] +\frac{2\gamma^T}{1-\gamma}\bar{r}.
\end{align*}
We are now to show that there exists a $\hat{\pi}$ such that \begin{align*}
    \mathbb{E}_{\pi_*}^e\left[\sum_{t=0}^{T-1}\gamma^tR_{t+1}\right] = \mathbb{E}_{\hat{\pi}\circ\Phi}^e\left[\sum_{t=0}^{T-1}\gamma^tR_{t+1}\right].
\end{align*}
For any two observations $o,o'\in\mathcal{O}$ such that $\Phi(o)=\Phi(o')$, since $\Phi(o)$ and $\Phi(o')$ are $T$-level reward sequence representations, we have $p(\mathbf{r}|o,\mathbf{a})=p(\mathbf{r}|o',\mathbf{a})$.
Thus, we have 
\begin{align*}
    \mathbb{E}_{\hat{\pi}\circ\Phi}^e\left[\sum_{t=0}^{T-1}\gamma^tR_{t+1}|O_0=o\right]=\mathbb{E}_{\hat{\pi}\circ\Phi}^e\left[\sum_{t=0}^{T-1}\gamma^tR_{t+1}|O_0=o'\right]
\end{align*} for any $\hat{\pi}$.
We define $\hat{\pi}(z)$ as $\pi_*(\bar{o})$, where $\bar{o}$ is an representative observation such that $\Phi(\bar{o})=z$. 
Then, we have
\begin{align*}
    \mathbb{E}_{\hat{\pi}\circ\Phi}^e\left[\sum_{t=0}^{T-1}\gamma^tR_{t+1}|O_0=o\right]=&\mathbb{E}_{\hat{\pi}\circ\Phi}^e\left[\sum_{t=0}^{T-1}\gamma^tR_{t+1}|O_0=\bar{o}\right]\\
    =&\mathbb{E}_{\pi_*}^e\left[\sum_{t=0}^{T-1}\gamma^tR_{t+1}\right].
\end{align*}
\end{proof}

\begin{theorem}
\label{app-thm:cf}
A representation $\Phi(o)$ from any observation $o\in\mathcal{O}$ is a $T$-level reward sequence representation if and only if there exits a predictor $\Psi$ such that for all $\bm{w}\in\mathbb{R}^{T},o\in\mathcal{O}$, and $\mathbf{a}\in\mathcal{A}^T$,
\begin{align*}
    \Psi(\bm{\omega};\Phi(o),\mathbf{a})=\varphi_{\mathbf{R}|o,\mathbf{a}}(\bm{\omega})=\mathbb{E}_{\mathbf{R}\sim p(\cdot|o,\mathbf{a})}\left[e^{i\langle\bm{\omega},\mathbf{R}\rangle}\right].
\end{align*}
\end{theorem}
\begin{proof}
By definition, for any observation $o\in\mathcal{O}$, $\Phi(o)$ is a $T$-level reward sequence representation if and only if there exists $f$ such that
\begin{align*}
    f(\mathbf{r};\Phi(o),\mathbf{a})=p(\mathbf{r}|o,\mathbf{a}).
\end{align*}
Let $\mathcal{M}$ be the set of mappings from $\mathcal{A}^T$ to $\Delta(\mathcal{R}^T)$. Then, $\Phi(o)$ is a $T$-level reward sequence representation if and only if there exists $\hat{f}:\mathcal{Z}\to\mathcal{M}$ such that $\hat{f}\circ\Phi (o)=M_o$ for any $o\in\mathcal{O}$, where $M_o(\mathbf{a})=p(\cdot|o,\mathbf{a})$ for any $\mathbf{a}\in\mathcal{A}^T$.
Since a characteristic function uniquely determines a distribution, and vice versa, there exists a bijection between the distributions $p(\cdot|o,\mathbf{a})$ and their corresponding characteristic functions $\varphi_{\mathbf{R}|o,\mathbf{a}}(\bm{\omega})=\mathbb{E}_{\mathbf{R}\sim p(\cdot|o,\mathbf{a})}\left[e^{i\langle\bm{\omega},\mathbf{R}\rangle}\right]$.
As a result, for any observation $o\in\mathcal{O}$, $\Phi(o)$ is a $T$-level reward sequence representation if and only if there exists $\Tilde{f}:\mathcal{Z}\to\mathcal{M}$ such that $\Tilde{f}\circ\Phi (o)=\Tilde{M}_o$, where $M_o(\mathbf{a})=\varphi_{\mathbf{R}|o,\mathbf{a}}(\cdot)$ for any $\mathbf{a}\in\mathcal{A}^T$.
Such result is equivalent to that there exists a predictor $\Psi$ such that $\Psi(\cdot;\Phi(o),\mathbf{a})=\varphi_{\mathbf{R}|o,\mathbf{a}}(\cdot)$. The predictor $\Psi$ is linear in the ideal case.
\end{proof}

\section{Research Methods}
\label{app-rm}

\subsection{Implementation Details}
\label{app-imp}

\paragraph{Dynamic Background Distractions}
We follow the dynamic background settings of Distracting Control Suite (DCS)~\cite{corr/abs-2101-02722}. To set up different training environments, we take the $N$ videos from the DAVIS 2017 training set. Notice that $N$ is the number of training environments. One environment uses one video as the background and randomly samples a scene and a frame from the video at the start of every episode. Moreover, we set $\beta_{\text{bg}}=1.0$, which means we use the distracting background instead of original skybox.
We apply the 30 videos from the DAVIS 2017 validation dataset as unseen dynamic backgrounds for evaluation. Moreover, we randomly select 1 of 30 dynamic backgrounds in each episode of the test environment.

\paragraph{Dynamic Color Distractions}
In dynamic color settings of DCS, the color is sampled uniformly per channel $x_{0} \sim \mathcal{U}(x-\beta, x+\beta)$ at the start of each episode, where $x$ is the original color in DCS, and $\beta$ is a hyperparameter. We enable the dynamic setting that the color $x_t$ changes to $x_{t+1} = \text{clip} (\hat{x}_{t+1}, x_t - \beta, x_t + \beta)$, where $\hat{x}_{t+1} \sim \mathcal{N}(x_t, 0.03\cdot\beta)$.
To set up different training environments with weak color distractions, we divide the interval $[0.1, 0.2]$ into $N$ parts to obtain different $\beta$: (1) for the one training environment, $\beta_1 = 0.1$; (2) for two training environments, $\beta_1=0.1, \beta_2=0.2$; (3) for three training environments, $\beta_1=0.1, \beta_2=0.15, \beta_3=0.2$.
Then, we evaluate all agents in the test environment with strong dynamic color distractions ($\beta=0.5$).

\paragraph{Network Details}
\label{app-nd}
A shared pixel encoder uses four convolutional layers using 3 × 3 kernels and 32 filters with an stride of 2 for the first layer and 1 for others. Rectified linear unit (ReLU) activations are applied after each convolution. Thereafter, a 64 dimensional output dense layer normalized by LayerNorm is applied with a $\tanh$ activation. Both critic and actor networks are parametrized with 3 fully connected layers using ReLU activations up until the last layer. The output dimension of these hidden layers is 1024.  The pixel encoder weights are shared for the critic and the actor, and gradients of the encoder are not computed through the actor optimizer.
Moreover, we use the random cropping for image pre-processing proposed by DrQ and RAD~\cite{laskin2020reinforcement} as a weak augmentation without prior knowledge of test environments.
For the predictor of characteristic functions, we use GPT architecture with 3 transformer blocks and 2 attention heads to model RSD-OA. The last 2 blocks in the transformer are used to predict the real and imaginary parts of characteristic functions, respectively. All dropout rates are $0.1$ in the transformer.
Furthermore, we list the hyperparameters in Table~\ref{table-hypp}.

\begin{figure}
    \begin{center}
    \centerline{\includegraphics[scale=0.3]{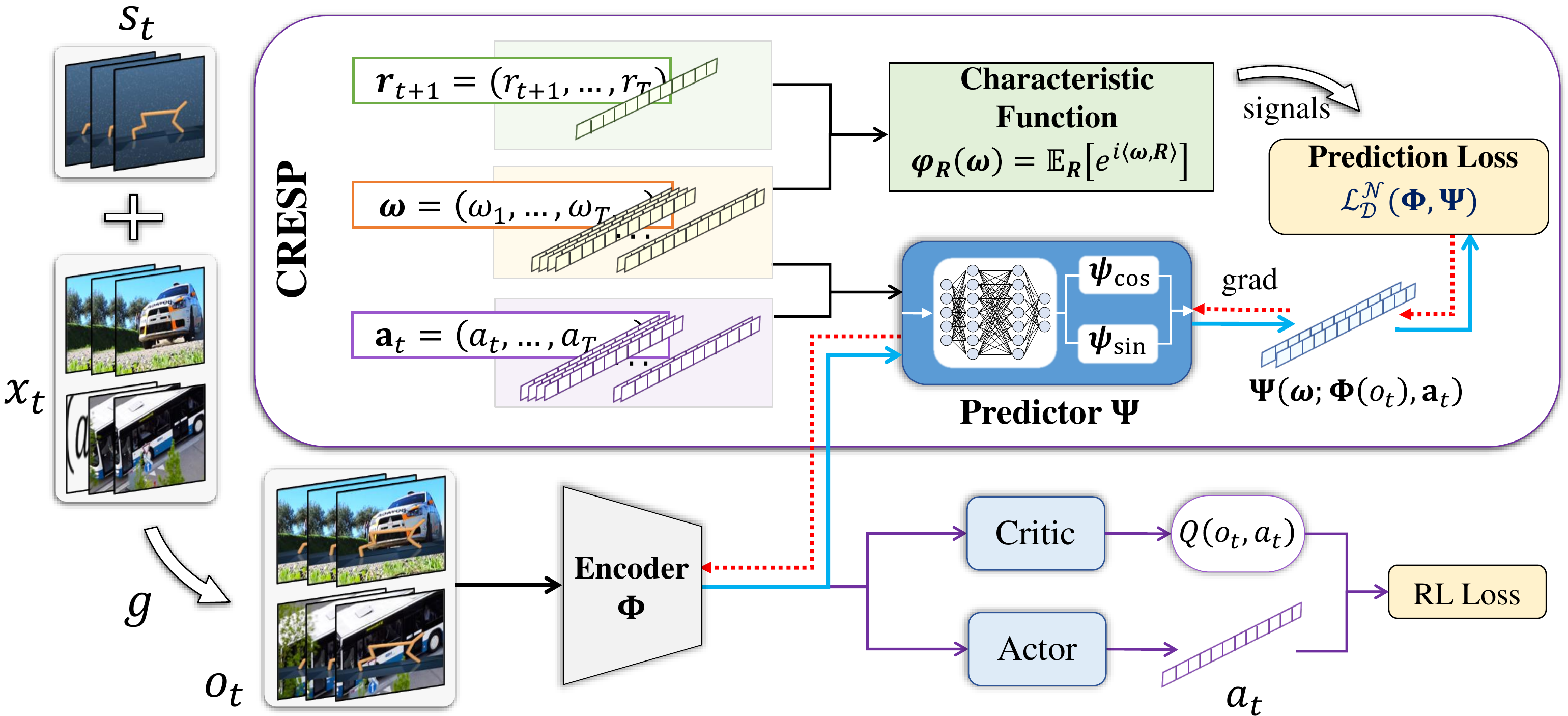}}
    \caption{The overall architecture of \ourM{}.}
    \label{fig-arc-m}
    \end{center}
\end{figure}

\begin{table}
  \caption{Hyperparameters of main experiments in the Distracting Control Suite.}
  \label{table-hypp}
  \centering
  \begin{tabular}{l|l}
    \toprule
    \textbf{Hyperparameter}             & \textbf{Setting} \\
    \midrule
    Optimizer                           & Adam \\
    Discount $\gamma$                   & 0.99 \\
    Learning rate                       & 0.001 \\
    Number of batch size                & 128 \\
    Number of hidden layers             & 2 \\
    Number of hidden units per layer    & 1024 \\
    Replay buffer size                  & 100,000 \\
    Initial steps                       & 1000 \\
    Target smoothing coefficient $\tau$ & 0.01 \\
    Critic target update frequency      & 2 \\
    Actor update frequency              & 1 \\
    Actor log stddev bounds             & [-5, 2] \\
    Initial temperature $\alpha$        & 0.1 \\
    \midrule
    \emph{Hyperparameters of \ourM{} and \ourM{-T}}     & \\
    \quad Gaussian distribution $\mathcal{N}$ & $\mathcal{N}(\textbf{0}, \textbf{1})$\\
    \quad Sample number $\kappa$ from $\mathcal{N}$ & 256 \\
    \quad Trade-off coefficient $\lambda$ & 0.5 \\
    \midrule
    \emph{Hyperparameters of \ourM{}}     & \\
    \quad Number of hidden layers for predictor & 1 \\
    \quad reward length    & 5 \\
    \midrule
    \emph{Hyperparameters of \ourM{-T}}     & \\
    \quad Number of blocks              & 3 \\
    \quad Number of heads               & 2 \\
    \quad dropout rate                  & 0.1 \\
    \quad weight decay                  & 0 \\
    \quad reward length    & 2 \\
    \bottomrule
  \end{tabular}
\end{table}

\paragraph{Details of Visualization}
In Figure~\ref{fig-tsne}, we visualize the representations of \ourM{}, \ourM{-T}, DrQ and CURL in both Cartpole-swingup and Cheetah-run tasks via t-SNE approach. We leverage the fixed 500 observations from two training environments with different dynamic backgrounds in each task. All labels we used to color the points are generated by KMeans with the original states as inputs.

\subsection{Experiments}
\label{app-ar}

\subsubsection{Main Results under Dynamic Backgrounds}
\label{subsec-exp-b}
We validate the generalization capacity of our approaches under one, two, and three training environment settings with dynamic background distractions, respectively.
Specifically, we train the agents in one, two, and three training environments respectively. Each training environment has its specific dynamic background distractions.
We then evaluate these agents in the unseen test environment, which randomly selects 1 of 30 unseen dynamic backgrounds to replace the original background in each episode.
Across all 6 tasks, \ourM{-T} achieves an average performance improvement of $+6.5\%$ (see Table~\ref{table-result}) under distracting backgrounds.
Figure~\ref{fig-result-dcs} shows the generalization performance of our policy under different numbers of the training environments (source domains) in detail.
Furthermore, we list the evaluation results under one, two, and three training environments with dynamic backgrounds in Table~\ref{table-result-2}, which shows that \ourM{} and \ourM{-T} significantly outperform prior state-of-the-arts.

\begin{figure}
  \centering
  \includegraphics[width=4.4cm]{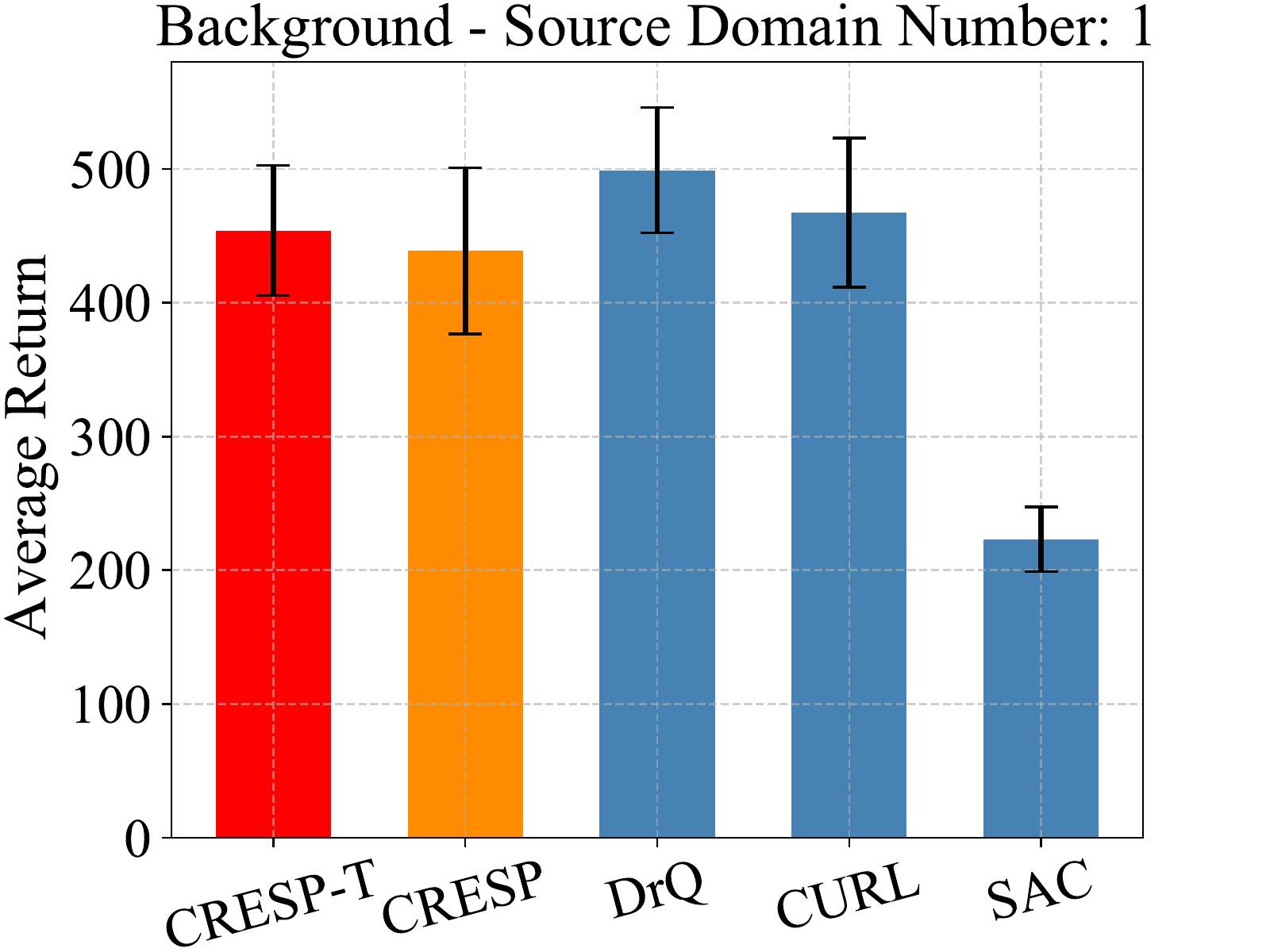}
  \includegraphics[width=4.4cm]{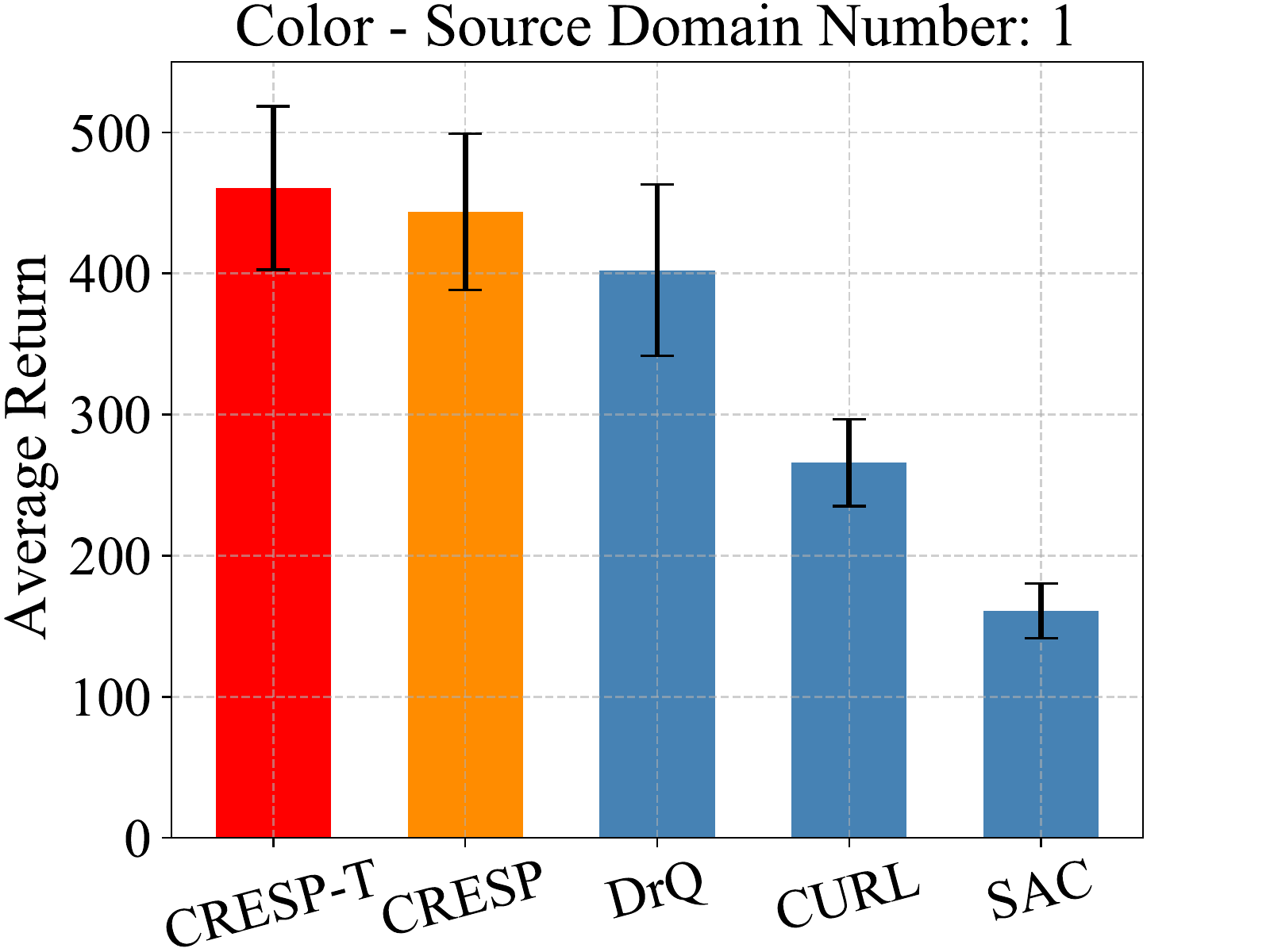}
  \vskip 0.1in
  
  \includegraphics[width=4.4cm]{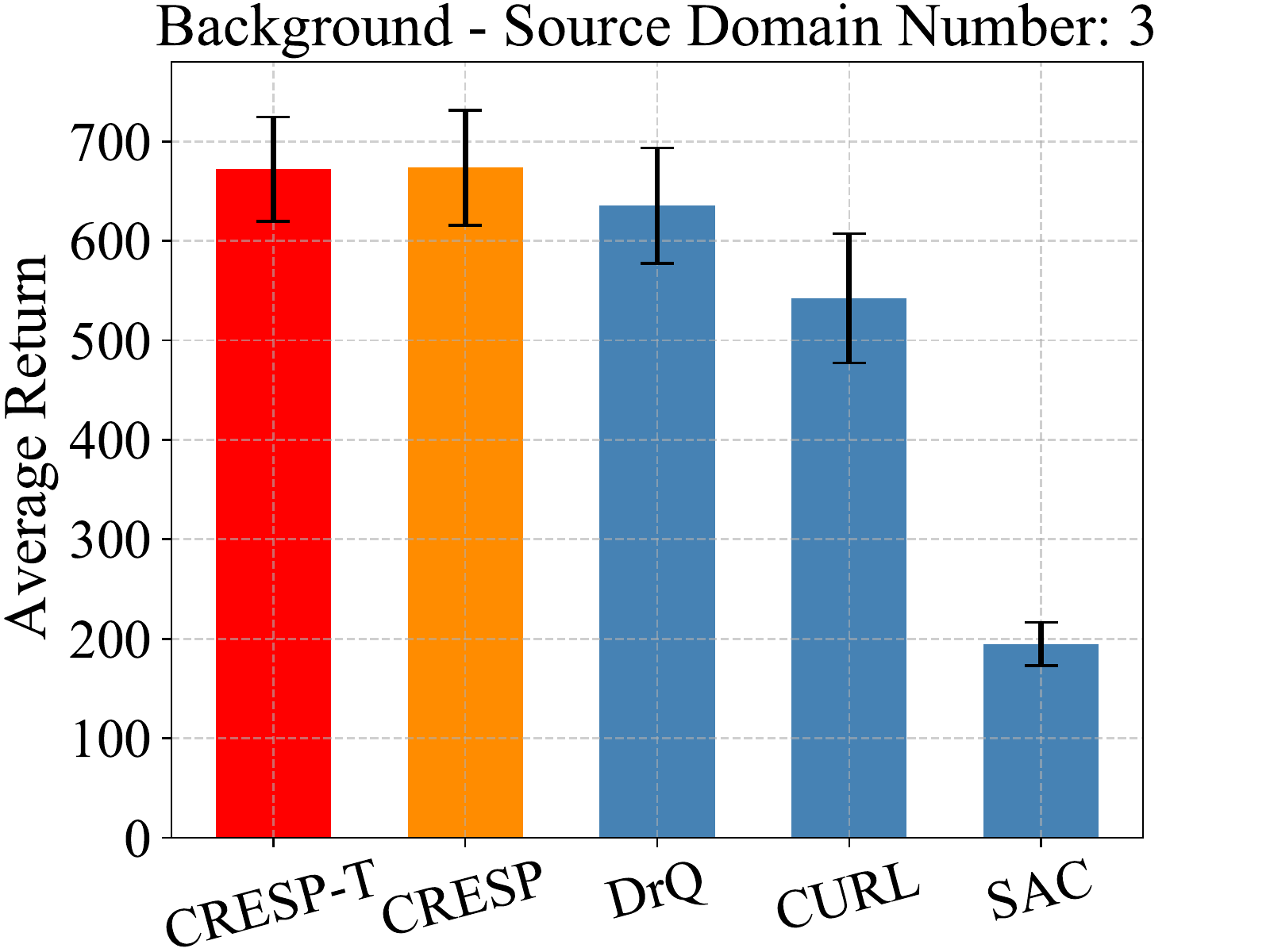}
  \includegraphics[width=4.4cm]{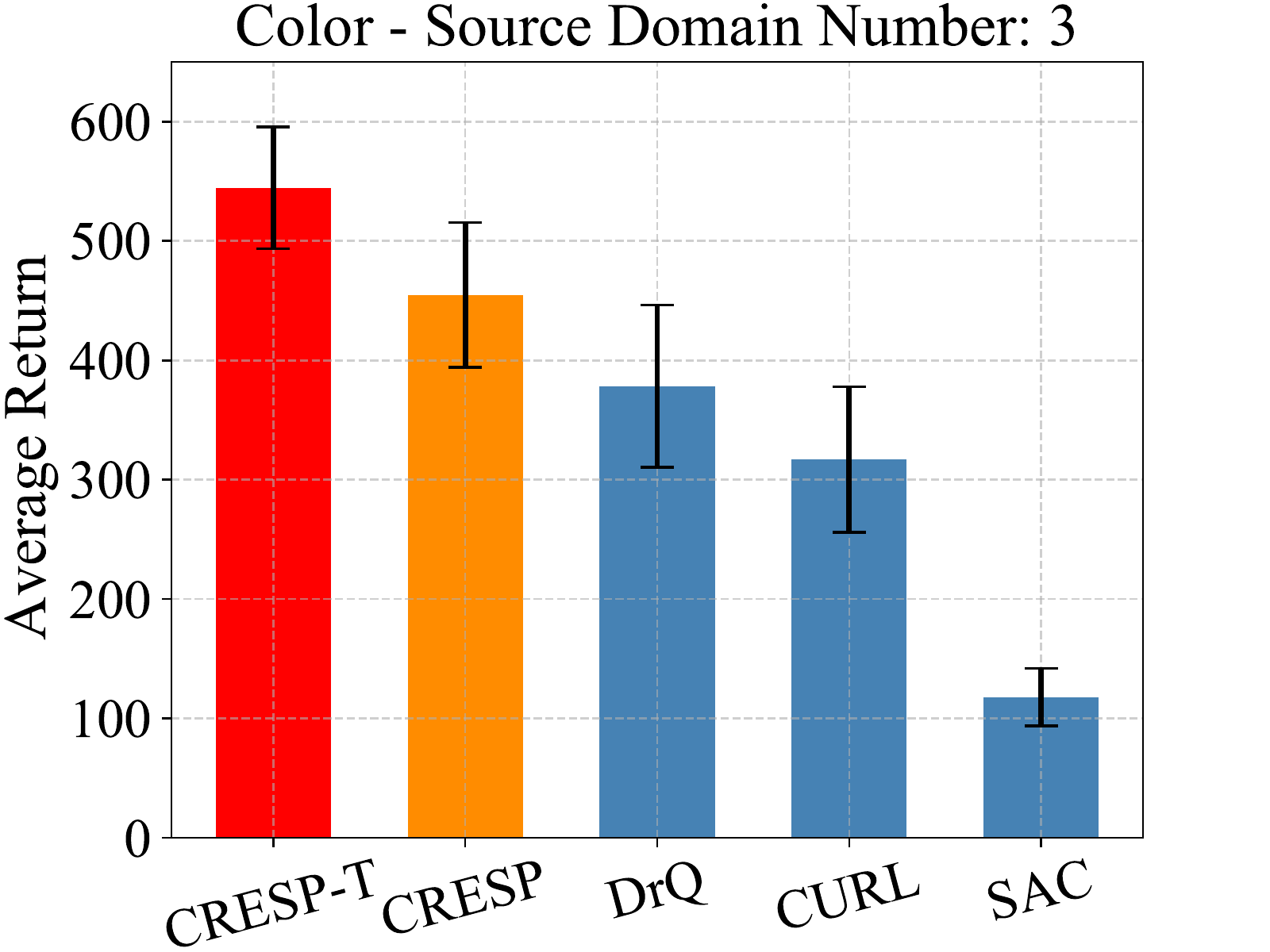}
  \caption{We report the mean and standard error results averaged over 6 DCS tasks under \textbf{one} and \textbf{three} training environment settings with 6 random seeds at 500K steps, respectively.}
  \label{fig-bar-avg-1}
\end{figure}

\subsubsection{Main Results under Dynamic Color Distractions}
\label{subsec-exp-c}
We then select another type of visual distractions, dynamic color distractions~\cite{corr/abs-2101-02722}, to further demonstrate that our method learns representations that ignore task-irrelevant distractions.
In practice, we change the color of the agent (i.e., the object under control) smoothly during episodes.
The variation of colors is modeled as a Gaussian distribution, whose mean is the color at the previous time and whose standard deviation is set as a hyperparameter $\beta$.
Specifically, we choose one, two, and three training environments with random standard deviations from $\beta=0.1$ to $\beta=0.2$. We then evaluate the agents under unseen test environments with $\beta=0.5$.

In Figure~\ref{fig-result-dcs}, the line charts demonstrate that the representations learned by \ourM{-T} under one, two, and three training environments are better for generalization than those learned by other methods.
According to the results in Table~\ref{table-result}, \ourM{-T} achieves an average gain of $+45.3\%$ over the best prior method under the test environment with color distractions, and \ourM{} also achieves an average gain of $+27.7\%$. 

\subsubsection{Additional Curves}
In Figures~\ref{fig-result-dc-1}, ~\ref{fig-result-db-1}, ~\ref{fig-result-dc-2}, ~\ref{fig-result-db-2}, ~\ref{fig-result-dc-3}, and~\ref{fig-result-db-3}, we show learning curves under in 1/2/3 training environment setting with dynamic background/color distractions. Notice that these curves are the training processes of the main results in Section~\ref{subsec-6.2}, which record the generalization abilities of agents during training. Each checkpoint in Figures~\ref{fig-result-dc-1}, ~\ref{fig-result-db-1}, ~\ref{fig-result-dc-2}, ~\ref{fig-result-db-2}, ~\ref{fig-result-dc-3}, and~\ref{fig-result-db-3} are evaluated by 10 episodes on unseen test environments.

\subsection{Hyperparameter Selection and Ablation Studies}
\label{app-hs-as}

\subsubsection{The Random Variable subject to The Gaussian Distribution}
\label{app-pred-loss}
For all our experiments, we use the predictive loss $\mathcal{L}^{\mathcal{N}}(\Phi, \Psi|\mathcal{D})$ to learn task-relevant representations. For the computation of $\mathbb{E}_{\Omega\sim\mathcal{N}}\left[\cdot\right]$ in such predictive loss function, we conduct experiments to select the sample number $\kappa$ of $\Omega$ in the left of Figure~\ref{fig-k}.
Results show that $\kappa=256$ performs better than others. Moreover, for the choice of the distribution $\mathcal{N}$, in \ourM{}, we parameterize the Gaussian distribution as $\mathcal{N}$ and update hyperpararemeters of $\mathcal{N}$ to maximize the predictive loss (N-max) or minimize it (N-min). We compare N-max and N-min with N-sta---use the standard Gaussian distribution as $\mathcal{N}$ in \ourM{}---in the right of Figure~\ref{fig-k}. Results indicate that N-max is not stable and N-min is similar to N-sta. For computational efficiency, we choose the standard Gaussian distribution as $\mathcal{N}$.

\begin{figure}
  \centering
  \includegraphics[width=4.3cm]{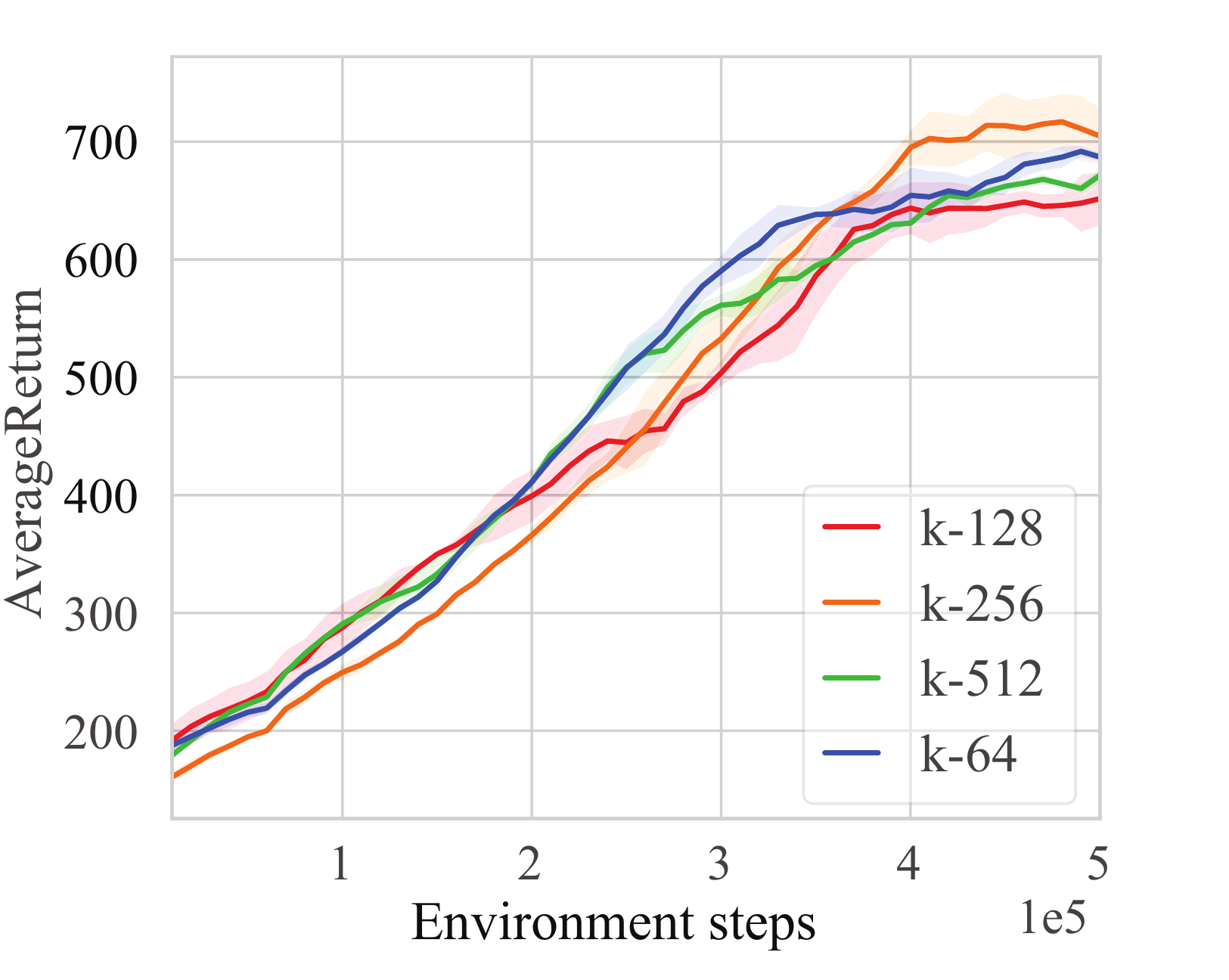}
  \includegraphics[width=4.3cm]{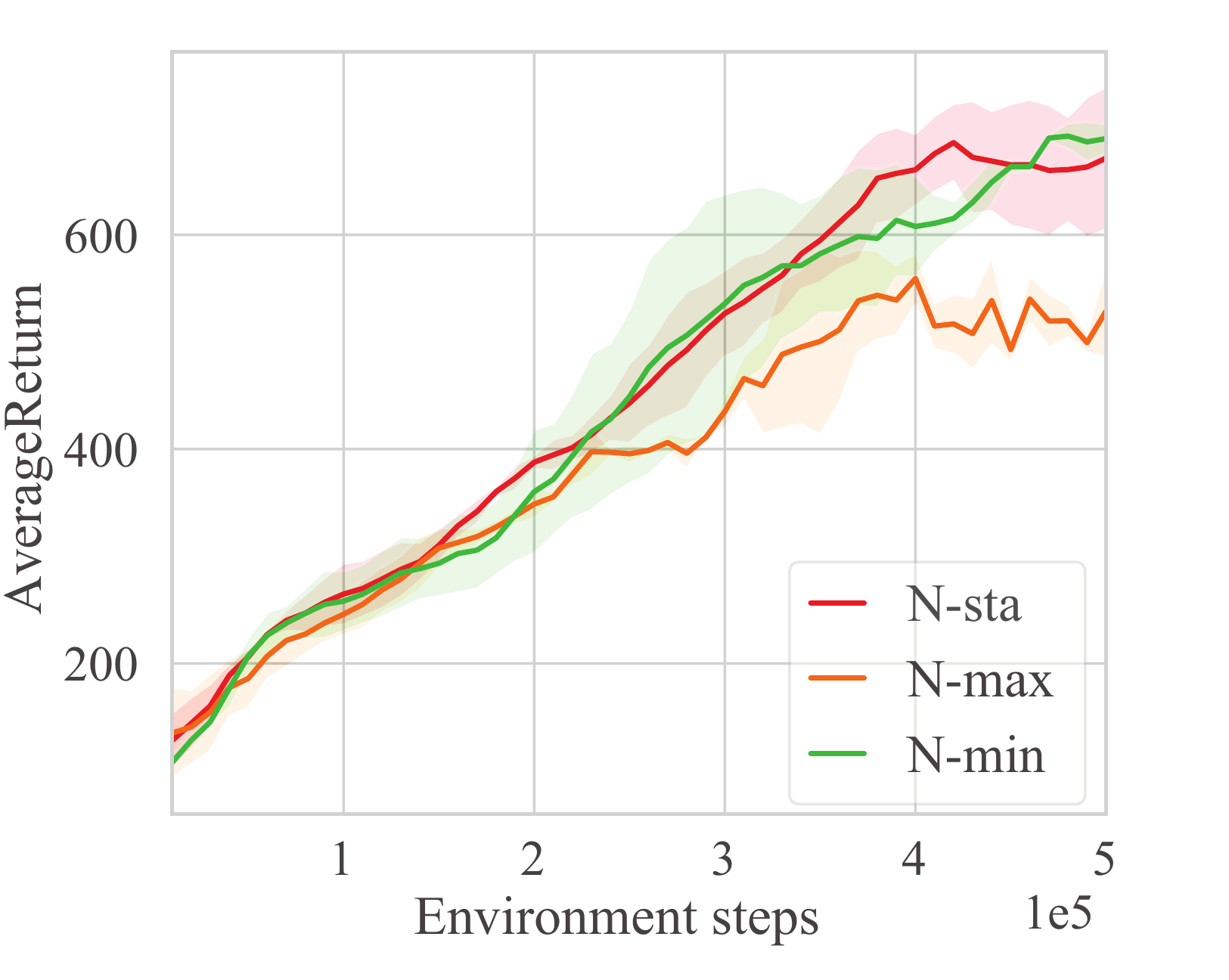}
  \caption{Hyperparameter selection of \ourM{} with the batch size 128 over 2 trails on Cartpole-swingup under two training environments with dynamic backgrounds. In the left, $\kappa$ is sample number of $\Omega$ from $\mathcal{N}$. In the right, N-sta is \ourM{} using the standard Gaussian distribution as $\mathcal{N}$. Both N-max and N-min are \ourM{} using the parameterized Gaussian distribution as $\mathcal{N}$. Furthermore, N-max and N-min update the hyperparameters of $\mathcal{N}$ to maximize and minimize the predictive loss function, respectively.}
  \label{fig-k}
\end{figure}

\subsubsection{Reward Sequence Distributions vs Distributions of Sum of Reward Sequences}
\label{app-as}
We introduce RSD-OA to capture task-relevant information in both rewards and transition dynamics.
In \ourM{}, we propose to use the characteristic functions of RSD-OA as supervised signals to learn representations.
Notice that we can also predict (1. \textit{RSP}) the expectations of RSD-OA, (2. \textit{RSP-Sum}) the expectations of the sum of RSD-OA, and (3. \textit{\ourM{-Sum}}) the characteristic functions of sum of RSD-OA, where the sum of RSD-OA is the distribution of sum of the reward sequence conditioned on the starting observation and predefined subsequent action sequence.
Therefore, we conduct a simple ablation study about \ourM{} with a 3-layer MLP predictor to evaluate the differences among RSP, RSP-Sum, \ourM{}, and \ourM{-Sum}. 
We adopt the same hyperparameters except for the reward length $T$.

All results in Table~\ref{table-sum} are in two training environments with dynamic background distractions. We boldface the results that have highest means. Notice that for reward length $T=1$, RSP is same as RSP-Sum, and \ourM{} is same as \ourM{-Sum}.
The average performance of RSP-Sum is lower than that of RSP for all different reward lengths. \ourM{} also outperforms \ourM{-Sum}. These results show that the high-dimensional targets can provide more helpful information than one-dimensional targets. Such results are similar to the effectiveness of knowledge distillation by a soft target distribution.
Moreover, \ourM{} achieves the best average performance, outperforming RSP, RSP-Sum, and \ourM{-Sum}, which empirically demonstrates that learning distributions provides benefits for representation learning.

\begin{table}
  \caption{We report the means and standard errors with 3 seeds on Cartpole-swingup with dynamic backgrounds at 500K environment steps. Notice that $T$ is the length of reward sequences using for representation learning. Best results are in boldfaced.}
  \label{table-sum}
  \centering
  \begin{tabular}{c|cccc}
  \toprule
  Rew Length & RSP & RSP-Sum & \ourM{} & \ourM{}-Sum \\
  \midrule
  $T=1$  & $\textbf{625} \pm \textbf{49}$ & $\textbf{625} \pm \textbf{49}$ & $616 \pm 89$ & $616 \pm 89$\\
  $T=3$  & $645 \pm 32$ & $631 \pm 20$ & $\textbf{666} \pm \textbf{24}$ & $610 \pm 51$\\
  $T=5$  & $575 \pm 64$ & $610 \pm 49$ & $\textbf{687} \pm \textbf{29}$ & $629 \pm 35$\\
  $T=7$  & $654 \pm 56$ & $599 \pm 79$ & $\textbf{667} \pm \textbf{25}$ & $639 \pm 35$\\
  \midrule
  Average  & $625 \pm 54$ & $613 \pm 55$ & $\textbf{658} \pm \textbf{57}$ & $623 \pm 58$\\
  \bottomrule
  \end{tabular}
\end{table}

\subsection{Code}
We implement all of our codes in Python version 3.8 and make the code available online~\footnote{\url{https://github.com/MIRALab-USTC/RL-CRESP}}.
We used NVIDIA GeForce RTX 2080 Ti GPUs for all experiments. Each trials of \ourM{} was trained for 32 hours on average.

\begin{figure*}
  \centering
  \includegraphics[width=9cm]{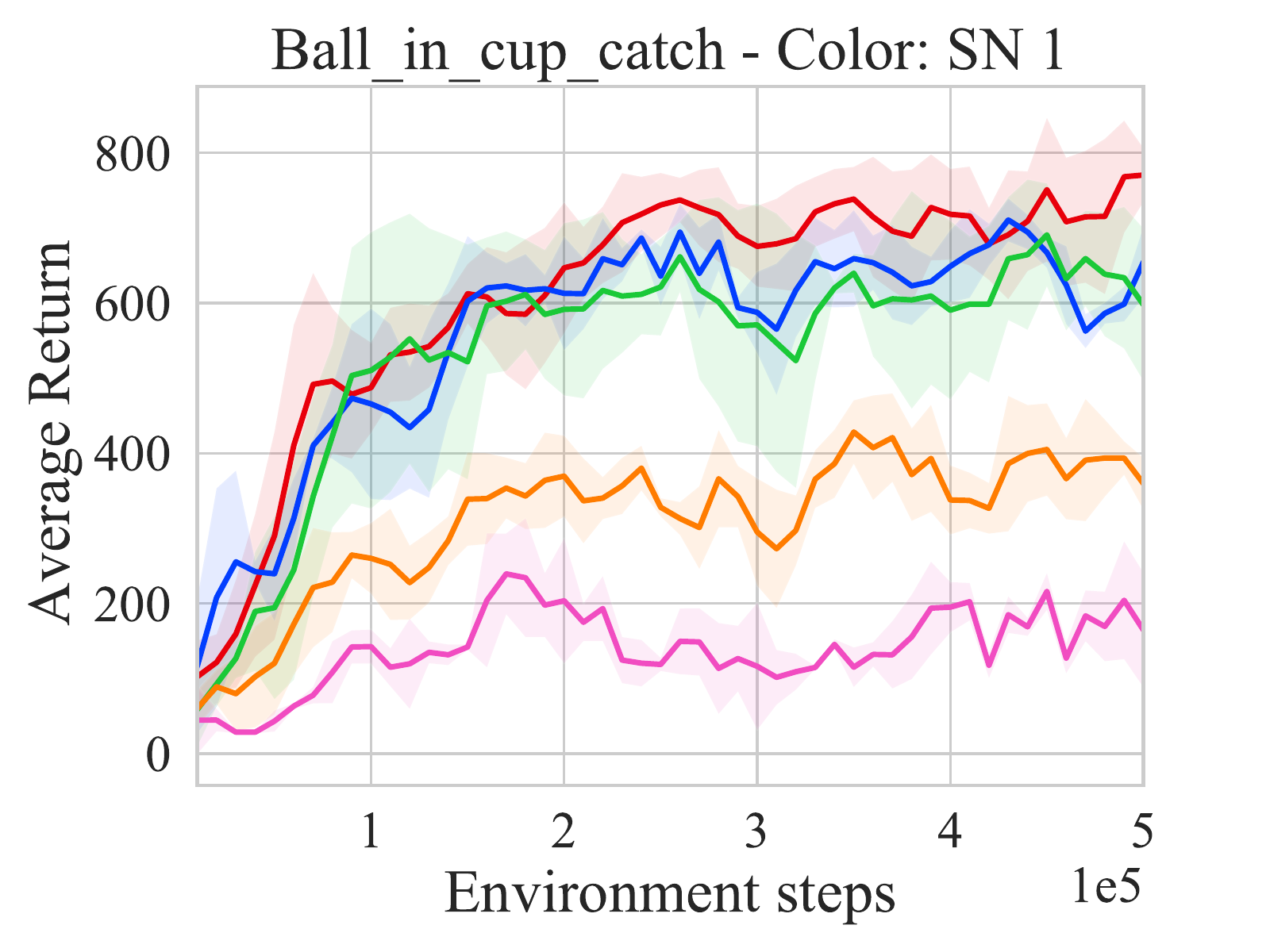}
  \includegraphics[width=9cm]{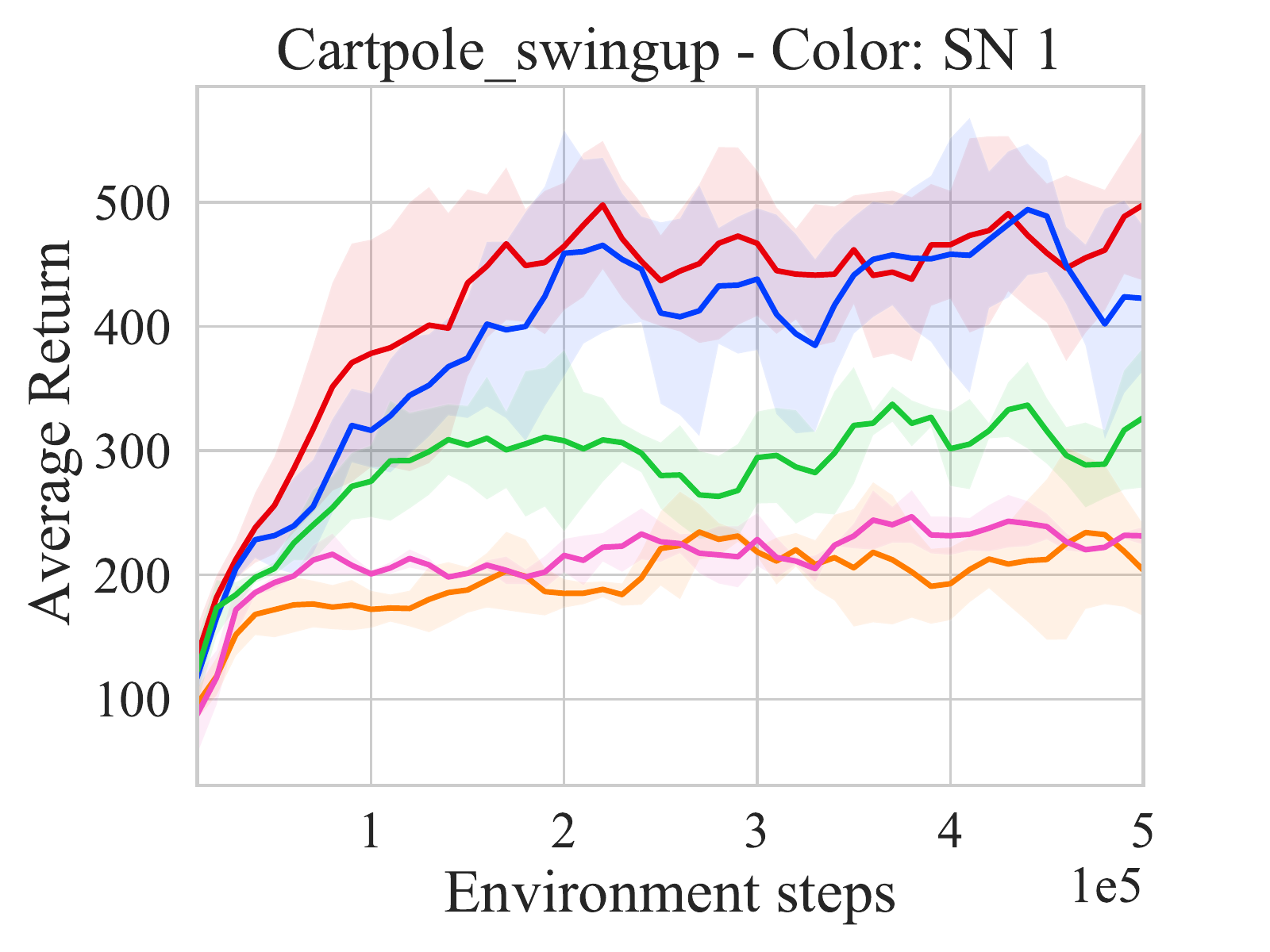}
  
  \includegraphics[width=9cm]{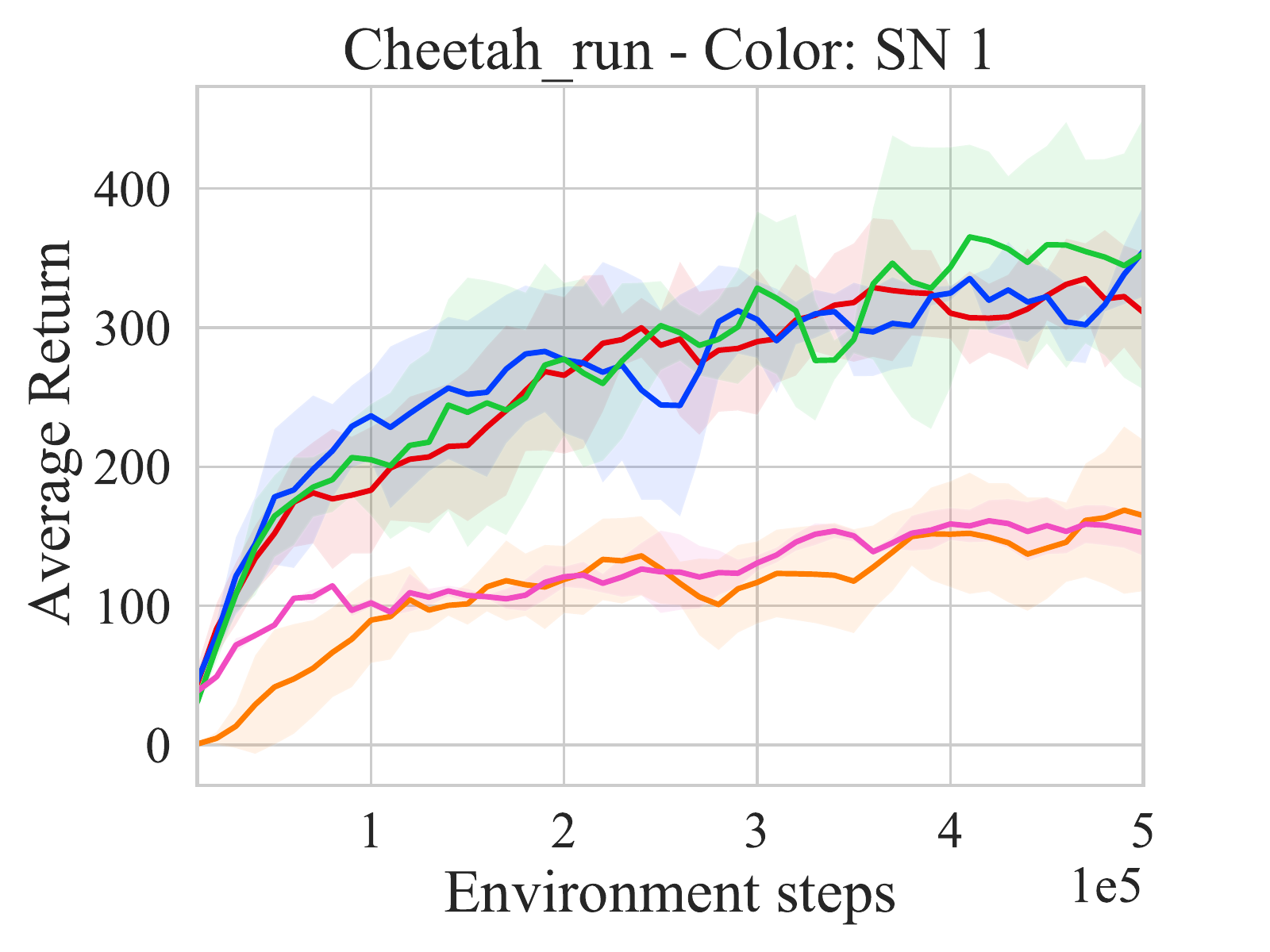}
  \includegraphics[width=9cm]{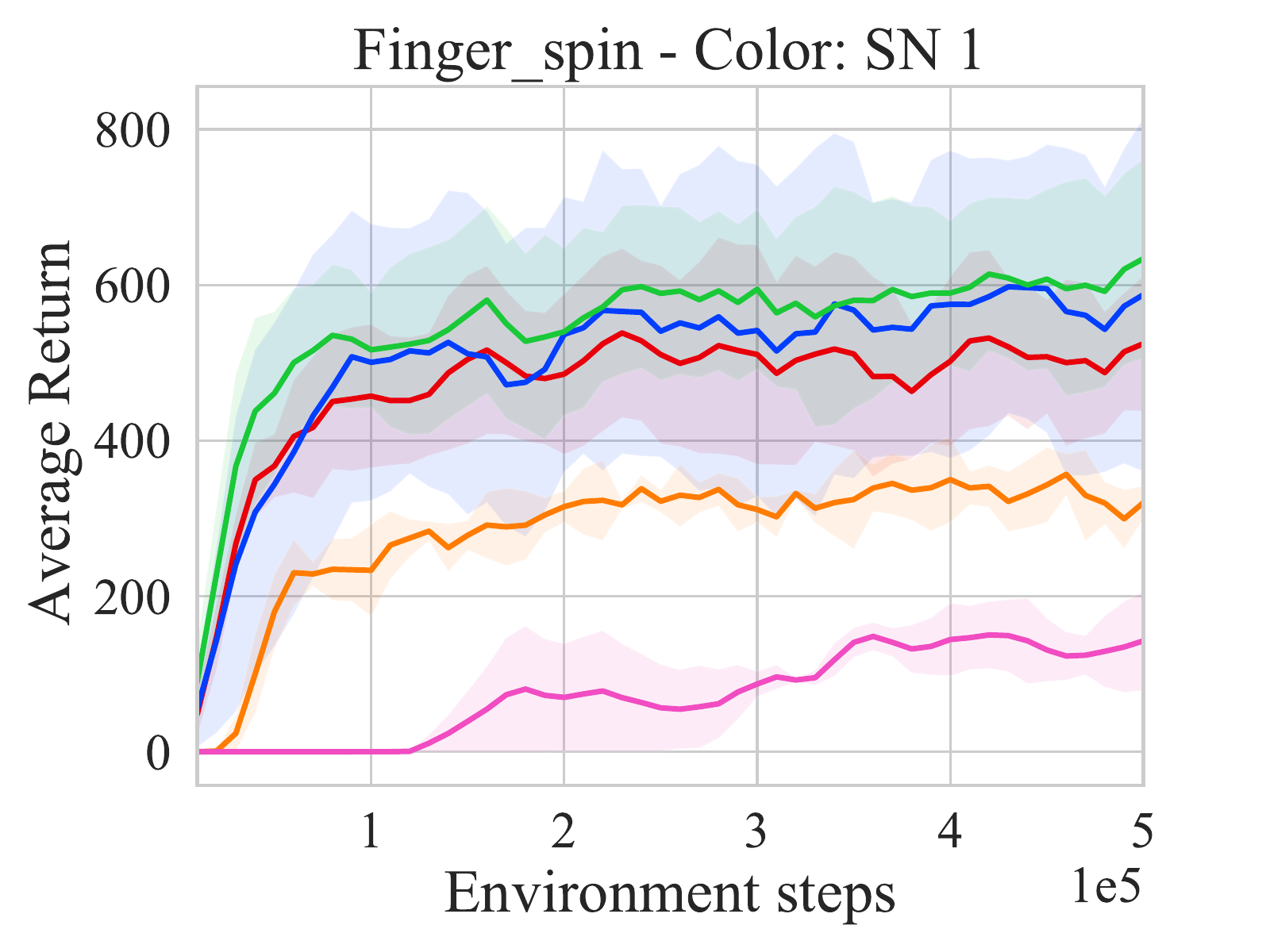}
  
  \includegraphics[width=9cm]{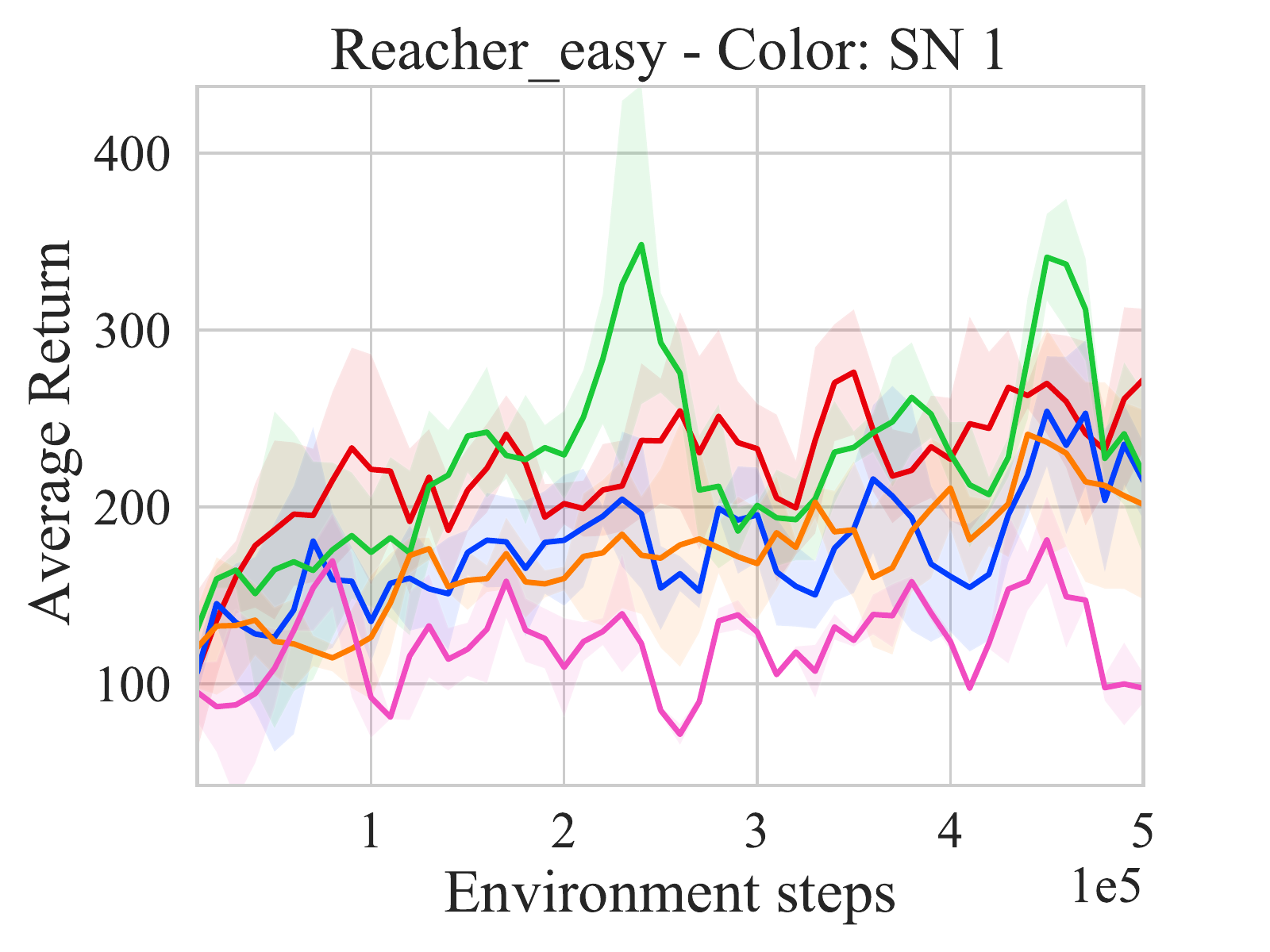}
  \includegraphics[width=9cm]{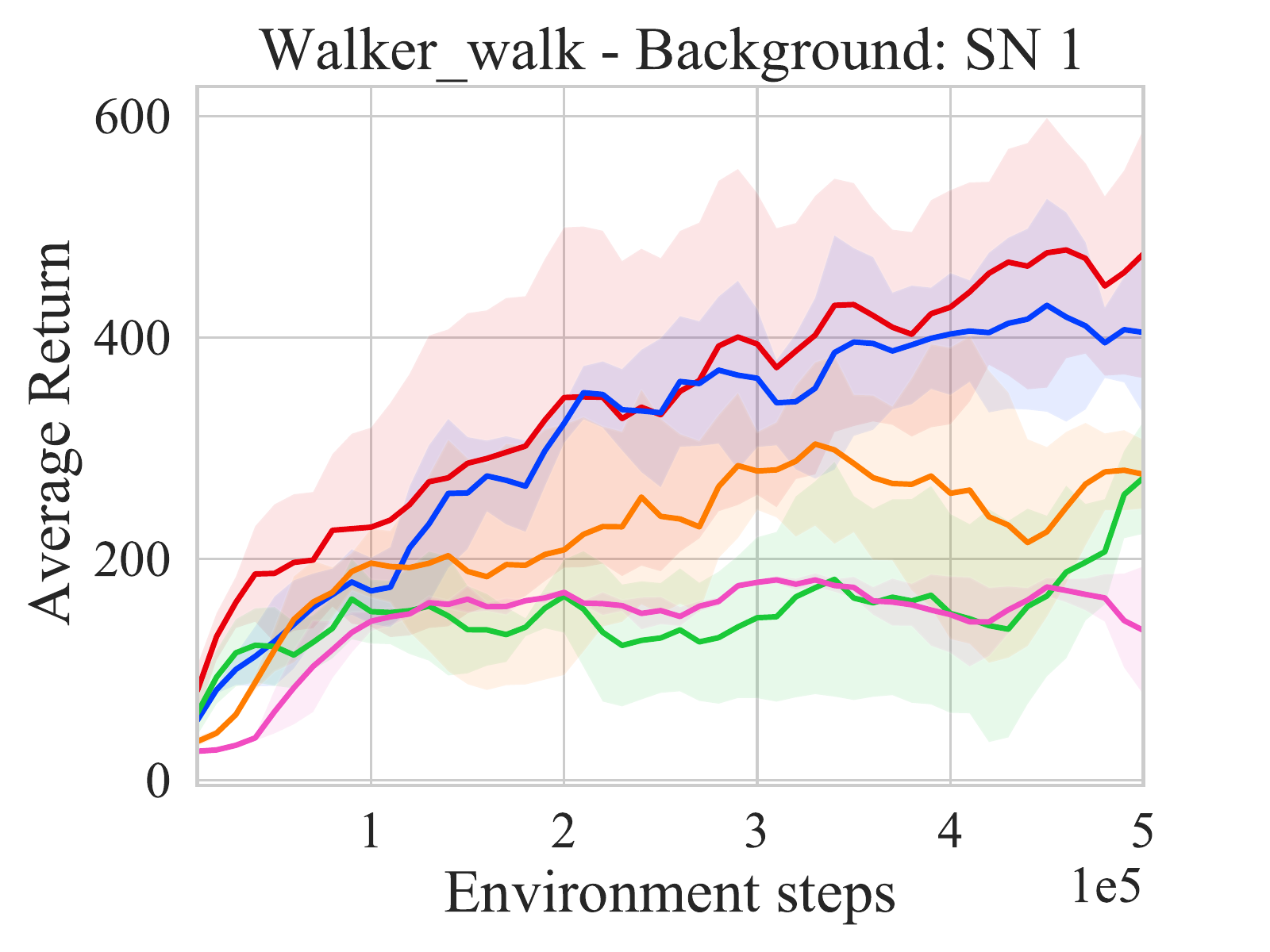}
  
  \includegraphics[width=15cm]{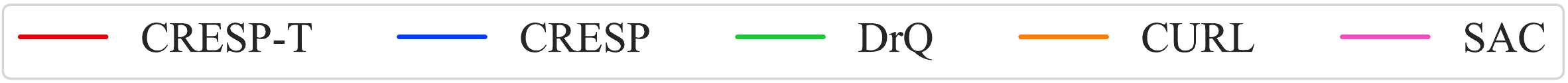}
  \caption{Learning curves on six tasks under the \textbf{one} training environment setting with dynamic \textbf{color} distractions for 500K environment steps. SN denotes the number of source domain, which is the number of training environment.}
\label{fig-result-dc-1}
\end{figure*}

\begin{figure*}
  \centering
  \includegraphics[width=9cm]{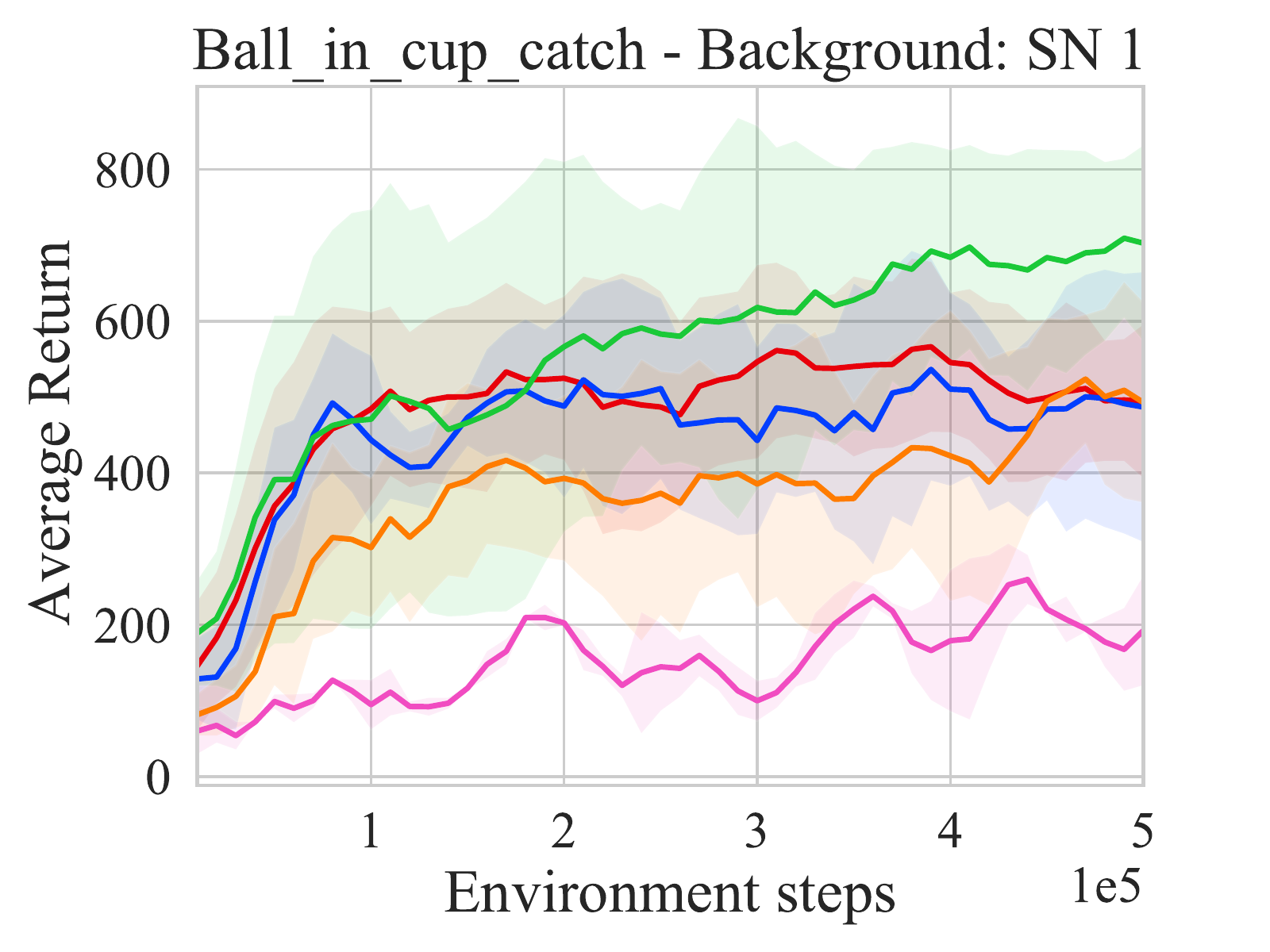}
  \includegraphics[width=9cm]{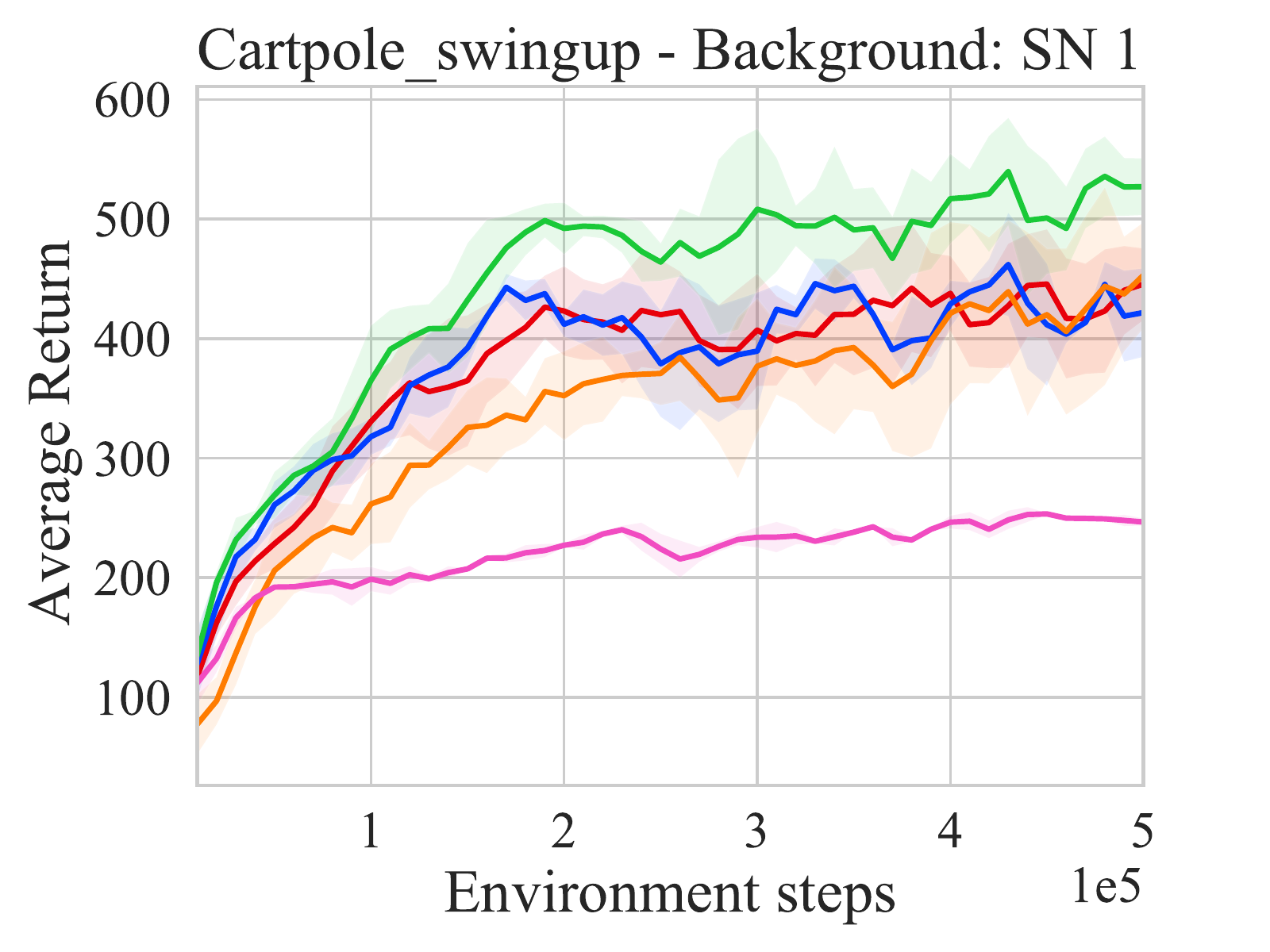}

  \includegraphics[width=9cm]{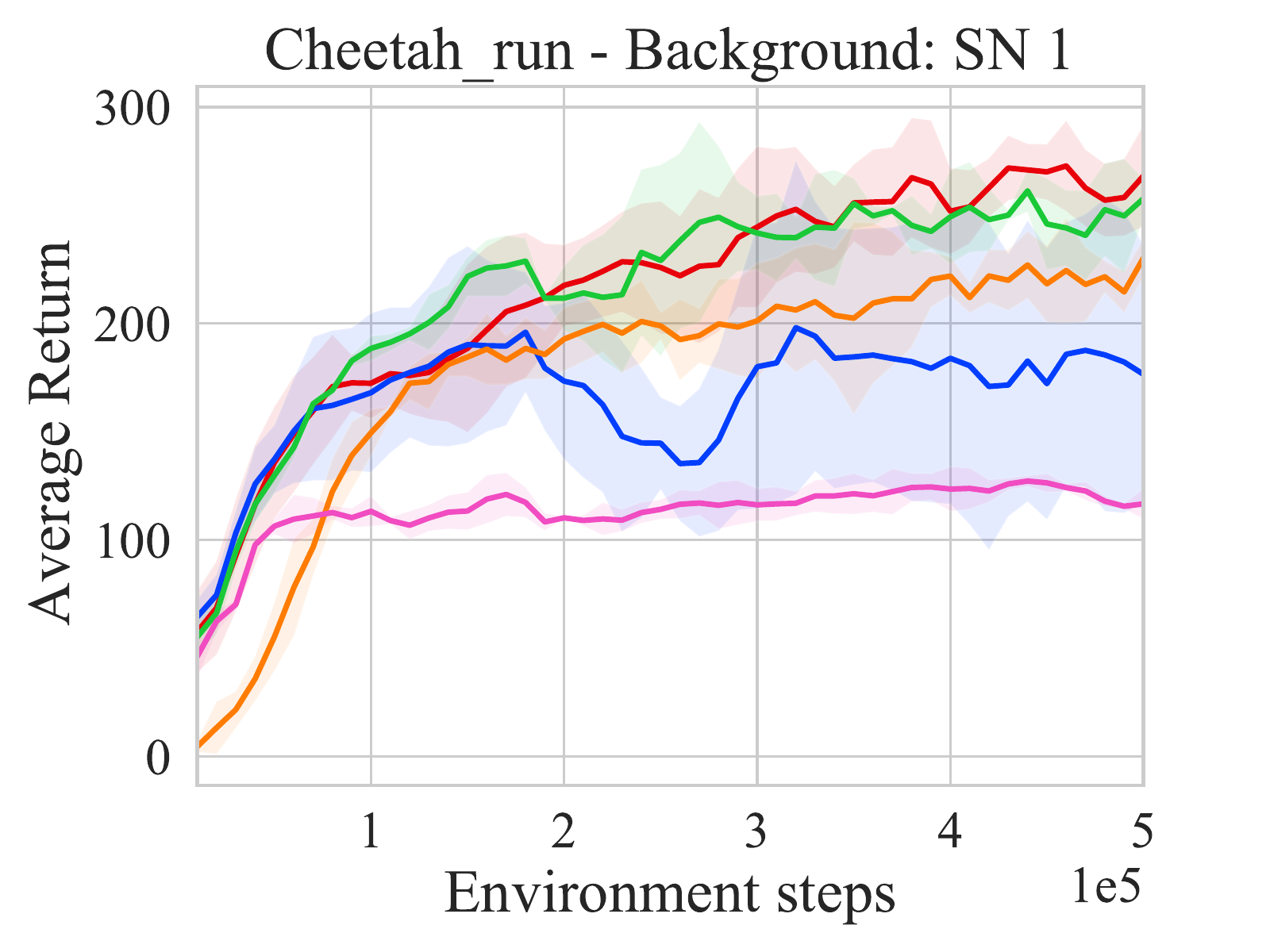}
  \includegraphics[width=9cm]{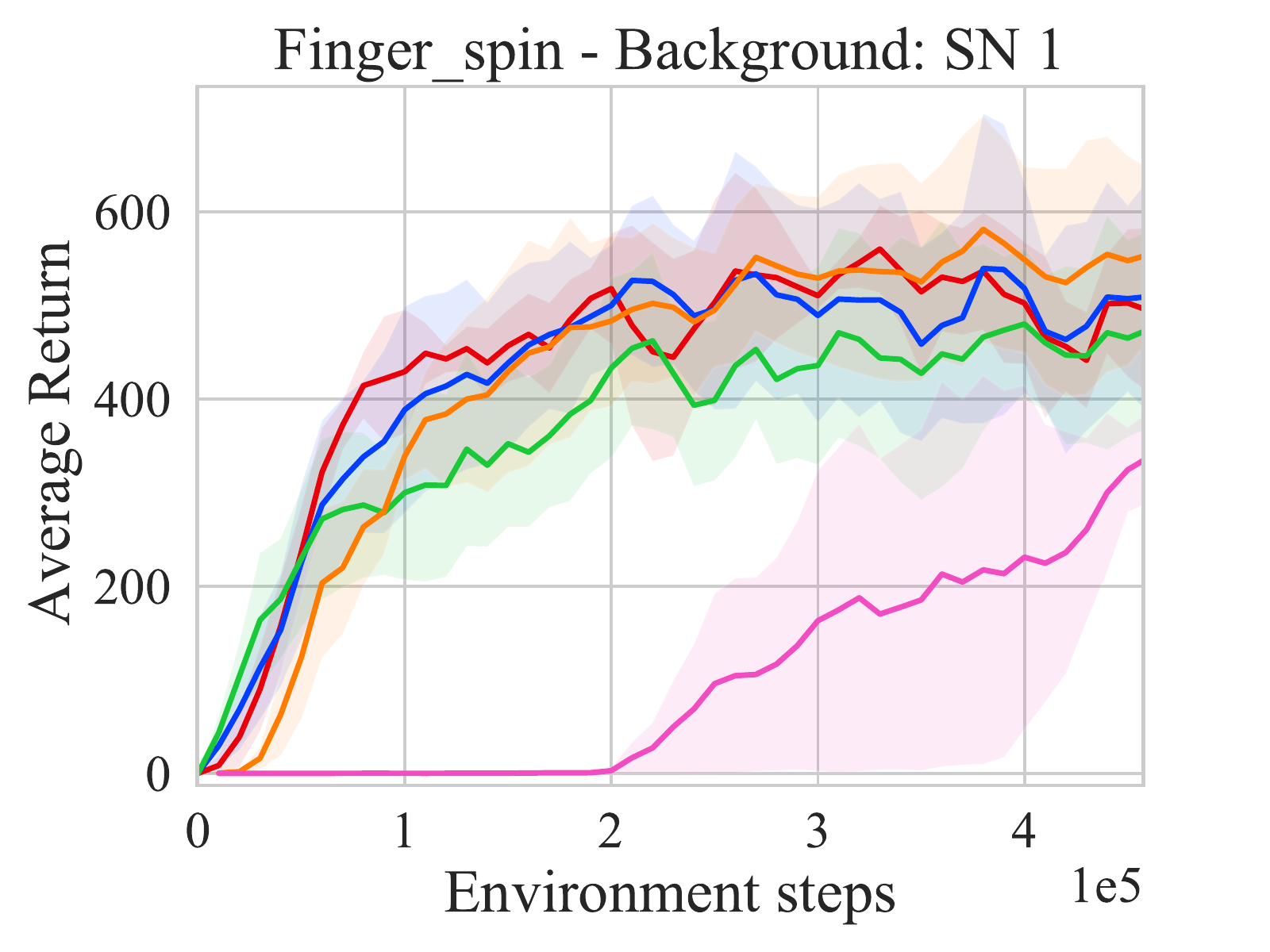}
  
  \includegraphics[width=9cm]{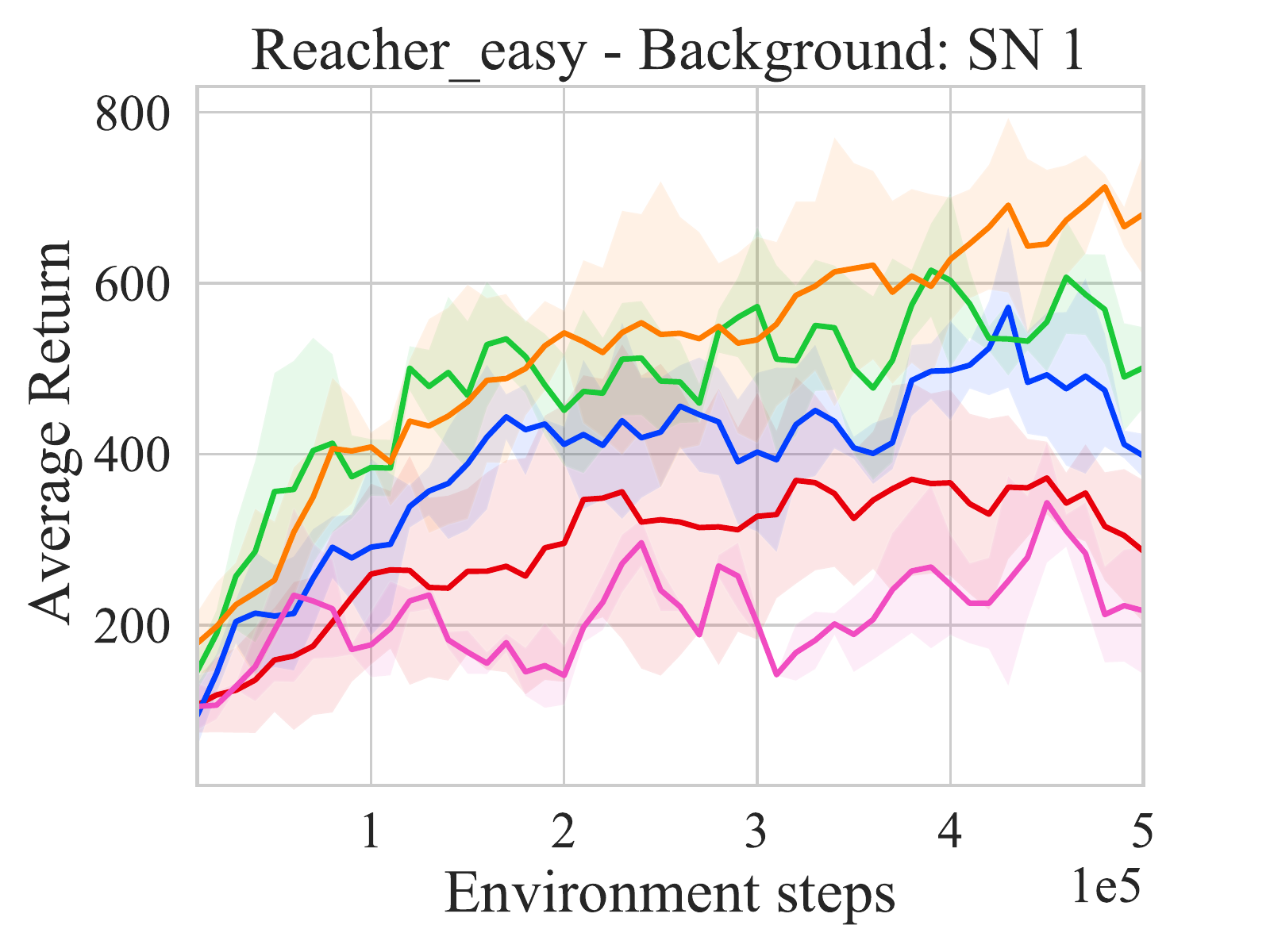}
  \includegraphics[width=9cm]{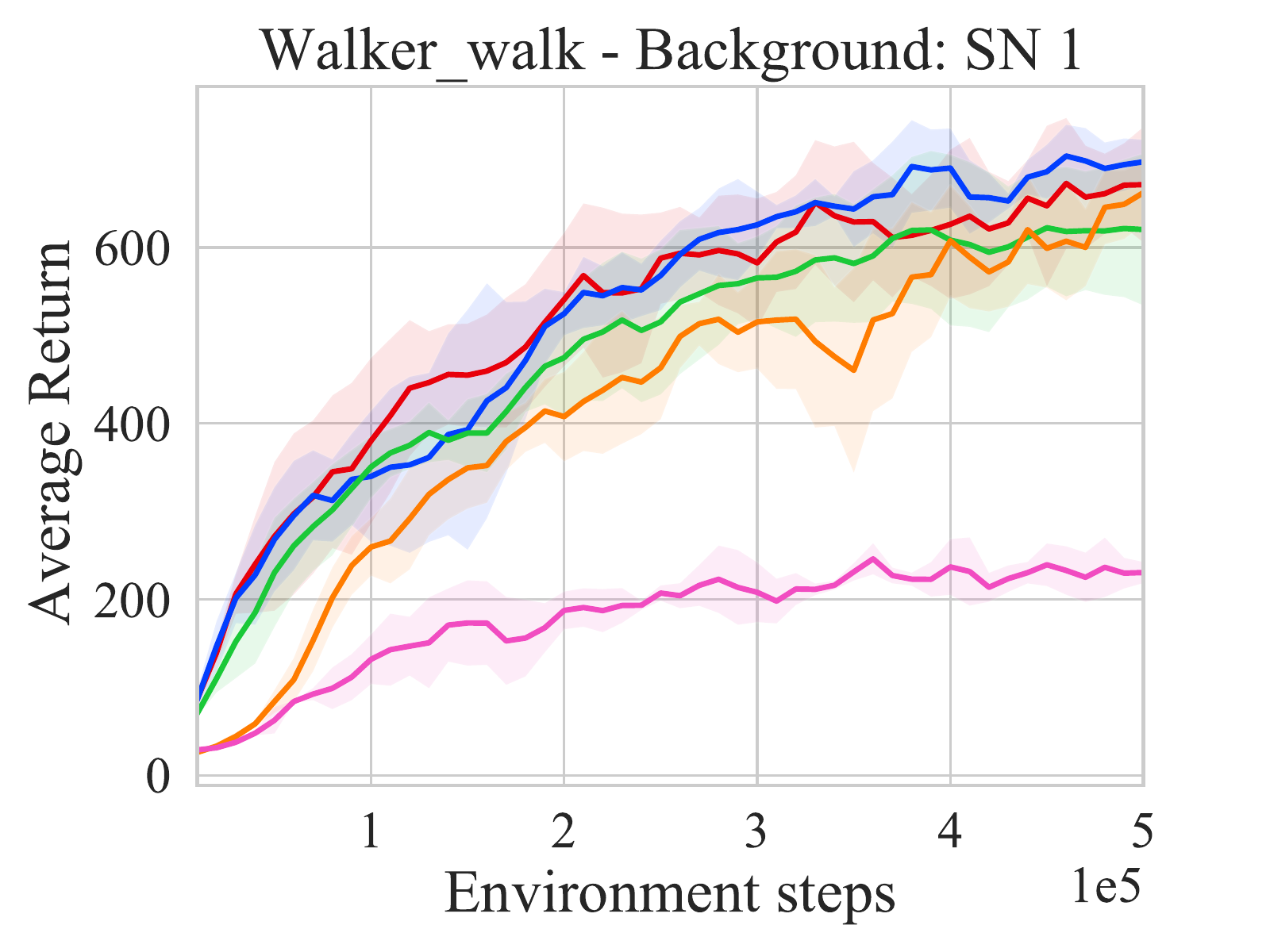}
  
  \includegraphics[width=15cm]{legend-13.pdf}
  \caption{Learning curves on six tasks under the \textbf{one} training environment setting with dynamic \textbf{background} distractions for 500K environment steps. SN denotes the number of source domain, which is the number of training environment.}
\label{fig-result-db-1}
\end{figure*}

\begin{figure*}
  \centering
  \includegraphics[width=9cm]{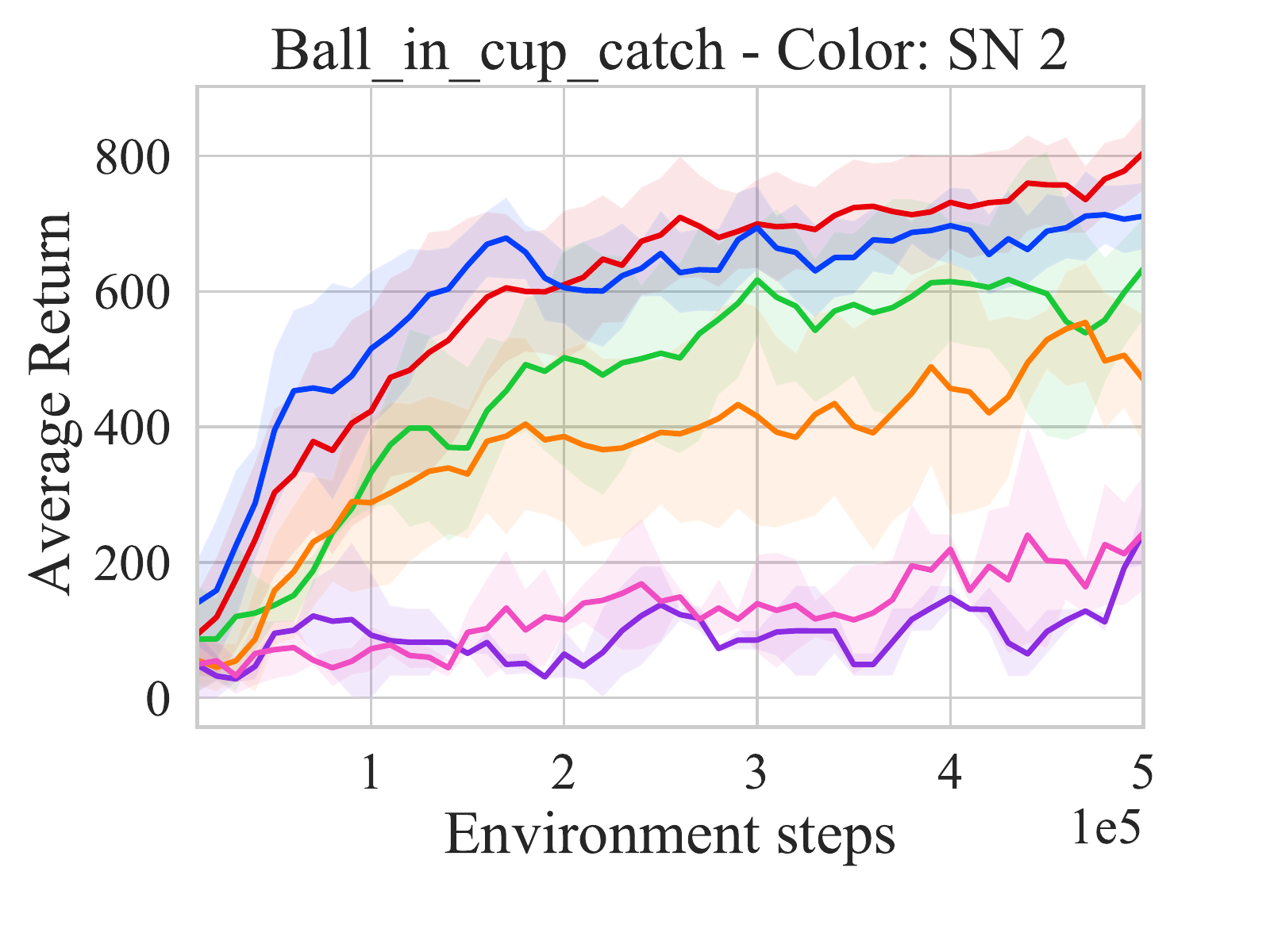}
  \includegraphics[width=9cm]{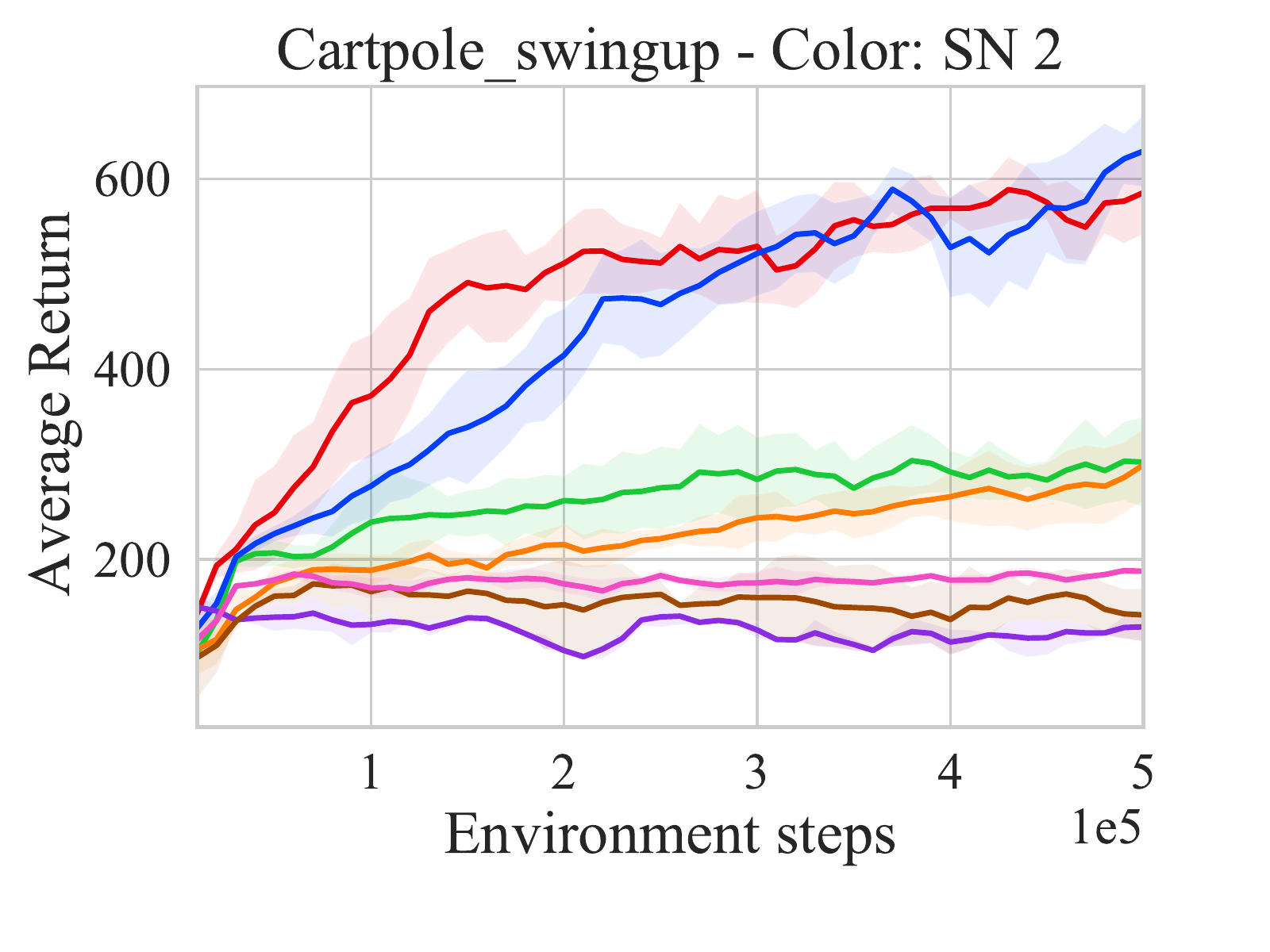}
  
  \includegraphics[width=9cm]{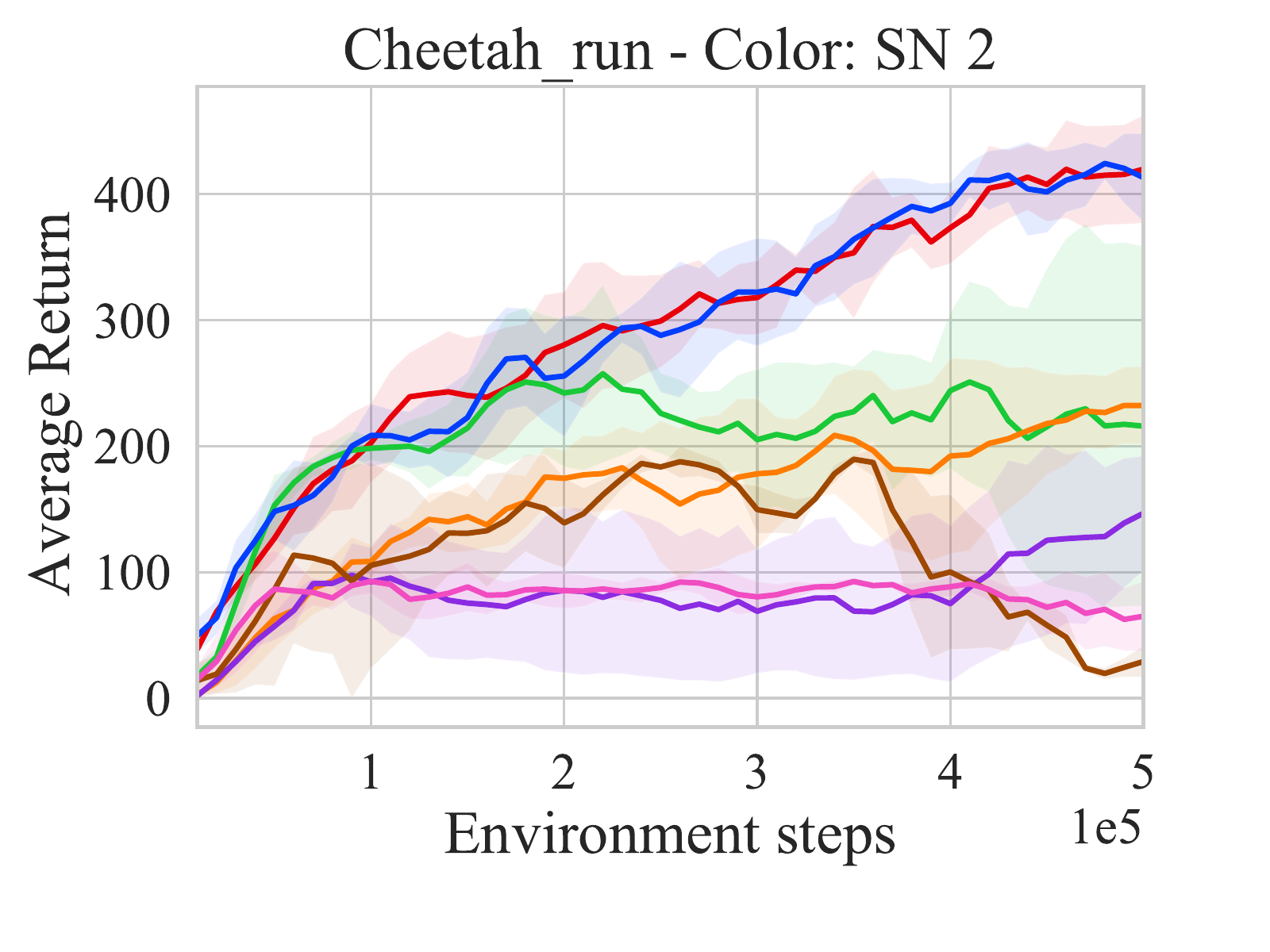}
  \includegraphics[width=9cm]{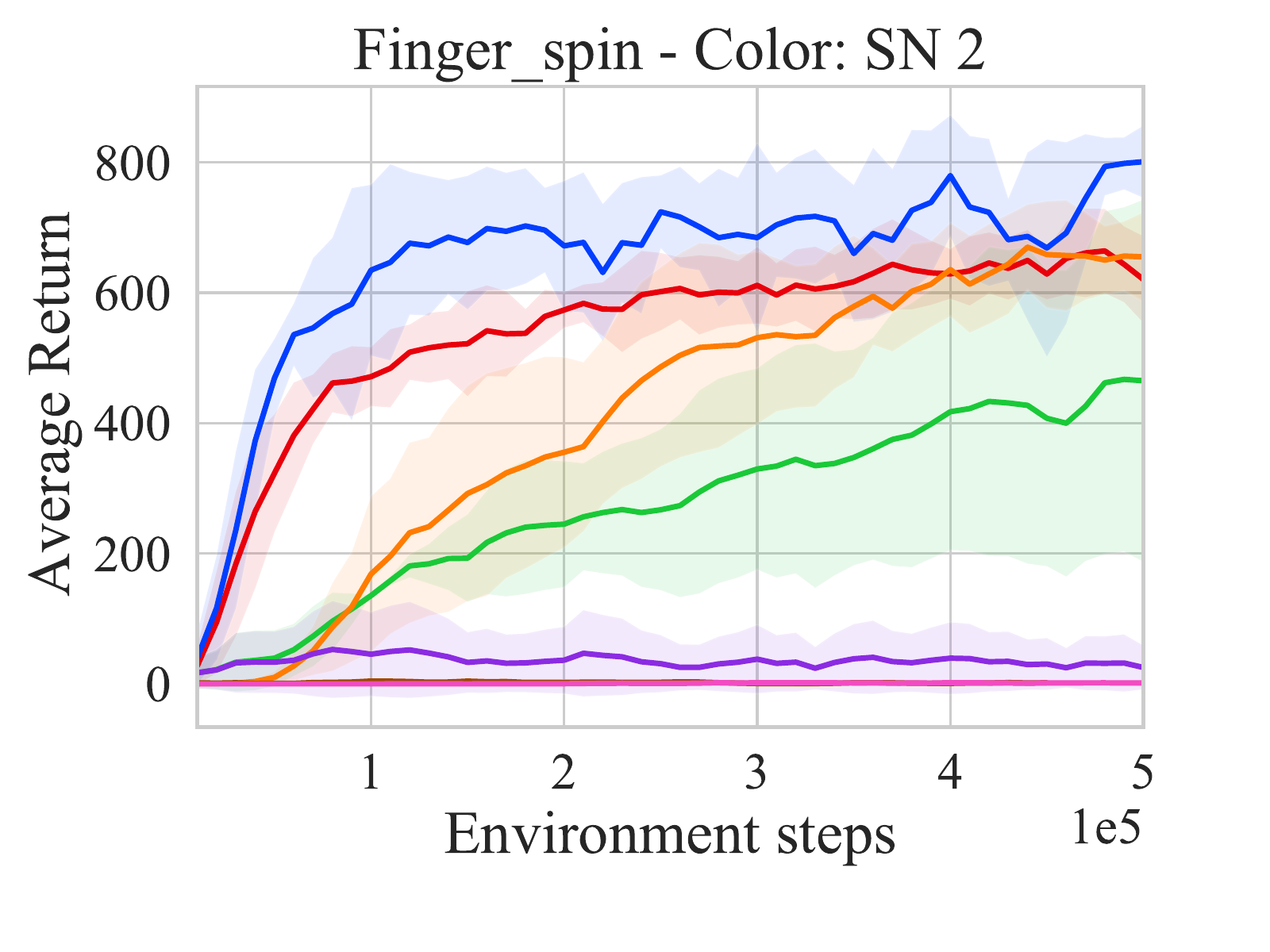}
  
  \includegraphics[width=9cm]{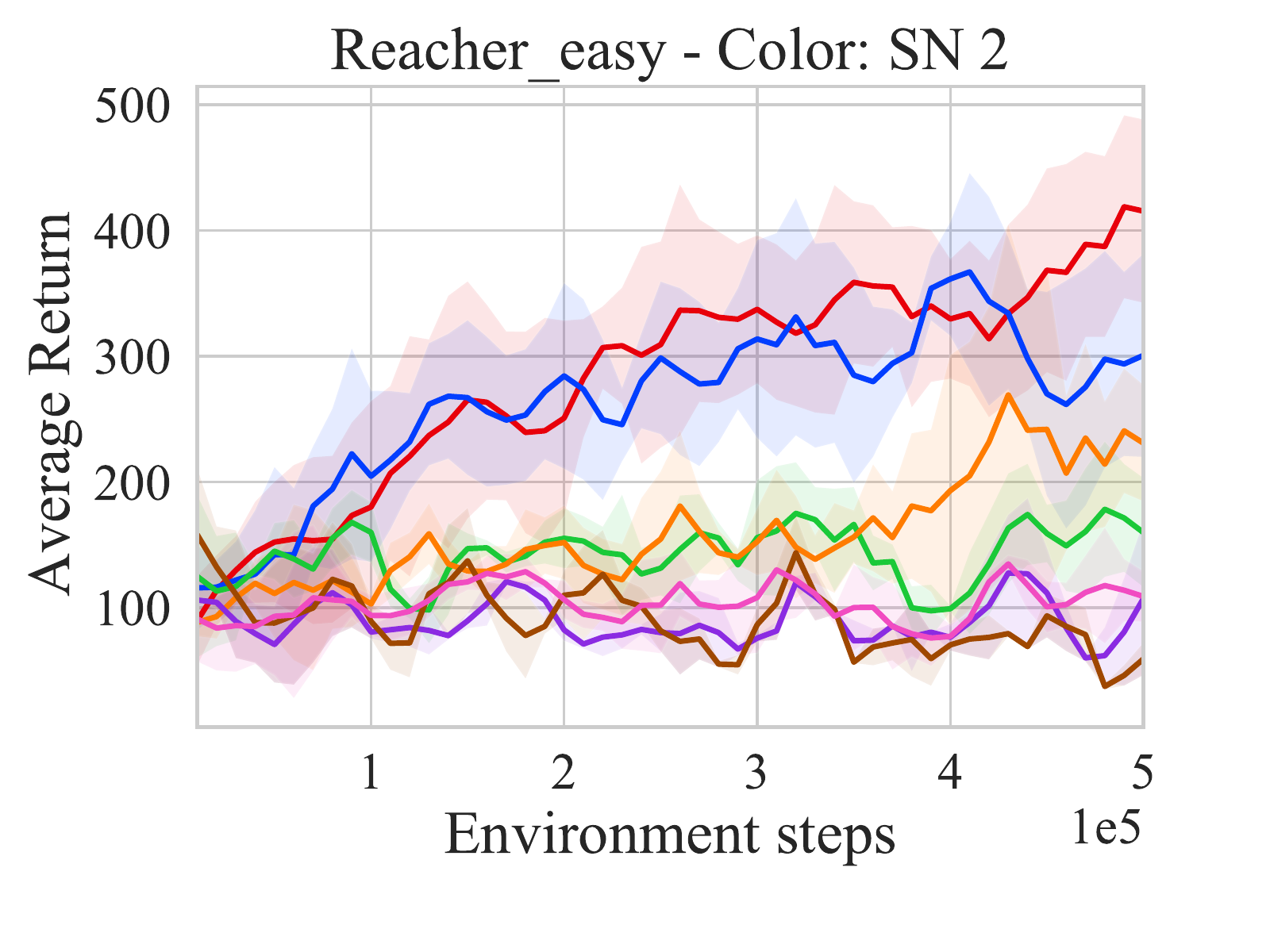}
  \includegraphics[width=9cm]{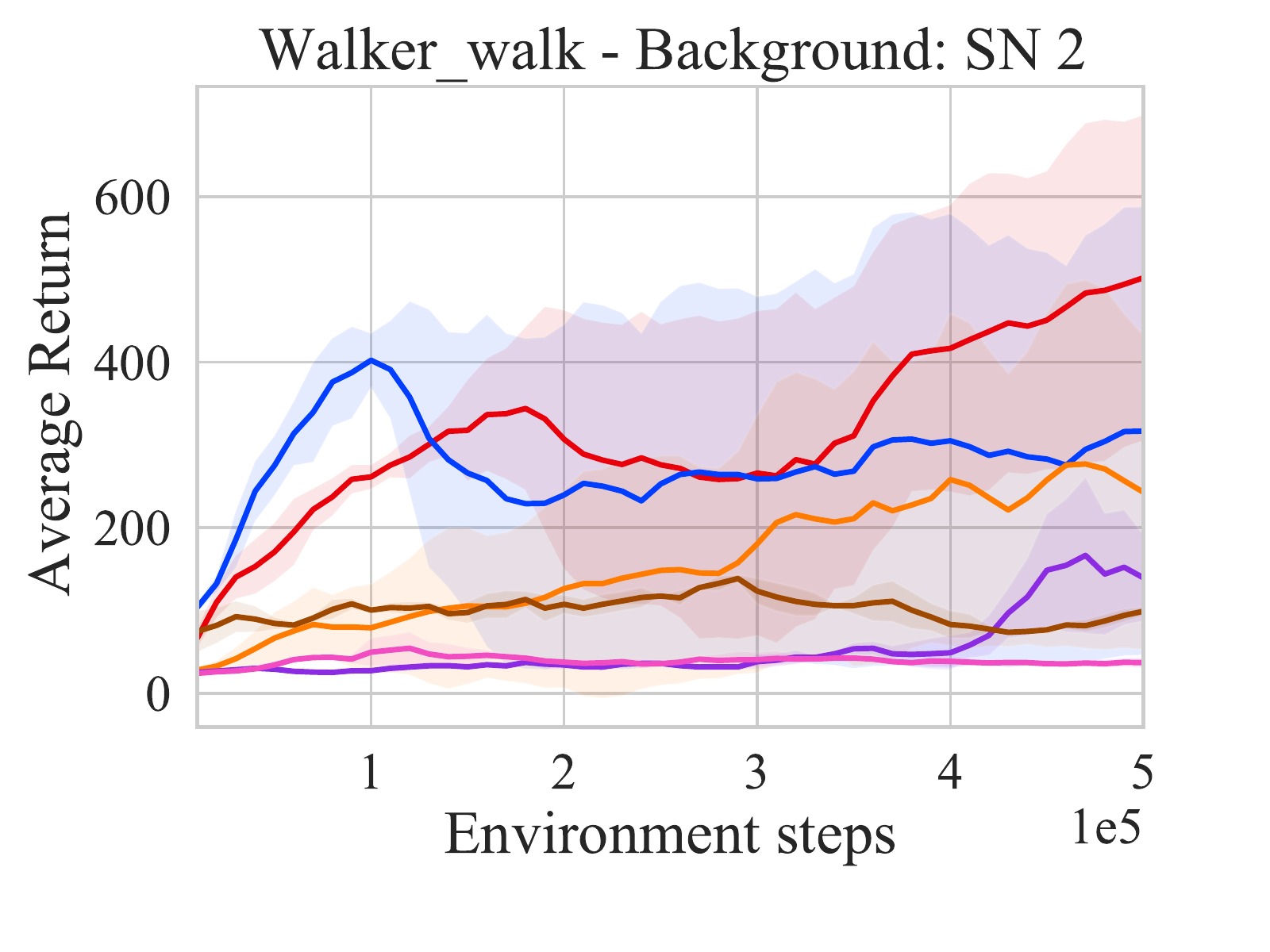}
  
  \includegraphics[width=18cm]{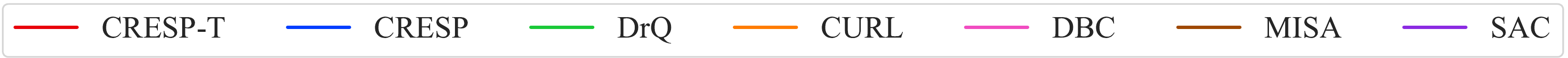}
  \caption{Learning curves on six tasks under the \textbf{two} training environment setting with dynamic \textbf{color} distractions for 500K environment steps. SN denotes the number of source domain, which is the number of training environment.}
\label{fig-result-dc-2}
\end{figure*}

\begin{figure*}
  \centering
  \includegraphics[width=9cm]{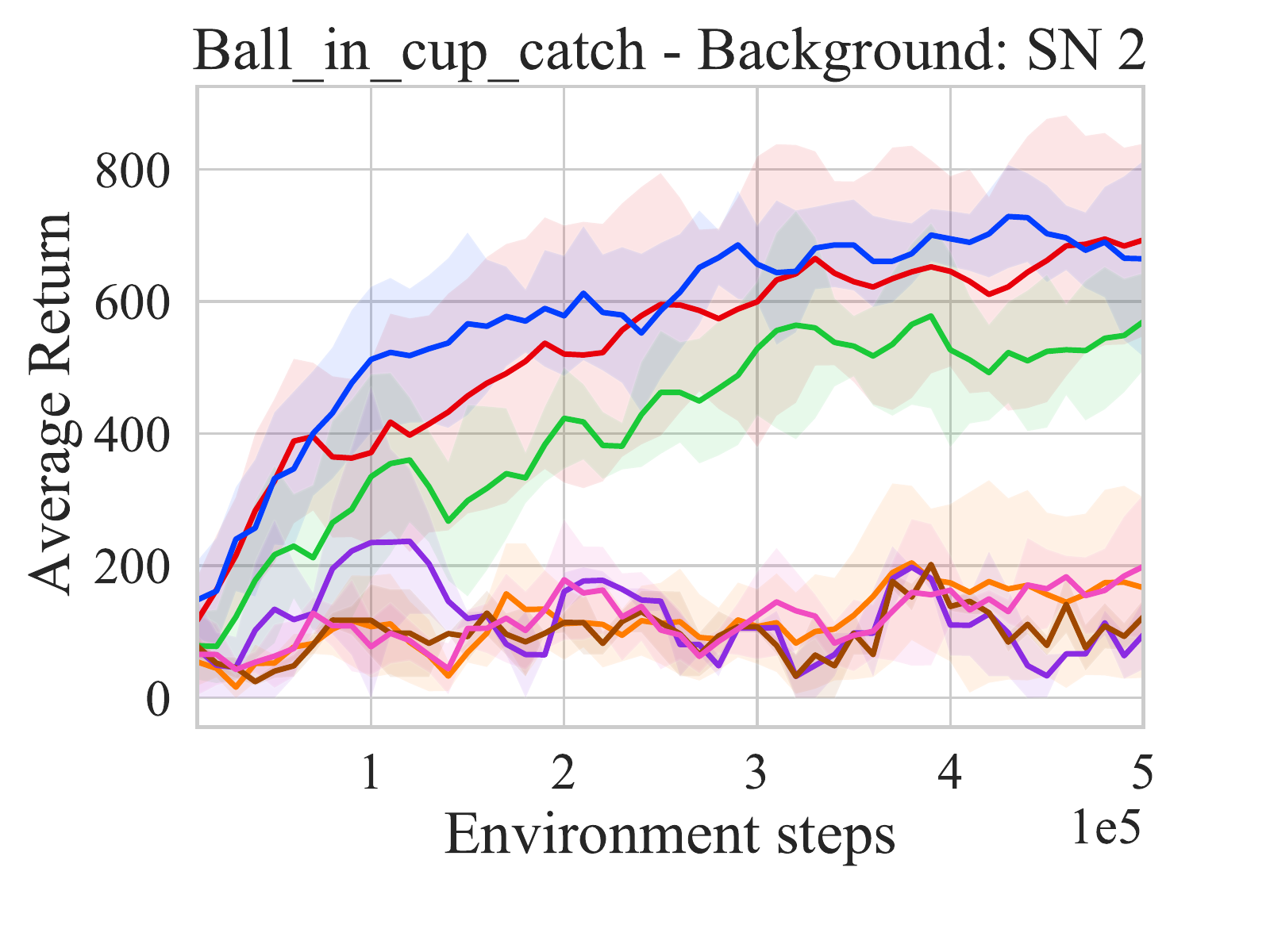}
  \includegraphics[width=9cm]{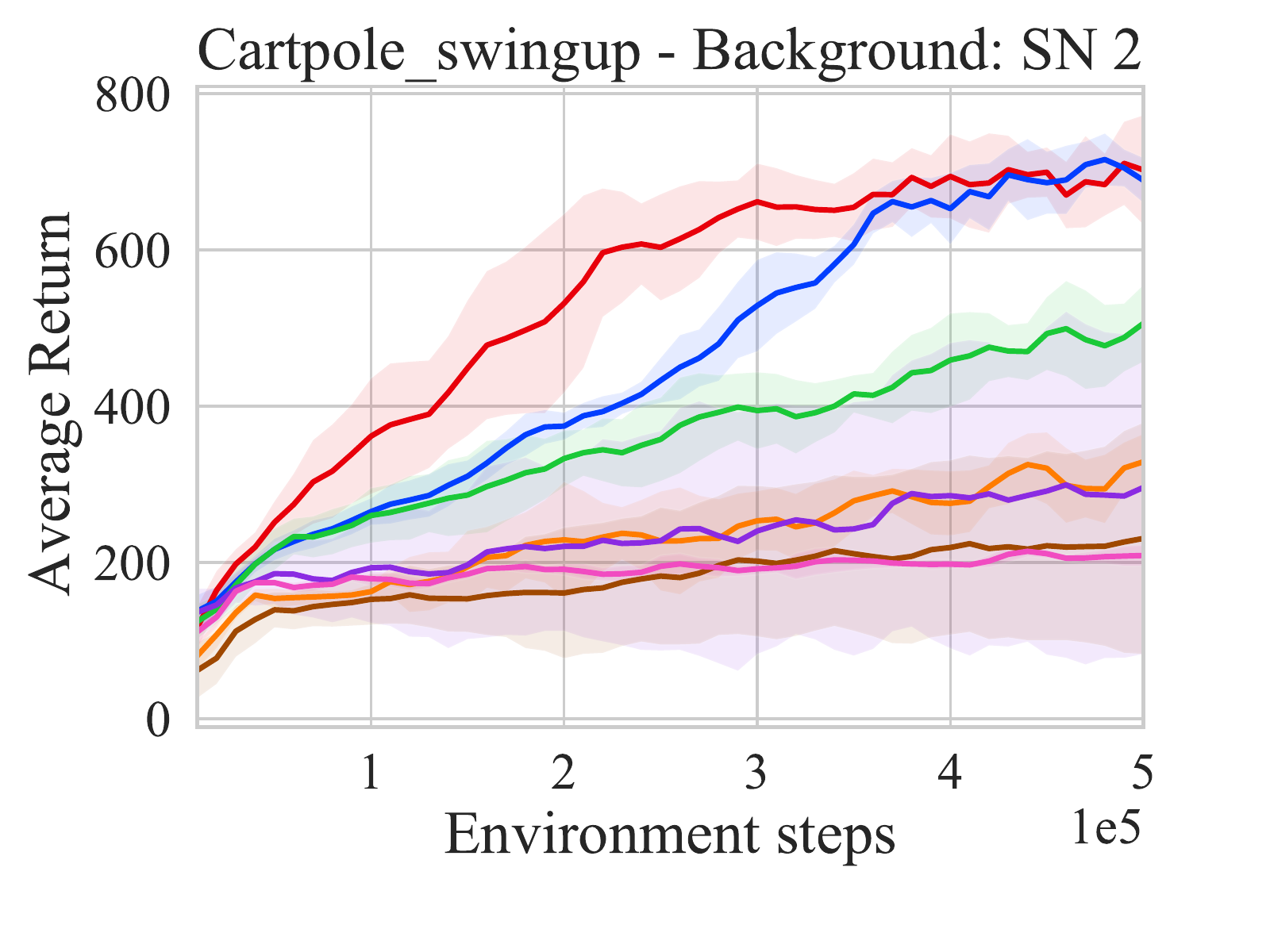}

  \includegraphics[width=9cm]{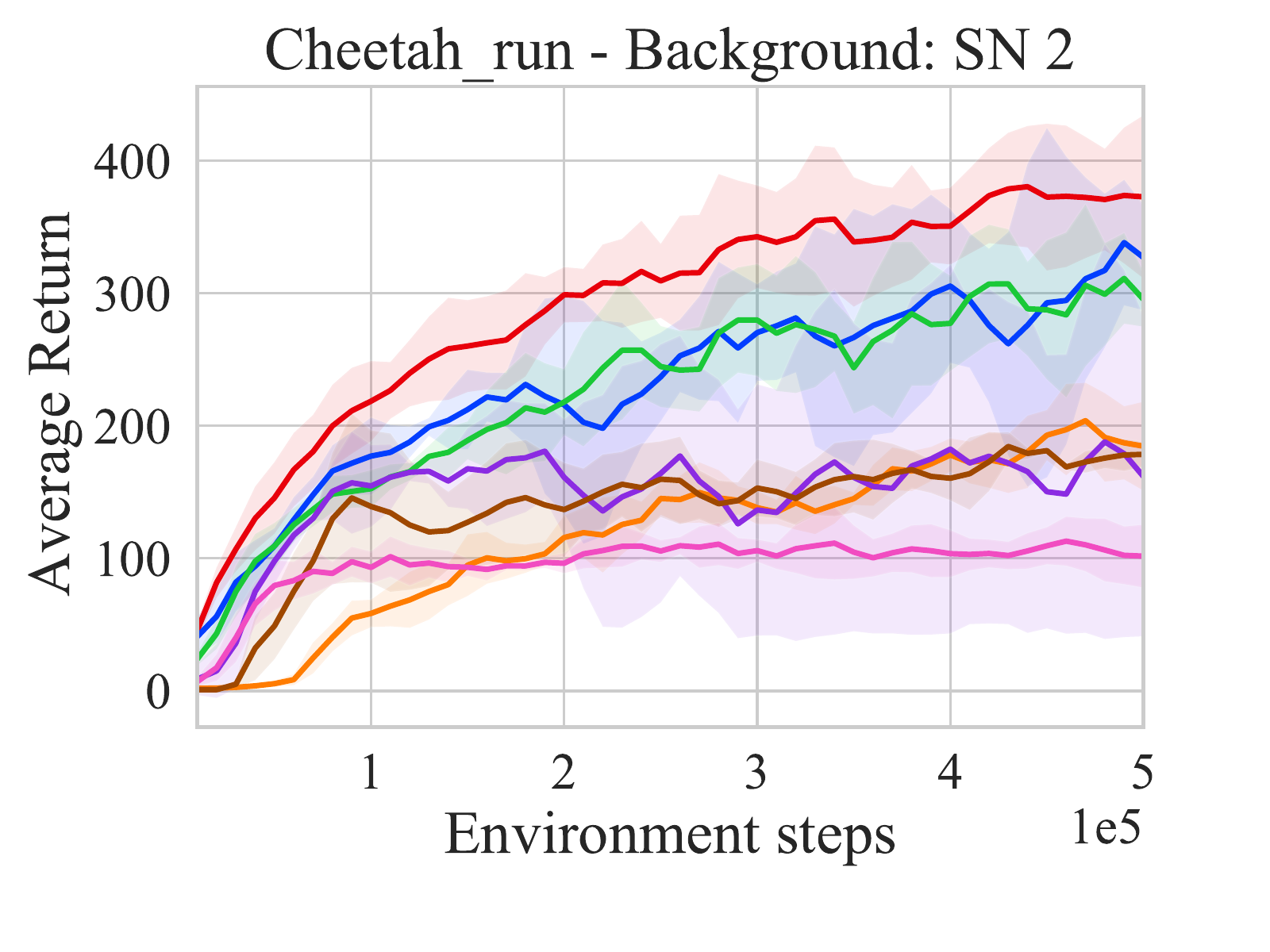}
  \includegraphics[width=9cm]{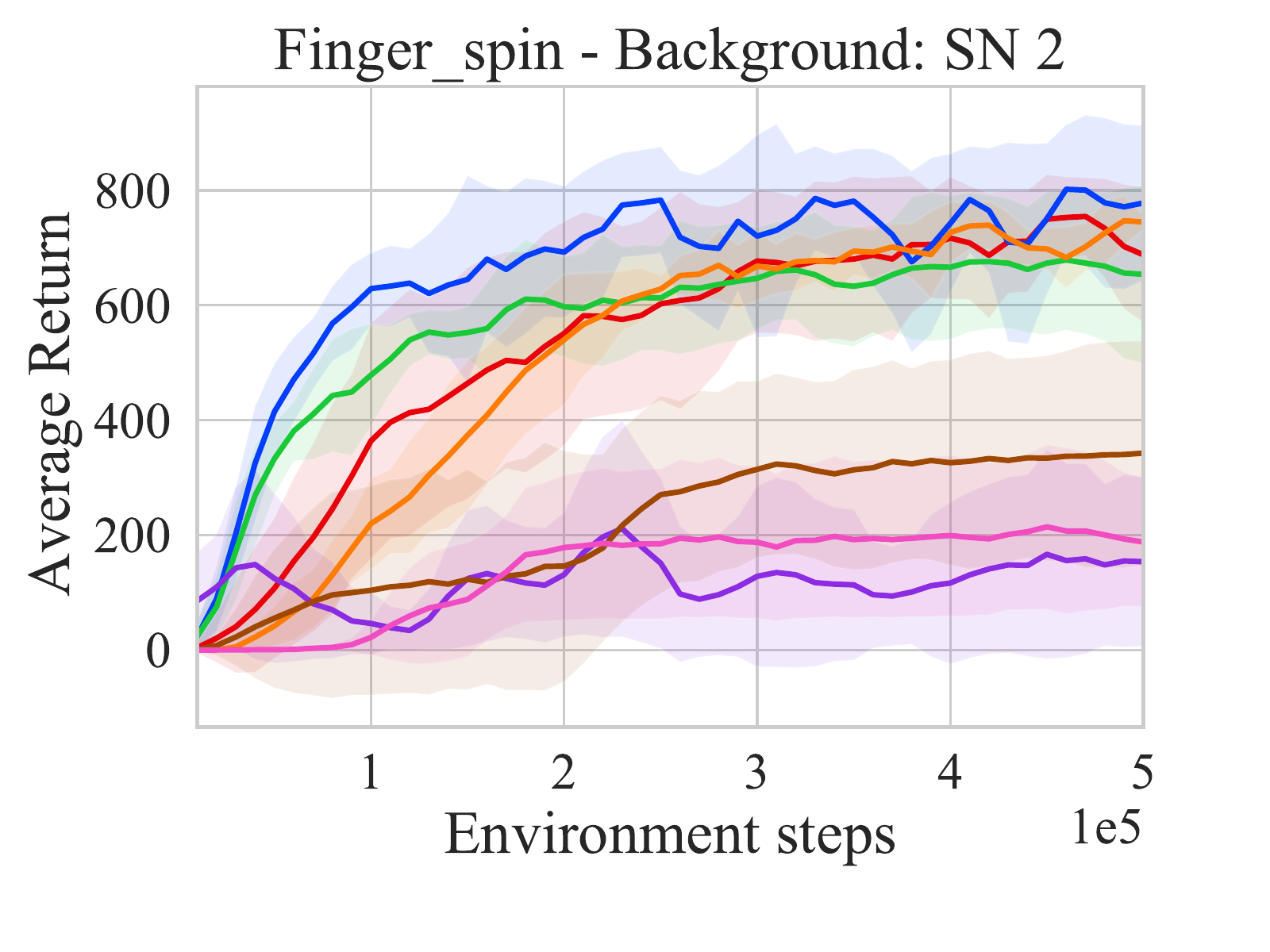}
  
  \includegraphics[width=9cm]{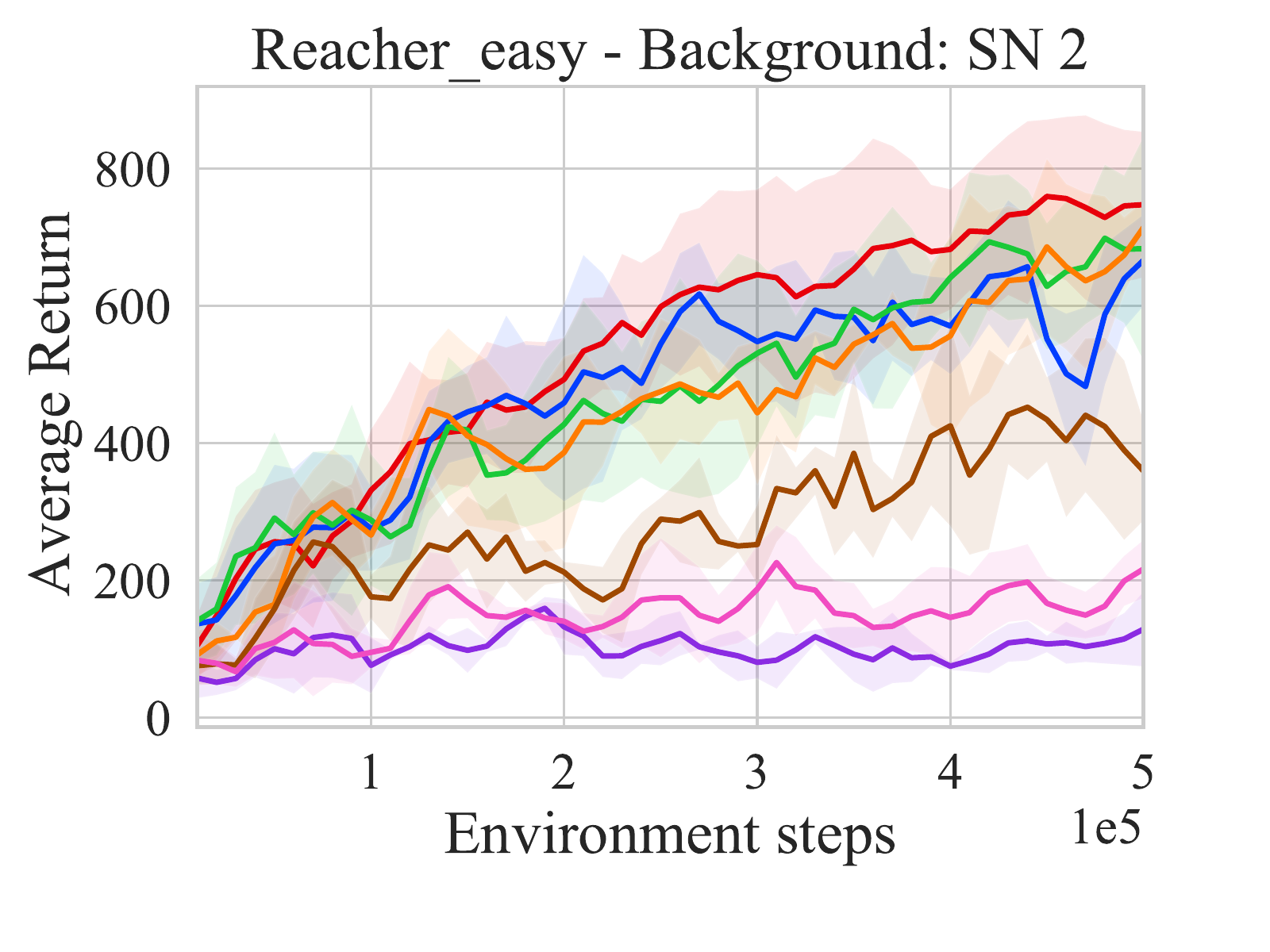}
  \includegraphics[width=9cm]{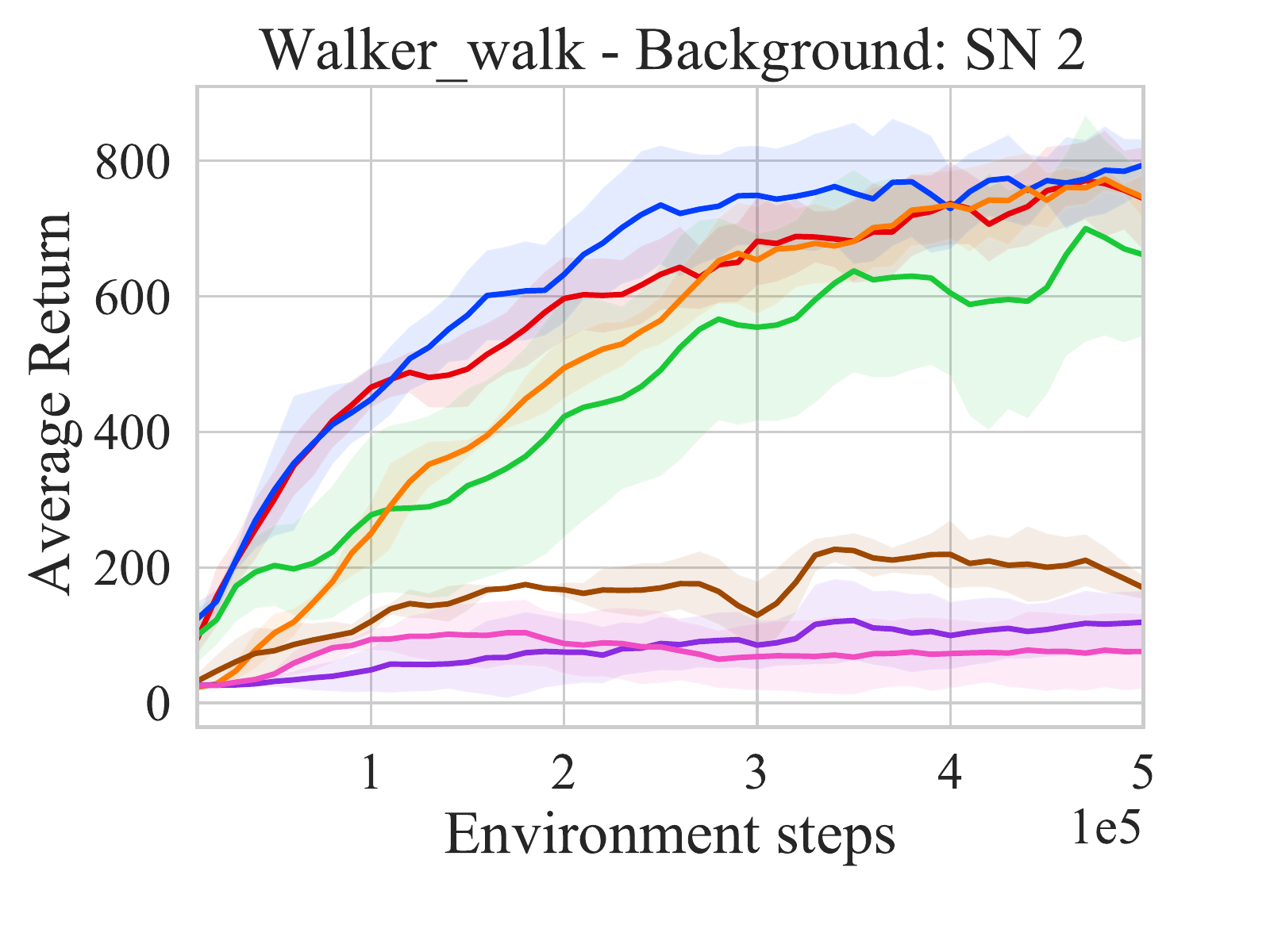}
  
  \includegraphics[width=18cm]{legend-2.pdf}
  \caption{Learning curves on six tasks under the \textbf{two} training environment setting with dynamic \textbf{background} distractions for 500K environment steps. SN denotes the number of source domain, which is the number of training environment.}
\label{fig-result-db-2}
\end{figure*}

\begin{figure*}
  \centering
  \includegraphics[width=9cm]{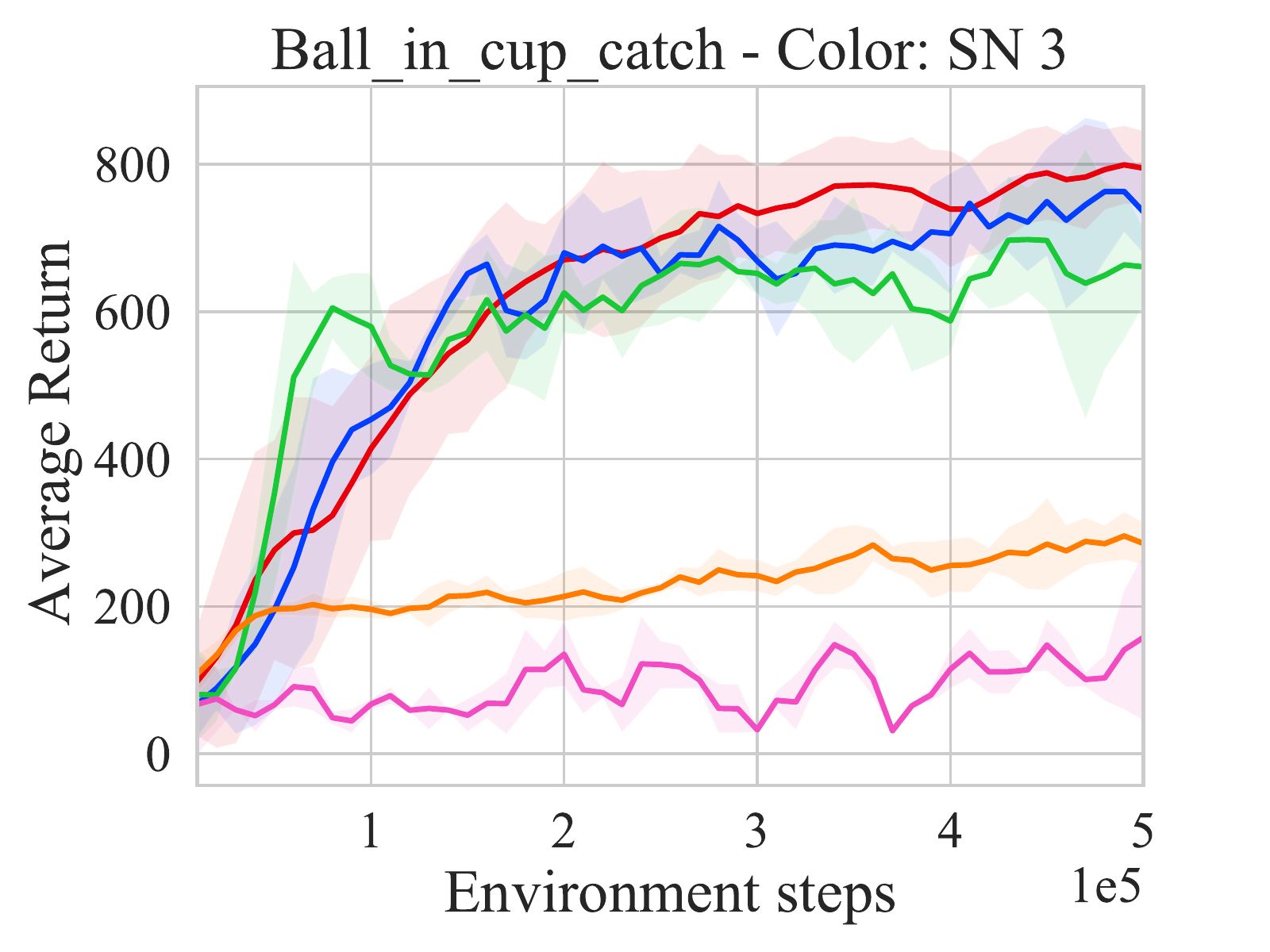}
  \includegraphics[width=9cm]{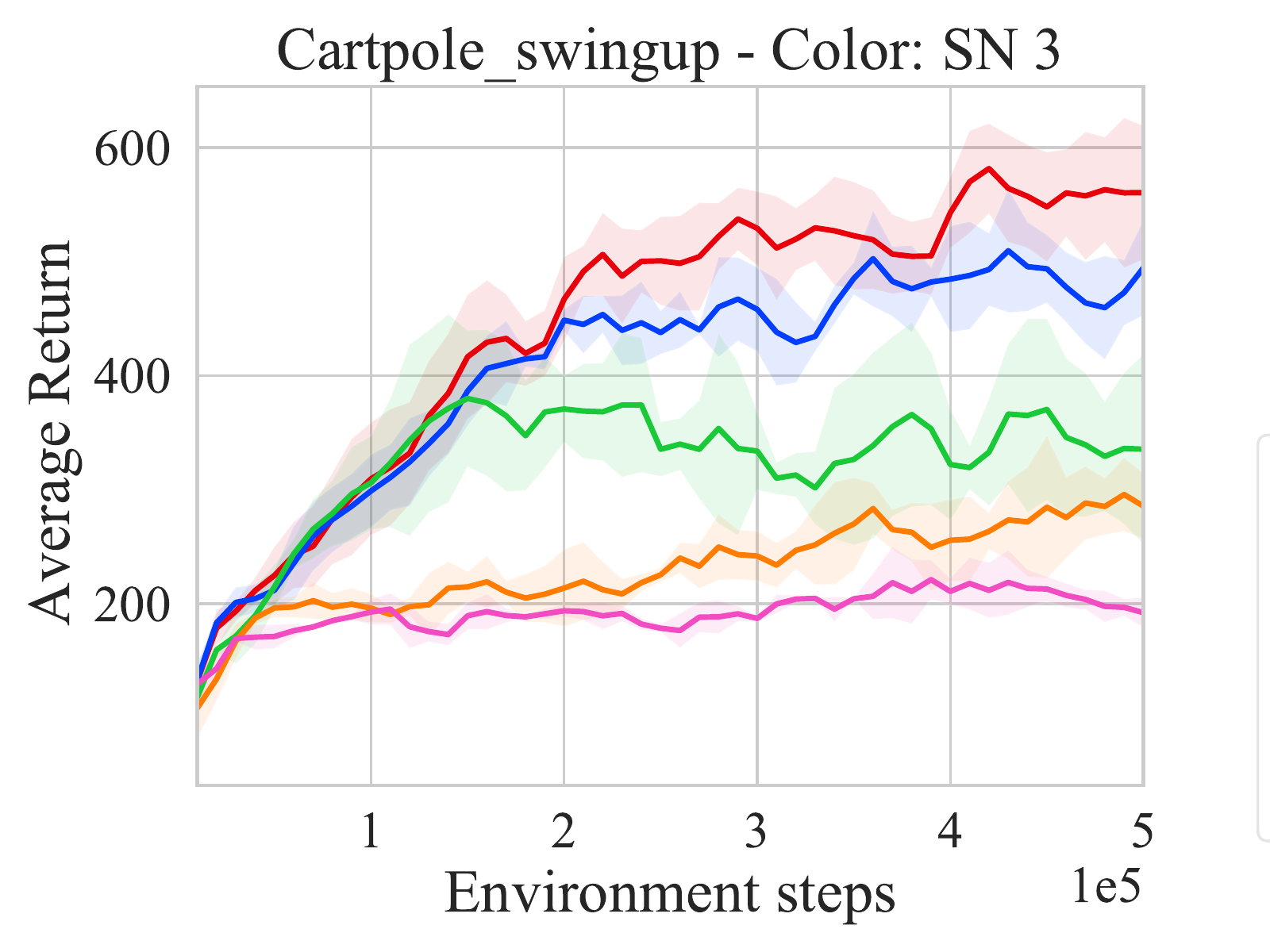}
  
  \includegraphics[width=9cm]{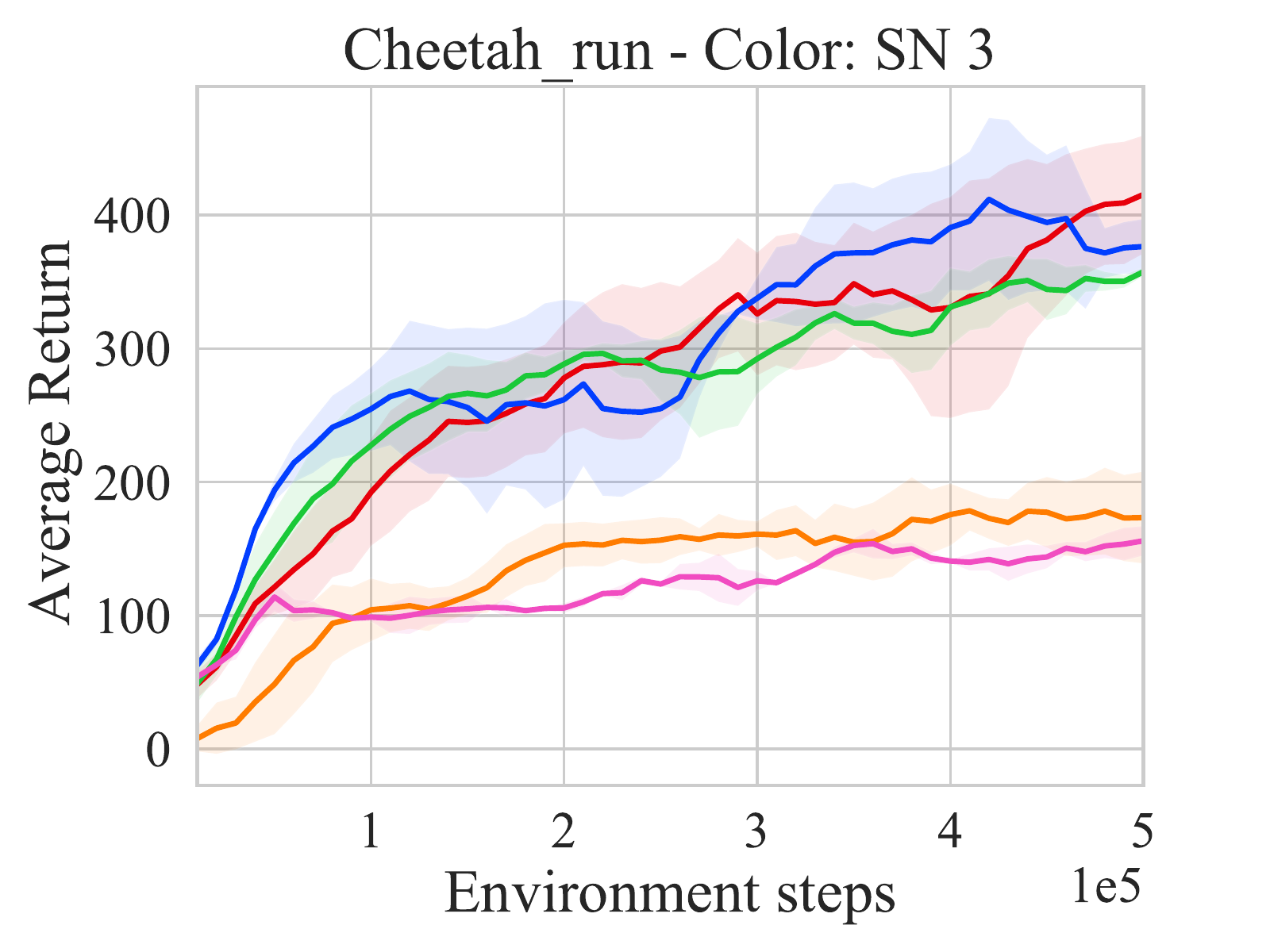}
  \includegraphics[width=9cm]{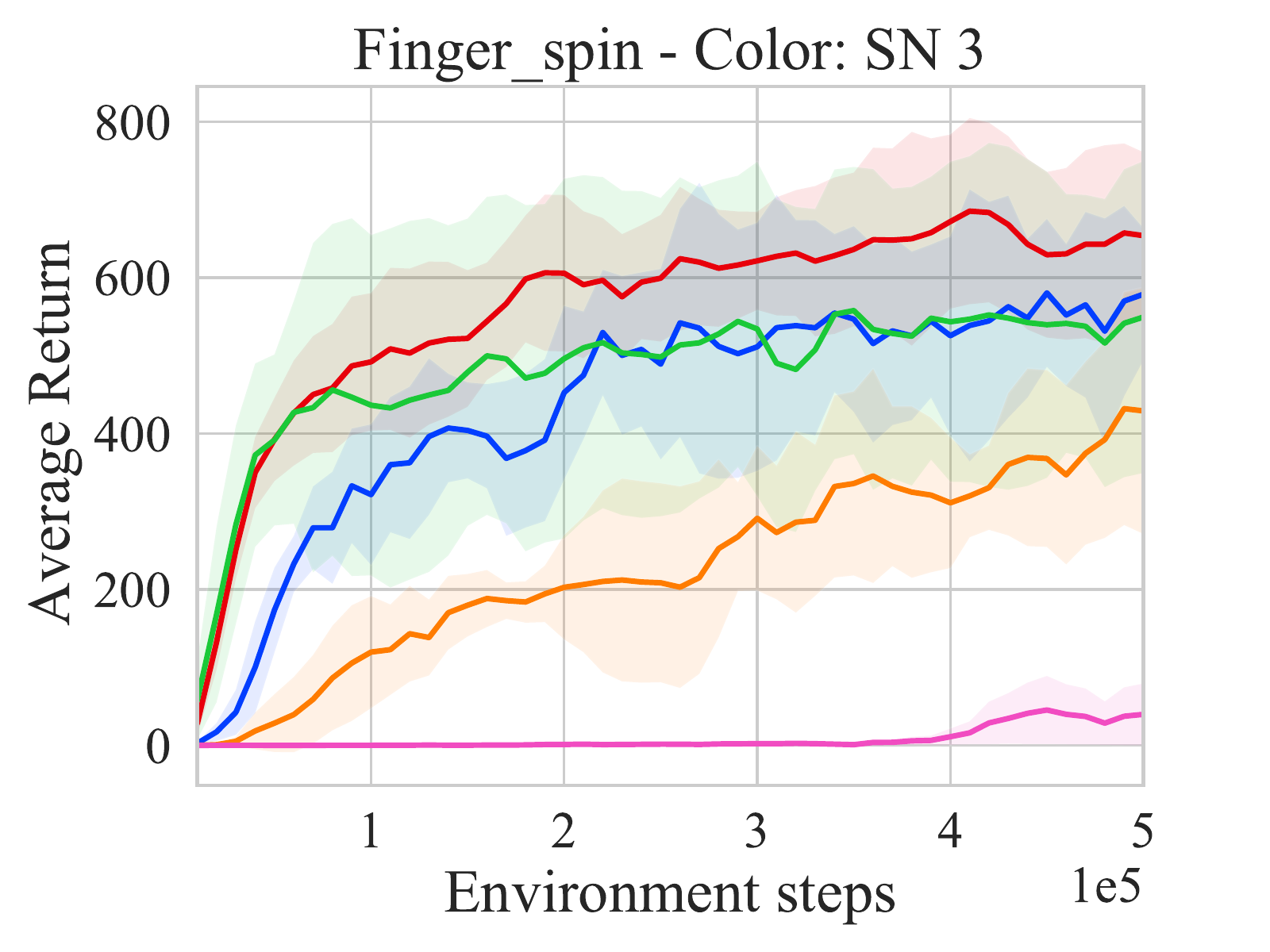}
  
  \includegraphics[width=9cm]{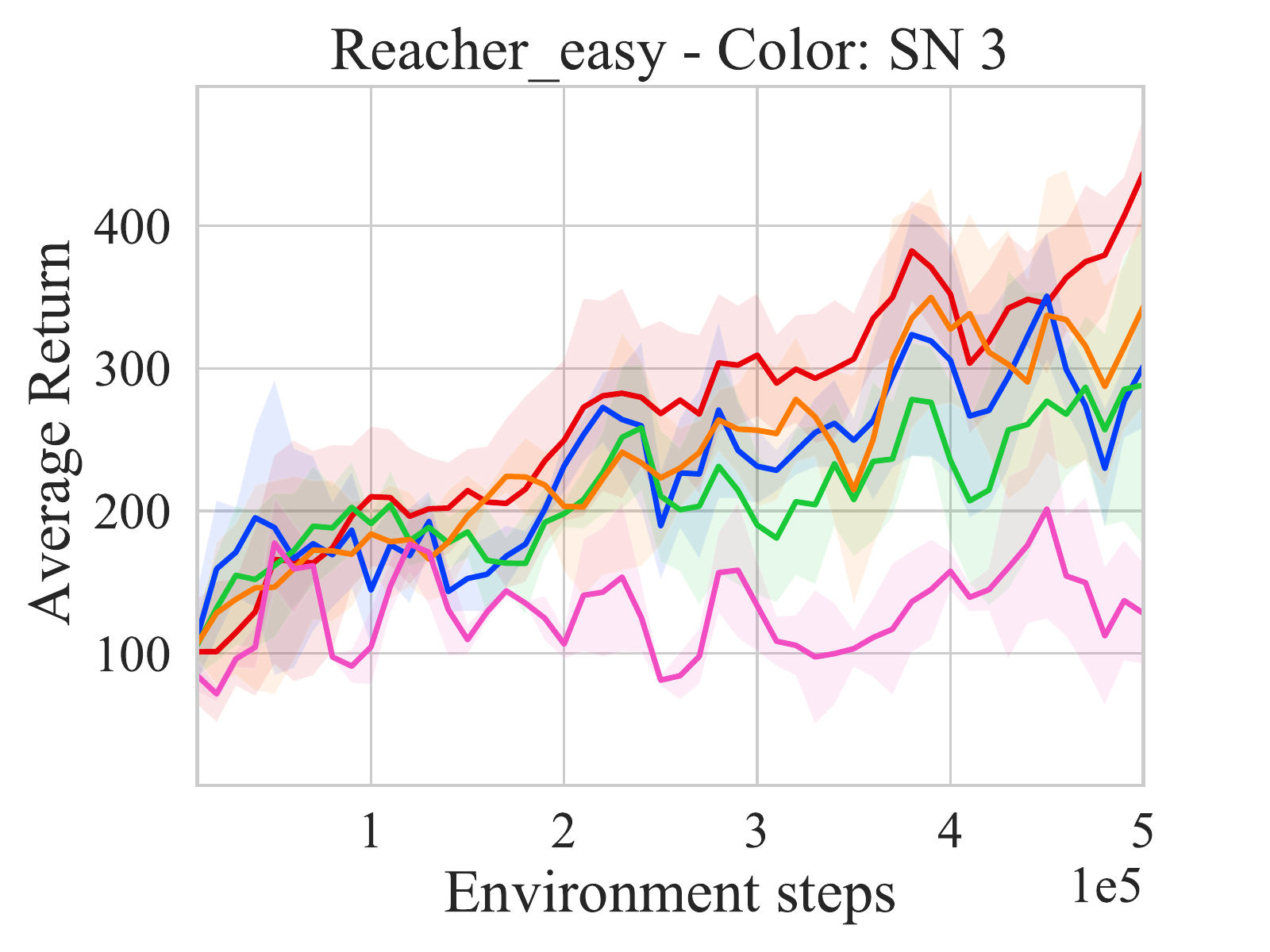}
  \includegraphics[width=9cm]{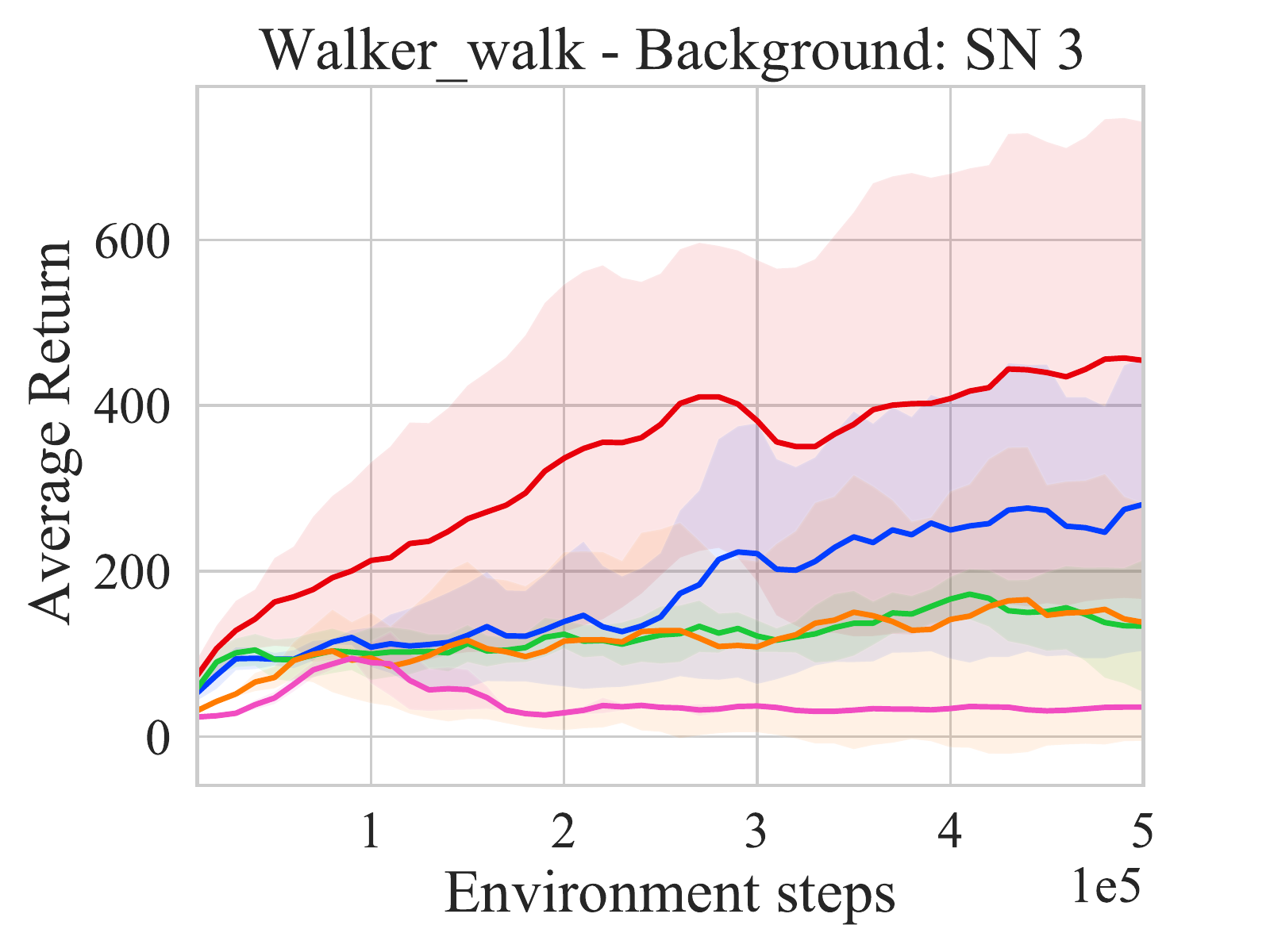}
  
  \includegraphics[width=15cm]{legend-13.pdf}
  \caption{Learning curves on six tasks under the \textbf{three} training environment setting with dynamic \textbf{color} distractions for 500K environment steps. SN denotes the number of source domain, which is the number of training environment.}
\label{fig-result-dc-3}
\end{figure*}

\begin{figure*}
  \centering
  \includegraphics[width=9cm]{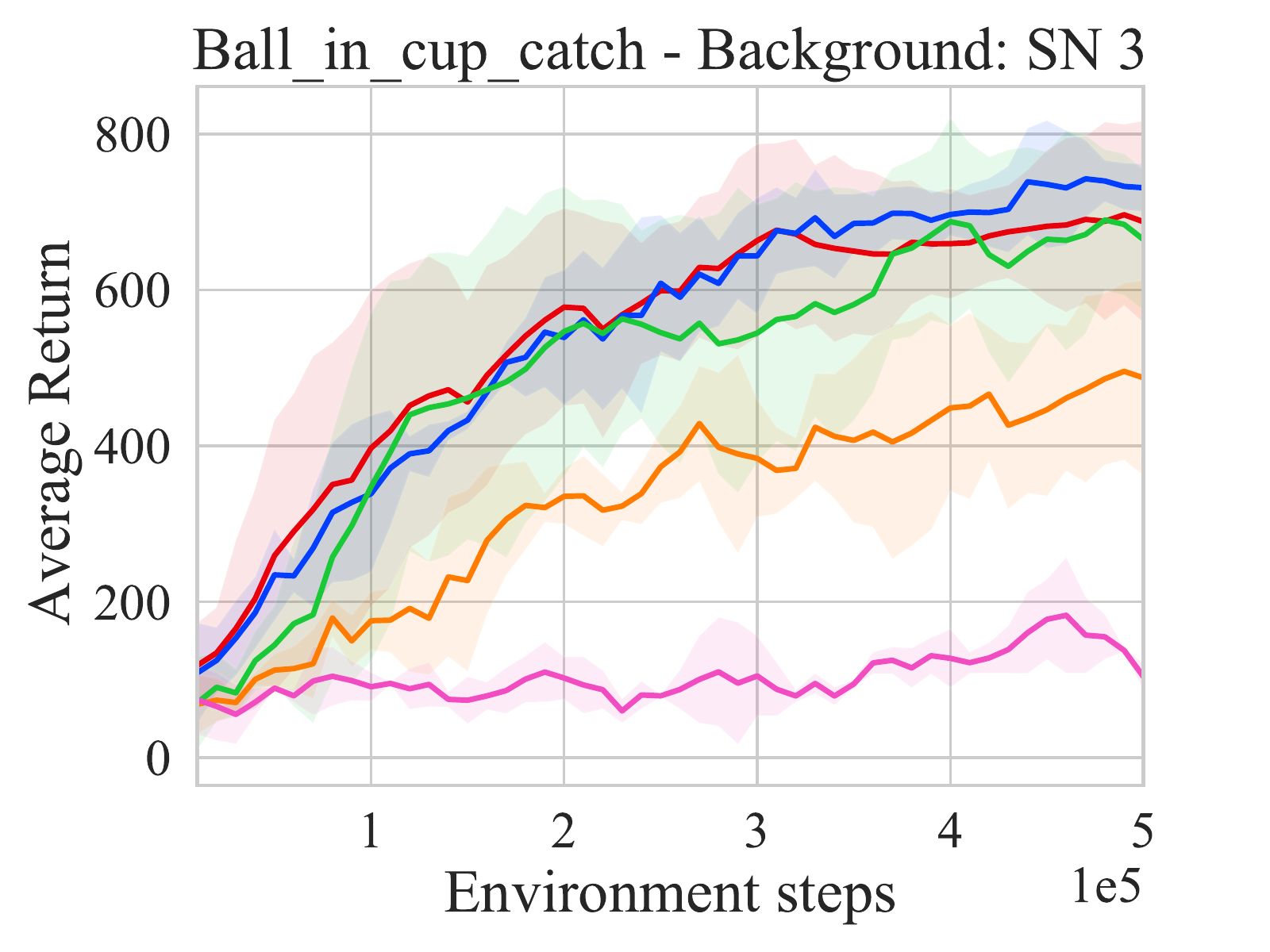}
  \includegraphics[width=9cm]{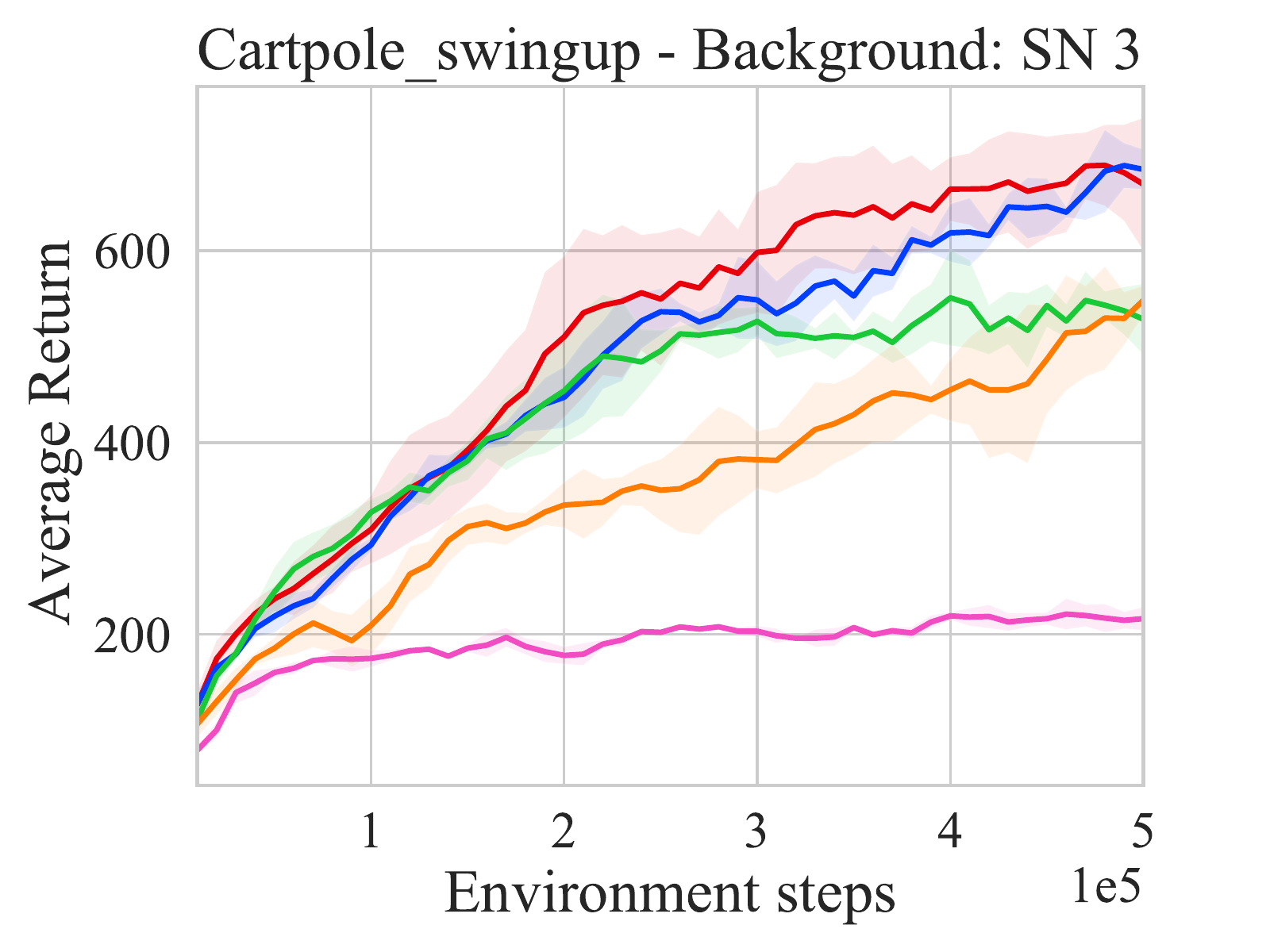}

  \includegraphics[width=9cm]{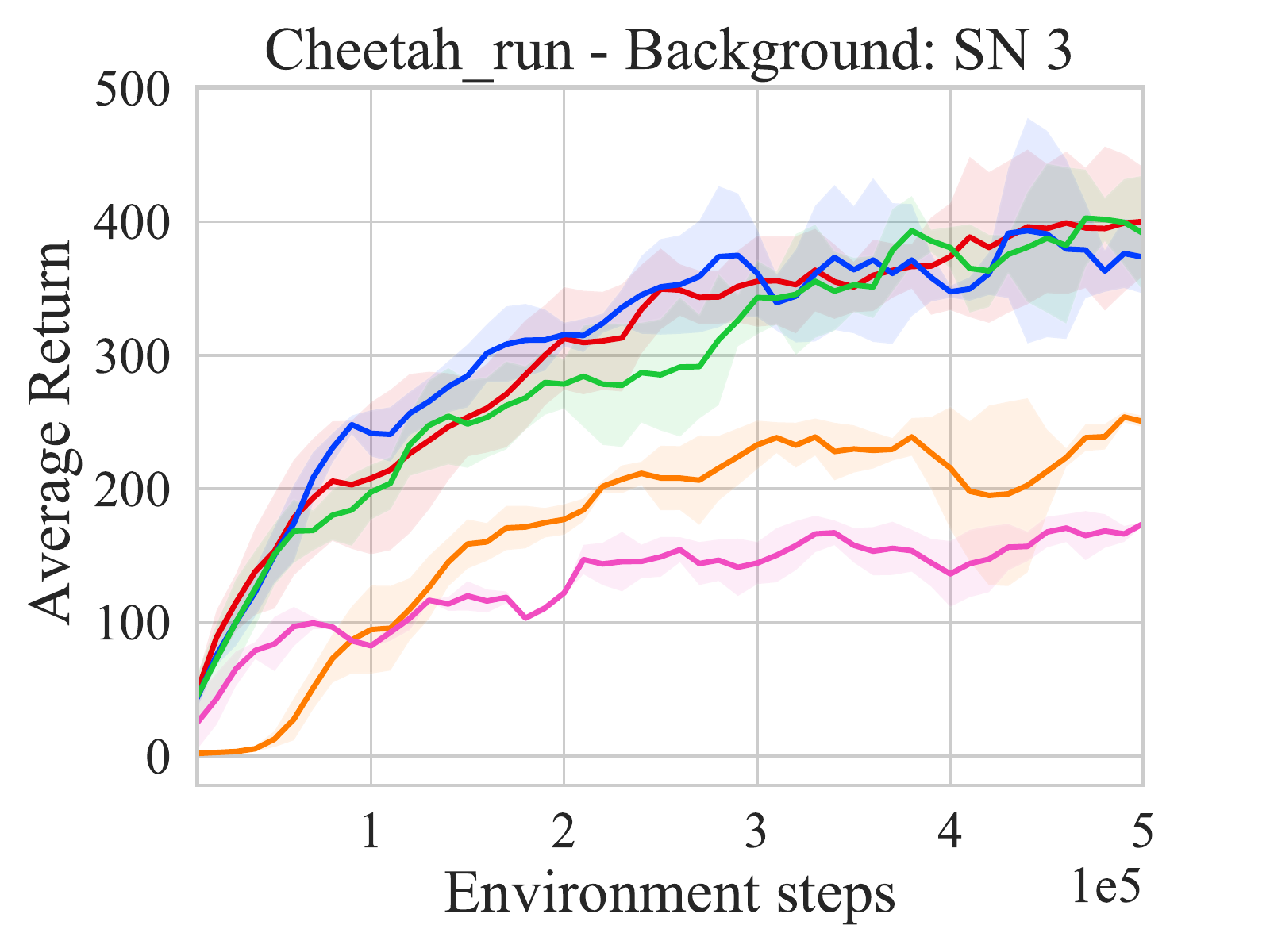}
  \includegraphics[width=9cm]{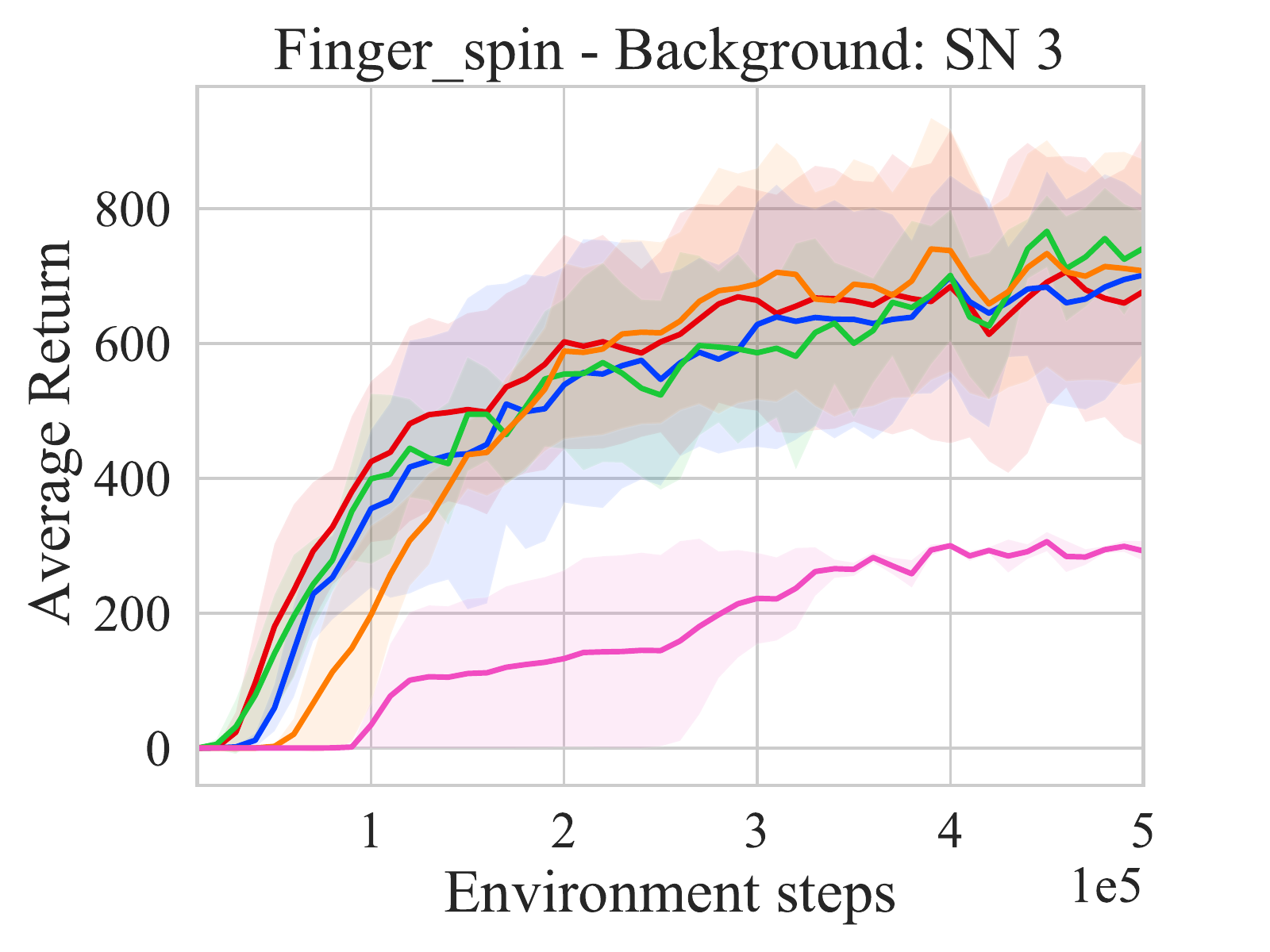}
  
  \includegraphics[width=9cm]{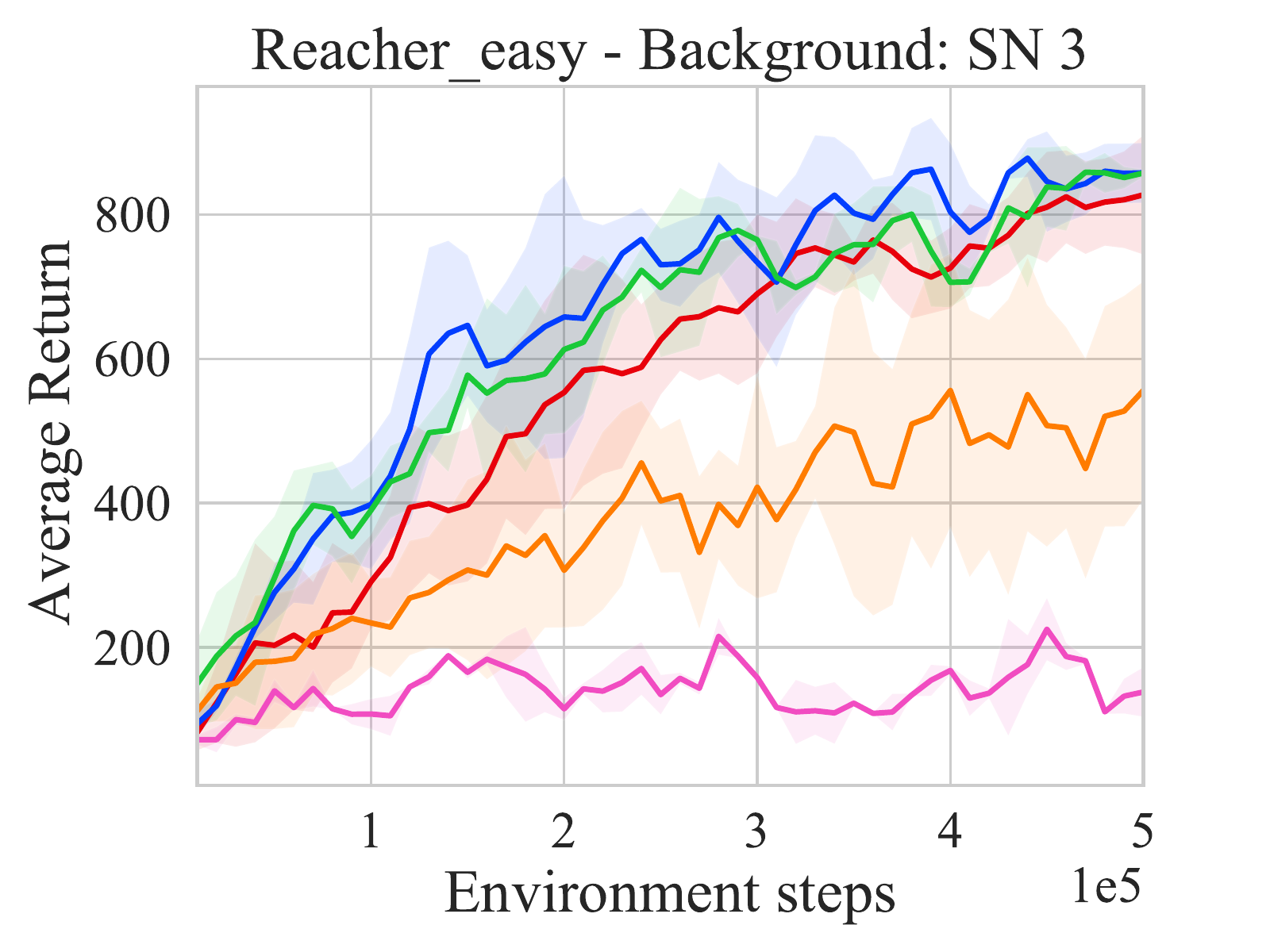}
  \includegraphics[width=9cm]{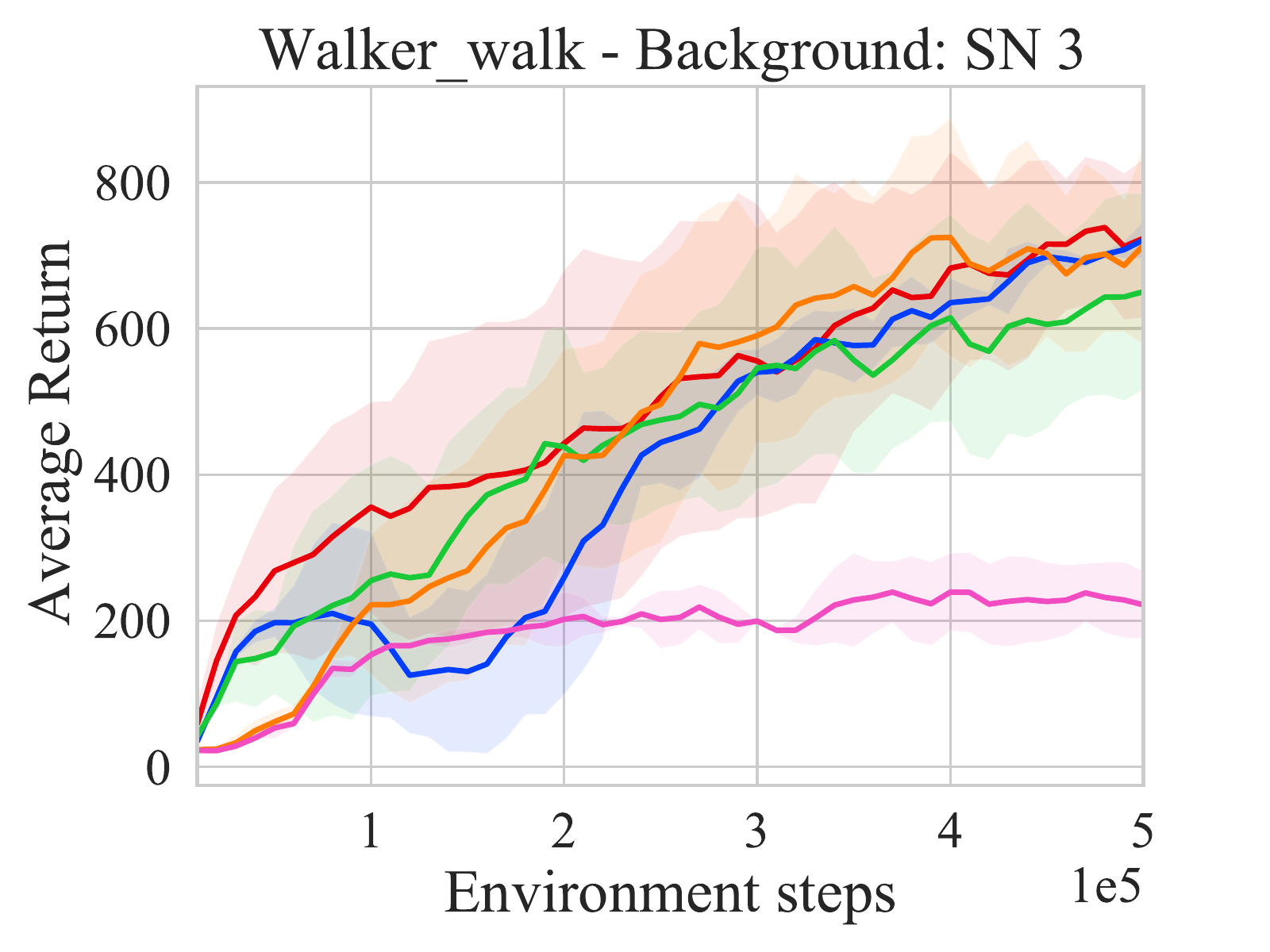}
  
  \includegraphics[width=15cm]{legend-13.pdf}
  \caption{Learning curves on six tasks under the \textbf{three} training environment setting with dynamic \textbf{background} distractions for 500K environment steps. SN denotes the number of source domain, which is the number of training environment.}
\label{fig-result-db-3}
\end{figure*}

\end{document}